\pdfoutput=1
\documentclass{article}

\usepackage{amsmath}
\numberwithin{figure}{section}
\numberwithin{table}{section}
\usepackage{amsthm}
\usepackage{amsfonts}
\newtheorem{lemma}{Lemma}[section]
\newtheorem{theorem}{Theorem}[section]
\newtheorem{definition}{Definition}[section]
\newtheorem{corollary}{Corollary}[section]
\newtheorem{observation}{Observation}[section]

\newtheorem{restatelemma}{Lemma}[section]
\newtheorem{restatecorollary}{Corollary}[section]
\usepackage{microtype}
\usepackage{graphicx}
\usepackage[section]{algorithm}
\usepackage{booktabs} % for professional tables
\usepackage{tabulary}
\usepackage{mathtools}
\usepackage{siunitx}
\usepackage{todonotes}
\usepackage{caption}
\usepackage{subcaption}
\usepackage{xspace}
\newcommand{\defeq}{\vcentcolon=}

\newcommand{\E}{\mathbb{E}}
\newcommand{\T}{\mathcal{T}}
\newcommand{\V}{\mathcal{V}}
\DeclarePairedDelimiterX{\infdivx}[2]{(}{)}{%
  #1\,\delimsize|\delimsize|#2%
}
\DeclarePairedDelimiterX{\infdivbigx}[2]{\biggl(}{ \biggr)}{%
  #1\,\delimsize|\delimsize|#2%
}
\let\originalleft\left
\let\originalright\right
\renewcommand{\left}{\mathopen{}\mathclose\bgroup\originalleft}
\renewcommand{\right}{\aftergroup\egroup\originalright}
\newcommand{\kld}[2]{\ensuremath{D_{KL}\infdivx{#1}{#2}}\xspace}
\newcommand{\fd}[2]{\ensuremath{D_{F}\infdivx{#1}{#2}}\xspace}
\newcommand{\fvd}[2]{\ensuremath{d_{2,\nu}(#1,#2)}\xspace}
\newcommand{\norm}[1]{\left\lVert#1\right\rVert}
\newcommand{\kldbig}[2]{\ensuremath{D_{KL}\left(\vphantom{#2}#1\right|\left|\vphantom{#1}#2\right)}}

\newcommand{\eqperiod}{\ensuremath{\,\text{.}}}
\newcommand{\eqcomma}{\ensuremath{\,\text{,}}}

\DeclareMathOperator*{\argmin}{arg\!\min}

\DeclareMathOperator{\tr}{tr}
\DeclareMathOperator{\clip}{clip}

\usepackage{hyperref}

\usepackage[accepted]{icml2020}

\icmltitlerunning{Generalization and Hyperparameter Optimization}

\begin{document}

\twocolumn[
\icmltitle{Efficient Hyperparameter Optimization By Way Of PAC-Bayes Bound Minimization}



\icmlsetsymbol{equ}{*}

\begin{icmlauthorlist}
\icmlauthor{John J. Cherian}{des}
\icmlauthor{Andrew G. Taube}{des}
\icmlauthor{Robert T. McGibbon}{des}
\icmlauthor{Panagiotis Angelikopoulos}{des}
\icmlauthor{Guy Blanc}{des}
\icmlauthor{Michael Snarski}{des}
\icmlauthor{Daniel D. Richman}{des}
\icmlauthor{John L. Klepeis}{des}
\icmlauthor{David E. Shaw}{des,col}
\end{icmlauthorlist}

\icmlaffiliation{des}{D.E. Shaw Research, New York, New York, USA}
\icmlaffiliation{col}{Department of Biochemistry and Molecular Biophysics, Columbia University, New York, New York, USA}
\icmlcorrespondingauthor{Andrew G. Taube}{Andrew.Taube@DEShawResearch.com}
\icmlcorrespondingauthor{David E. Shaw}{David.Shaw@DEShawResearch.com}

\icmlkeywords{hyperparameter optimization, generalization, PAC-Bayes}

\vskip 0.3in
]



\printAffiliationsAndNotice{}  % leave blank if no need to mention equal contribution

\begin{abstract}
Identifying optimal values for a high-dimensional set of hyperparameters is a problem that has received growing attention given its importance to large-scale machine learning applications such as neural architecture search.  Recently developed optimization methods can be used to select thousands or even millions of hyperparameters.  Such methods often yield overfit models, however, leading to poor performance on unseen data.  We argue that this overfitting results from using the standard hyperparameter optimization objective function.  Here we present an alternative objective that is equivalent to a Probably Approximately Correct-Bayes (PAC-Bayes) bound on the expected out-of-sample error.  We then devise an efficient gradient-based algorithm to minimize this objective; the proposed method has asymptotic space and time complexity equal to or better than other gradient-based hyperparameter optimization methods.  We show that this new method significantly reduces out-of-sample error when applied to hyperparameter optimization problems known to be prone to overfitting.
\end{abstract}

\section{Introduction}
\label{sec:intro}
Hyperparameters are settings that must be chosen prior to fitting a model; such hyperparameters include the learning algorithm, choice of training data, and optimizer. Without careful selection of the values of these hyperparameters, models fit to a training set often fail to achieve the goal of machine learning: accurate prediction on unseen (i.e., out-of-sample) data. To address this risk, the standard hyperparameter optimization strategy begins by constructing an independent validation data set to serve as a proxy for out-of-sample data. Optimization then proceeds in two nested steps. In the inner step, the model parameters are fit to minimize the training set error; in the outer step, the hyperparameters are adjusted to minimize the trained model's validation set error. An outer step thus requires a potentially costly refit of the model parameters.

Though computationally burdensome, the optimization of a large number of hyperparameters can sometimes greatly improve model performance on unseen data \cite{franceschi17a, franceschi18a, Liu2018}. \citet{Metz19}, for example, were able to substantially reduce the out-of-sample error of an image classifier by optimizing the 20,000 hyperparameters that control their classifier's optimizer. For many applications, however, large-scale hyperparameter optimization fails to improve model performance on unseen data, despite reducing model error on the validation set \cite{lorraine2019optimizing, LiTalwalkar19, zela2020understanding}. To understand this generalization error (i.e., the discrepancy between validation set error and error on truly unseen data), we must examine the standard hyperparameter optimization objective function
\begin{equation}
    \begin{split}
        \min_\lambda &\quad \hat{R}_\V\left(\theta^*(\lambda, S_\T); S_\V, \lambda \right)\\
        &\text{s.t. } \theta^*(\lambda, S_\T) = \argmin_\theta \hat{R}_\T(\theta;  S_\T, \lambda)\eqcomma
    \end{split}
    \label{eq:bilevel-ho}
\end{equation}
$\theta$ are the model parameters, $\lambda$ are hyperparameters, $S_\T$ and $S_\V$ are the training and validation data sets, $\hat{R}_\T(\cdot; S_\T, \lambda)$ is a measure of the training set error, and $\hat{R}_\V(\cdot; S_\V, \lambda)$ is a measure of the validation set error. We denote the minimizer of this objective as $\lambda^*$. 

We consider two sources of generalization error in hyperparameter optimization.  The first is performance degradation when we evaluate the model $\theta^*(\lambda^*, S_\T)$ on out-of-sample validation data instead of $S_\V$. The second is the variability of $\theta^* (\lambda^*,S_\T)$ when $S_\T$ is replaced with out-of-sample training data. In applications of hyperparameter optimization that were feasible prior to the development of highly scalable optimization methods, these two sources of generalization error were seldom encountered. Optimization of a single regularization penalty, for instance, does not overfit the validation data because of its limited capacity, and the optimized penalty is rarely applied to additional training data fits. By contrast, modern applications such as learned optimizers or neural architecture search involve thousands of re-used hyperparameters, and are thus susceptible to both training and validation overfitting \cite{Metz19, LiTalwalkar19}. Heuristics for reducing overfitting in certain applications have been empirically validated, but remain poorly understood \cite{guiroy2019towards, zela2020understanding}. Without a theoretical framework for studying overfitting, it is difficult to generalize these methods to a broader class of hyperparameter optimization problems.

Here we present an objective for hyperparameter optimization that both addresses the two sources of generalization error described above and motivates existing methods for reducing overfitting. Using a novel extension of ``Probably Approximately Correct Bayesian'' (PAC-Bayes) theory \cite{mcallester1999some}, we derive a bound on the expected out-of-sample error. In Section~\ref{sec:objective}, we describe a theory of generalization error that motivates the presented objective.  We outline an efficient algorithm in Section~\ref{sec:algorithm} to minimize this objective and compare its asymptotic complexity with existing methods for hyperparameter optimization.  In Section~\ref{sec:experiments}, we demonstrate the utility of our approach by applying our algorithm to several problems known to be prone to validation set overfitting. We also show how the PAC-Bayesian bounds we derive can help to explain the benefits and limitations of previously developed heuristics for reducing overfitting.
\section{Objective} \label{sec:objective}
\subsection{Preliminaries} \label{sec:prelims}
The hyperparameter vector $\lambda \in \Lambda \subseteq \mathbb{R}^n$ and parameter vector $\theta \in \Theta(\lambda) \subseteq \mathbb{R}^{m}$ define a model that predicts a label $y$ given an input $x$. Note that $m$ can be determined by $\lambda$ and we assume $\Theta$ is compact. The training set $S_\T = \{x_i, y_i\}_{i  = 1}^{n_\T}$ is composed of input-label tuples $(x_i, y_i)$ that are independently sampled from the data-generating distribution $\mathcal{D}_\T$. The validation set $S_\V$ is constructed similarly. Note that the validation data-generating distribution $\mathcal{D}_\V$ is not necessarily identical to $\mathcal{D}_\T$. 

We measure the performance of a particular choice of parameters and hyperparameters by evaluating a loss function, $\ell_{(\cdot)}: \Theta \times \mathcal{X} \times \mathcal{Y} \times \Lambda \to \mathbb{R}^+$. We define the empirical (or in-sample) risk $\hat{R}_{(\cdot)}$ and expected (or out-of-sample) risk $R_{(\cdot)}$ of a parameter $\theta$ below. The subscript $\T$ denotes the loss/risk function used during training, while $\V$ denotes the loss/risk function applied to the validation set. When the subscript for $\ell$, $R$, $S$, and $\mathcal{D}$ can be either $\T$ or $\V$, we omit the $(\cdot)$ placeholder. 
\begin{align*}
    \hat{R}(\theta; S, \lambda) &\defeq \frac{1}{|S|}\sum_{(x_i, y_i) \in S} \ell(\theta; x_i, y_i, \lambda) \eqcomma\\
    R(\theta; \mathcal{D}, \lambda) &\defeq \E_{\mathcal{D}} [\ell(\theta; x, y, \lambda)]\eqperiod
\end{align*}
We assume that $\int_\Theta e^{-\tau \hat{R}(\theta; S, \lambda)} \pi(d\theta) < \infty$ for all $S \sim \mathcal{D}^n$ and $\tau > 0$. We take the reference measure to be Lebesgue measure on the compact set $\Theta$. Then, the Gibbs posterior distribution over the parameters $\theta$ given $S, \lambda$, and $\tau$ has density \cite{catoni2004statistical} 
\begin{align}
    p^{\tau}(\theta | S, \lambda) \propto e^{-\tau \hat{R}(\theta; S, \lambda)} \label{eq:posteriordef}\eqperiod
\end{align}
We assume that any distribution over parameters we consider admits a differentiable probability density. 

We rely upon several measures of difference between probability distributions in our exposition. The Kullback-Leibler (KL) divergence is defined as
\begin{align}
    \kld{Q}{P} \defeq \E_Q \left [ \log q(x) - \log p(x) \right] \label{eq:kldiv} \eqperiod
\end{align}
The 1-Wasserstein distance is defined as follows for $P$ and $Q$ with bounded support \cite{villani2008optimal}; let $\mathcal{F}$ be the set of real-valued Lipschitz continuous functions with Lipschitz constant no greater than $1$.
\begin{align}
    d_{\mathcal{W}}^1 (Q, P) \defeq \sup_{f \in \mathcal{F}} \left \{\E_Q [f(x)] - \E_P[f(x)] \right\} \label{eq:wassdist}\eqperiod
\end{align}
Unlike the KL divergence, the Fisher divergence (Eq.~\ref{eq:fisherdiv}) and the $(2, \nu)$-Fisher distance (Eq.~\ref{eq:fisherdist}) are computable even when the partition functions of $P$ and $Q$ are unknown \cite{huggins2018practical}. 
\begin{subequations}
\begin{align}
    \fd{Q}{P} \defeq& \E_Q \left [\norm{\nabla \log q(x) - \nabla \log p(x)}_2^2 \right ] \label{eq:fisherdiv}\eqcomma \\
    \fvd{Q}{P} \defeq& \sqrt{\E_\nu \left [ \norm{\nabla \log q(x) - \nabla \log p(x)}_2^2 \right]} \label{eq:fisherdist}\eqperiod
\end{align}
\end{subequations}
When $\nu = Q$, note that $\fvd{Q}{P} = \sqrt{\fd{Q}{P}}$. 

For the sake of brevity, we make the following notational simplifications. We define $p(\theta | S, \lambda) \defeq p^{|S|}(\theta | S, \lambda)$, and we use $\nabla_\theta F(\theta_t)$, instead of $\nabla_\theta F(\theta)|_{\theta = \theta_t}$, to denote the gradient of $F$ evaluated at $\theta_t$.
\subsection{Motivation} \label{sec:motivation}
The poor generalization of large-scale hyperparameter optimization can be understood through the lens of adaptive data analysis \cite{dwork2015generalization}. To optimize the hyperparameters $\lambda$, we reuse a validation set that is held out from the parameter optimization procedure. Repeatedly querying this holdout data set during hyperparameter optimization can result in models that are overfit to the validation set. This phenomenon contradicts our assumption that the validation error is an accurate proxy for out-of-sample error; we have no justification for why minimizing Eq.~\ref{eq:bilevel-ho} would result in models that generalize. One approach to reduce this overfitting is to carefully control access to the holdout set. \citet{dwork2015generalization} prove generalization guarantees for holdout set algorithms that return obfuscated statistics or answer only a limited number of queries. Leveraging these methods in real problems, however, can catastrophically and unnecessarily limit the model's performance \cite{ji2014differential}. Instead, we propose a regularized hyperparameter optimization objective that penalizes hyperparameters more likely to result in overfit models.

To motivate this approach, we first examine a straightforward application of hyperparameter optimization that is susceptible to overfitting: feature selection. Methods specifically developed for feature selection, such as the Akaike Information Criterion (AIC), guard against overfitting by penalizing the inclusion of additional features; models with more features must have substantially lower in-sample risk to justify their selection. The feature selection problem suggests that we can correct Eq.~\ref{eq:bilevel-ho} by penalizing the dimensionality of the parameter vector $\theta$. Model complexity, however, is not always a useful proxy for generalization. Single-classifier generalization bounds, which rely on measures of model complexity such as parameter dimensionality or Vapnik-Chernovenkis dimension, are uncorrelated with the true generalization error for over-parameterized models \cite{zhang2016understanding}. Even if some yet-undiscovered measure of complexity is an accurate proxy for generalization error, many applications of hyperparameter optimization do not affect model complexity. In some meta-learning problems, for instance, the hyperparameter learned is the parameter initialization for future optimizations \cite{franceschi18a}. How, then, can we modify the objective in Eq.~\ref{eq:bilevel-ho} to promote choices of $\lambda$ that are likely to produce generalizable models?

As a step towards a superior objective, we reconsider the formulation of the parameter optimization in Eq.~\ref{eq:bilevel-ho}. Modern optimization methods applied to deep neural networks do not reliably converge to a single optimum of the training risk. Instead, methods such as stochastic gradient descent (SGD) and early stopping perform nonparametric variational inference on a particular distribution over $\Theta$ \cite{mandt2017sgd, duvenaud2016early}. In fact, \citet{mandt2017sgd} and \citet{chaudhari2018stochastic} speculate that the distribution sampled by SGD is similar to $p(\theta | S, \lambda)$. We modify the hyperparameter optimization objective to reflect these observations. Rather than evaluating the validation risk at a minimizer of the training risk, we instead assume that the output of the parameter optimization is a sample from a distribution approximating $p(\theta | S_\T, \lambda)$. In defining this objective, we assume that we run $T$ steps of an iterative gradient-based optimization method initialized with $\theta_0$ sampled from a user-specified initial distribution $P_0$. For notational simplicity, we define the distribution of the $t$-th iterate, $\theta_t$, as $\nu_t$. We denote the distribution of the final output (i.e., $\nu_T$) as $p(\theta_T| S_\T, \lambda)$ for emphasis. We then analyze the expected risk under this ``posterior'' distribution. Using this notation, we can succinctly reformulate the objective in Eq.~\ref{eq:bilevel-ho} as follows:
\begin{align}
    \min_\lambda \E_{p(\theta_T | S_\T, \lambda)} [\hat{R}_\V(\theta; S_\V, \lambda)] \label{eq:bilevel-ho-e}\eqperiod
\end{align}
Unlike Eq.~\ref{eq:bilevel-ho}, this objective quantifies the average validation risk of the network given that training can result in convergence to \emph{many} possible local optima. Leveraging this sampling-oriented perspective on optimization, we propose an alternative hyperparameter optimization objective that adds a regularizer to penalize hyperparameters that are unlikely to yield generalizable models.
\begin{multline}
\label{eq:ho-stability}
    \min_\lambda \E_{p(\theta_T | S_\T, \lambda)} [\hat{R}_\V(\theta; S_\V, \lambda)] + \\
    \zeta \sqrt{\sum_{t = 0}^{T-1} d^2_{2,\nu_t}\left (p^1(\theta | S^{(t)}_\T, \lambda), p^1(\theta | S^{(t)}_\V, \lambda) \right)}\eqperiod
\end{multline}
Recall that Eq.~\ref{eq:posteriordef} defines $p^1 (\theta| S ,\lambda)$ as the Gibbs posterior with $\tau=1$. $\zeta$ is a user-selected parameter that reflects the desired trade-off between minimizing generalization error and minimizing empirical validation risk. The superscript in $S^{(t)}$ denotes the $t$-th mini-batch, relevant if a stochastic method such as SGD is used. 

We rigorously justify the use of Eq.~\ref{eq:ho-stability} in Section~\ref{sec:gentheory}, but here we provide two intuitive arguments for its utility. While optimizing the objective in Eq.~\ref{eq:bilevel-ho-e} aligns the modes of the training and validation posterior distributions, it does not induce any broader agreement between the two distributions. In contrast, when $\nu_t \approx p(\theta | S_\T, \lambda)$, the regularizer in Eq.~\ref{eq:ho-stability} penalizes hyperparameters that produce similar posterior modes, but dissimilar posterior \emph{distributions}. When $\hat{R}_\V(\theta; S_\V, \lambda)$ is strongly convex in $\theta$, the $(2,\nu_t)$-Fisher distance summands of the regularizer are upper bounds for the 1- and 2-Wasserstein distances between the training and validation posteriors \cite{huggins2018practical}. Minimizing Eq.~\ref{eq:ho-stability} thus ensures agreement in not only the modes, but also the moments of both posterior distributions. 

We provide a second interpretation for how the regularizer improves generalization by substituting the definition of the $(2,\nu_t)$-Fisher distance from Eq.~\ref{eq:fisherdist} in Eq.~\ref{eq:ho-stability}. After rewriting the summands, we observe that the regularizer penalizes the 2-norm of the difference between the training and validation gradients (i.e., the proposed objective adds a so-called ``gradient incoherence'' regularizer to the standard hyperparameter optimization objective \cite{negrea2019information, guiroy2019towards}). Prior work has shown that models with similar, ``coherent'' gradient descent trajectories across multiple data sets exhibit faster convergence during optimization and improved generalization \cite{nichol2018first}.

These two arguments hint at how the regularizer in Eq.~\ref{eq:ho-stability} induces model ``stability'' \cite{bousquet2002stability}. Model stability measures the sensitivity of the posterior distribution to perturbations of the data set. Minimizing the $(2, \nu_t)$-Fisher distance then improves stability because a model is more likely to be stable if the training and validation optima coincide \emph{and} the curvature at these optima agree (i.e., the second moments of the posterior distributions match). Stability bounds on the out-of-sample error are ideal hyperparameter optimization objectives because the posterior distribution's characteristics are not only affected by the size of the search space, but also by other hyperparameters, such as the choice of data-generating distribution and optimizer.

\subsection{Related Work} \label{sec:pbrelwork}
Hyperparameter optimization by way of generalization bound minimization has been explored previously in the PAC-Bayes literature. \citet{pmlr-v76-thiemann17a}, for instance, use this strategy to determine the weights for an ensemble of Support Vector Machine (SVM) classifiers in which each SVM is trained on a small subset of the training data. To choose these weights, they derive a quasiconvex PAC-Bayes bound on the expected risk and an algorithm for bound minimization that provably converges to the globally optimal posterior distribution. They show that the weighted ensemble of SVMs produced by this algorithm can predict out-of-sample labels with similar accuracy to their benchmark: a kernelized SVM trained on the whole data set.

\citet{ambroladze2007tighter} minimize a PAC-Bayes bound to directly select the hyperparameters of a kernelized SVM. They construct this bound by using a held-out portion of the training set to define a PAC-Bayes prior; the PAC-Bayes posterior is defined using the remaining training data. \citet{ambroladze2007tighter} demonstrate that performing grid search over hyperparameters to minimize this bound produces an SVM with similar test-set error to the SVM that results from hyperparameter optimization using $10$-fold cross-validation. Cross-validation yields an effective baseline for evaluating new approaches to hyperparameter optimization, but is too computationally expensive for practical use in modern applications.

Our approach extends the methods introduced in these and other prior work to a broader class of hyperparameter optimization problems. The method proposed by \citet{pmlr-v76-thiemann17a} can only be applied when the partition function of the posterior distribution is known. Without the use of simplifying approximations to the true posterior, this method cannot be directly applied to most hyperparameter optimization problems. Similarly, \citet{ambroladze2007tighter} solely consider Gaussian PAC-Bayes posterior and prior distributions. The algorithm they then propose, which involves directly computing the KL divergence between the posterior and prior, is only feasible under restrictive assumptions regarding the form of these distributions. We build upon this method by developing an algorithm for the minimization of PAC-Bayes bounds that are a function of intractable parameter distributions implicitly defined by an optimization method. This enables hyperparameter optimization for complex models, such as molecular dynamics force fields, that admit non-Gaussian posterior distributions \cite{rizzi2012uncertainty, kulakova2017}. 

\citet{ambroladze2007tighter} also do not consider the category of problems in which the training and validation data-generating distributions are distinct; in this case, a subset of the training set is no longer a good proxy for unseen validation data. By defining a data-dependent PAC-Bayes prior, we are able to select an informative prior even when $\mathcal{D}_\T \neq \mathcal{D}_\V$. Last, we remedy the poor scaling of grid search for hyperparameter optimization by introducing a gradient-based algorithm in Section~\ref{sec:algorithm} that can efficiently optimize Eq.~\ref{eq:ho-stability} with respect to millions of hyperparameters.

\subsection{Generalization Theory} \label{sec:gentheory}
While we show in Section~\ref{sec:experiments} that the minimization of Eq.~\ref{eq:ho-stability} reduces out-of-sample error for typical choices of parameter optimization method, our analysis of Eq.~\ref{eq:bilevel-ho-e} and~\ref{eq:ho-stability} depends on the use of an iterative sampling method that introduces Gaussian noise at each iteration. We thus assume that the parameter optimization is performed with either Stochastic Gradient Langevin Dynamics (SGLD) or Langevin Dynamics (LD) \cite{welling2011bayesian}. We then establish that optimizing Eq.~\ref{eq:ho-stability} is equivalent to minimizing two PAC-Bayes bounds on the expected validation risk and the change in the training posterior induced by out-of-sample training data. Proofs of the results presented in this section are in Appendix~\ref{sec:genappendix}.

In classical PAC-Bayes theory, generalization error is upper-bounded by a KL divergence between the ``posterior'' distribution (constructed after observing the data set) and a data-independent ``prior'' distribution \cite{mcallester1999some}. The resulting bound can be quite large, however, as it is difficult to choose a prior that is similar to the posterior before observing the data. Our work extends previous attempts to define PAC-Bayes bounds with data-dependent priors. Naively applying these bounds to hyperparameter optimization would require excluding a subset of the validation data from optimization or using a Gaussian to approximate the training posterior distribution \cite{ambroladze2007tighter, parrado12data, dziugaite2018data}. These bounds are thus suboptimal when either the validation set is small, or a spherically symmetric Gaussian is a poor approximation to the true posterior distribution. 

To develop more informative bounds, we prove that a PAC-Bayes prior can be chosen using an algorithm that depends on the data set so long as the dependence is sufficiently weak. The approach we take is similar to that of \citet{dziugaite2018data}, but by using a weaker notion of data dependence, we avoid approximating the posterior with a Gaussian. We then derive a data-dependent PAC-Bayes bound and show that it can be applied to the expected validation risk. Next, we compute a tractable upper bound of the PAC-Bayes bound by expanding the KL divergence into a sum of Fisher distances; this recovers the objective in Eq.~\ref{eq:ho-stability}. Last, we show that the minimization of Eq.~\ref{eq:ho-stability} also controls the variability of $\theta^*(S_\T, \lambda)$ with respect to unseen $S^\prime_\T$ (i.e., to changes in the training set).

We first define a measure of data-dependence known as $(\epsilon, \delta)$-differential privacy (DP).
\begin{definition}[\citet{dwork2014algorithmic}]
Let $\epsilon, \delta \geq 0$ and let $\mathcal{A}$ be a randomized algorithm that takes a data set $S$ as input and produces a random output in some space $U$. We consider the application of the algorithm to two adjacent data sets $S_1$ and $S_2$ differing in only one element. Then, $\mathcal{A}$ is an $(\epsilon, \delta)$-DP algorithm if for any subset $I \subseteq U$
\begin{align*}
    P(\mathcal{A}(S_1) \in I) \leq e^\epsilon P(\mathcal{A}(S_2) \in I) + \delta
\end{align*}
holds for all choices of $S_1$ and $S_2$.
\end{definition}
Intuitively, a randomized algorithm is $(\epsilon, \delta)$-DP with respect to $S$ if the distribution over its output is (mostly) insensitive to the replacement of a single data point in $S$ with any other point. 

We build upon prior work from \citet{dziugaite2018data} and \citet{rivasplata19} by showing that a PAC-Bayes prior can be data-dependent, so long as samples from that prior distribution are  $(\epsilon, \delta)$-DP with respect to $S$. In our analysis, we exclude the use of any prior whose samples are the product of a non-trivial composition over the data set $S$. An $(\epsilon, \delta)$-DP algorithm that operates on $S$ is a non-trivial composition if it is the result of sequentially composing several algorithms $\mathcal{A}_i: S_i \subseteq S \to U$ such that $\cap_i S_i \neq \emptyset$ (i.e., the algorithms in the composition do not operate on disjoint subsets of $S$).

Here we present a data-dependent PAC-Bayes theorem that only holds for loss functions bounded in $[0,1]$. In Theorem~\ref{thm:app_pbapproxdp}, we extend this result to all bounded loss functions. To establish these results, we require an additional function, $\beta(\epsilon,\delta,s)$, that is defined in Appendix~\ref{sec:dpappendix}.
\begin{theorem}
We assume that the loss function $\ell$ is bounded in $[0,1]$ and that $P$ is chosen such that samples drawn from it are $(\epsilon, \delta)$-DP with respect to $S \sim \mathcal{D}^s$ without the use of any non-trivial composition. We require $\epsilon \in \left(0,\frac{1}{2}\right]$ and $\delta \in (0, \epsilon)$ such that $\beta(\epsilon, \delta, s) < \min \left (1, s \exp \left \{s \left( c_1\epsilon^2 + c_2 \sqrt{\frac{\delta}{\epsilon}} - 2\right)\right \} \right)$ for some positive constants $c_1$ and $c_2$.  Then, for $S \sim \mathcal{D}^s$ and for all distributions $Q$ over $\Theta$,
\begin{multline*}
    \E_Q [R(\theta; \mathcal{D}, \lambda)] \leq \E_Q[\hat{R}(\theta; S, \lambda)] + \left\{\frac{1}{s} \kld{Q}{P} + \right. \\
     \left.\frac{1}{s}\log \frac{5s}{\Delta} +  c_1\epsilon^2 + c_2\sqrt{\frac{\delta}{\epsilon}} \right\}^{1/2}
\end{multline*}
holds with probability at least $1 - \Delta$.
\label{thm:sqrtpbapproxdp}
\end{theorem}

If the restriction on $\beta$ assumed in Theorem~\ref{thm:sqrtpbapproxdp} is infeasible, we present a more general version of Theorem~\ref{thm:sqrtpbapproxdp} in Theorem~\ref{thm:app_pbapproxdp} that only requires $\beta(\epsilon, \delta, s) < 1$. We can now exploit the relaxed differential privacy requirements of Theorem~\ref{thm:sqrtpbapproxdp} to define a more informative prior $P$. We define $P$ to be the distribution of the $T$-th iterate of an $(\epsilon, \delta)$-DP SGLD algorithm (Algorithm~\ref{alg:dpsgld_onechain}) applied to the data set $S$, which we denote as $p(\theta^{(\epsilon, \delta)}_T | S, \lambda)$ \cite{wang2015privacy, li2017connecting}. The implementation details and privacy results regarding this $(\epsilon, \delta)$-DP version of SGLD are deferred to Appendix~\ref{sec:sgldappendix}.

Next, we apply the bound from Theorem~\ref{thm:sqrtpbapproxdp} to prove that optimizing Eq.~\ref{eq:ho-stability} minimizes the expected validation risk of parameters drawn from $p(\theta_T | S_\T, \lambda)$. To do so, we replace the missing subscripts with $\V$ and make the following substitutions: $Q \equiv p(\theta_T | S_\T, \lambda)$ and $P \equiv p(\theta^{(\epsilon, \delta)}_T | S_\V, \lambda)$. We assume that the Langevin samplers that yield $P$ and $Q$ are each initialized with $\theta_0 \sim P_0$, where $P_0$ is chosen independently of $S_\V$. Denoting the terms that do not depend on $\lambda$ as $A$, the bound can be rewritten as follows.
\begin{multline}
    \E_{p(\theta_T | S_\T, \lambda)}[R_\V(\theta; \mathcal{D}_\V, \lambda)] \leq \E_{p(\theta_T | S_\T, \lambda)}[\hat{R}_\V(\theta; S_\V, \lambda)] \\ 
    + \left\{ \frac{1}{n_\V} \kld{p(\theta_T | S_\T, \lambda)}{p(\theta^{(\epsilon, \delta)}_T | S_\V, \lambda)} + A \right\}^{1/2}\eqperiod
\label{eq:pb_bound}
\end{multline}
As currently proven, $A$ is large enough to make this bound vacuous. We speculate, however, that this constant can be dramatically reduced; the correlation of this regularizer with the true generalization error in Section~\ref{sec:experiments} lends empirical support to this claim.

Minimizing Eq.~\ref{eq:pb_bound} with respect to $\lambda$ is intractable because of the KL divergence; the normalizing constants of $p(\theta_T | S_\T, \lambda)$ and $p(\theta^{(\epsilon, \delta)}_T | S_\V, \lambda)$ are unknown. Next, we prove an upper bound that can be optimized.
\begin{lemma}[Proposition~2.6 of \citet{negrea2019information}]
Let $Q$ and $P$ be joint distributions of $\theta_0,\dots,\theta_T$, and let $Q_t$ and $P_t$ denote the associated marginal distributions of $\theta_t$. $Q_{t|}$ and $P_{t|}$ are defined as the distributions of $\theta_t$ when $Q_t$ and $P_t$, respectively, are both conditioned on $\theta_0,\dots,\theta_{t - 1} \sim Q_{0:(t-1)}$. Suppose that $Q_0 = P_0$. Then,
\begin{align*}
    \kld{Q_T}{P_T} \leq \sum_{t = 1}^T \E_{Q_{0:(t - 1)}} [\kld{Q_{t|}}{P_{t|}}]\eqperiod
\end{align*}
\vspace{-15pt}
\label{lma:klchainrule}
\end{lemma}
We apply the bound from Lemma~\ref{lma:klchainrule} to the KL divergence in Eq.~\ref{eq:pb_bound}. As we show in Corollary~\ref{cor:klbound}, the conditional distributions in the resulting inequality are Gaussian and, thus, KL divergences between them are tractable.
\begin{corollary}
We assume that the same constant step size $\eta$ is used to define both $p(\theta_T | S_\T, \lambda)$ and $p(\theta^{(\epsilon, \delta)}_T | S_\V, \lambda)$, $\hat{R}_{\V}(\theta; S_\V, \lambda)$ is $\gamma$-Lipschitz, and that both iterative methods are initialized with $\theta_0 \sim P_0$. Then, given that $\nu_t$ is the distribution of the $t$-th iterate of the Langevin sampler applied to $S_\T$,
\begin{multline*}
    \kld{p(\theta_T | S_\T, \lambda)}{p(\theta^{(\epsilon, \delta)}_T | S_\V, \lambda)} \leq B + \frac{n_\V \eta}{4} \times\\
    \sum_{t = 0}^{T - 1} \E_{\nu_t} \left [\norm{\nabla_\theta \hat{R}_\T(\theta; S^{(t)}_\T, \lambda) - \nabla_\theta \hat{R}_\V(\theta; S^{(t)}_\V, \lambda)}_2^2 \right]\eqcomma
\end{multline*}
for some constant $B(\eta,n_\T,n_\V,m)$.
\label{cor:klbound}
\end{corollary}
The identical step-size and Lipschitz assumptions are unrealistic for practical problems. We derive a tractable upper bound on $\kld{p(\theta_T | S_\T, \lambda)}{p(\theta^{(\epsilon, \delta)}_T | S_\V, \lambda)}$ in Appendix~\ref{sec:genappendix} that does not require either assumption. We discuss connections between this more general bound and hypotheses regarding the superiority of ``flat optima'' in Section~\ref{sec:relatedobs} \cite{hochreiter1997flat}. 

Recognizing that 
\begin{multline*}
    d^2_{2, \nu_t}(p^1(\theta | S^{(t)}_\T, \lambda), p^1(\theta | S^{(t)}_\V, \lambda)) = \\
    \E_{\nu_t} \left [\norm{\nabla_\theta \hat{R}_\T(\theta; S^{(t)}_\T, \lambda) - \nabla_\theta \hat{R}_\V(\theta; S^{(t)}_\V, \lambda)}_2^2 \right]\eqcomma
\end{multline*}
we apply Corollary~\ref{cor:klbound} and the triangle inequality to the right-hand-side of Eq.~\ref{eq:pb_bound} and exclude terms that do not depend on $\lambda$,
\begin{multline}
    \E_{p(\theta_T | S_\T, \lambda)}[\hat{R}_\V(\theta; S_\V, \lambda)] + \\
    \left\{\frac{\eta}{4} \sum_{t = 0}^{T-1} d^2_{2, \nu_t}(p^1(\theta | S^{(t)}_\T, \lambda), p^1(\theta | S^{(t)}_\V, \lambda))\right\}^{1/2}\eqperiod
\label{eq:ho-stability_cat1}
\end{multline}
Then, substituting $\zeta$ for $\sqrt{\frac{\eta}{4}}$ recovers the objective in Eq.~\ref{eq:ho-stability}.

Eq.~\ref{eq:ho-stability_cat1} establishes that the minimization of Eq.~\ref{eq:ho-stability} reduces the expected risk for out-of-sample validation data. Next, we consider the effect of using out-of-sample training data in the parameter optimization. To measure this impact, we would ideally compute the difference between the observed posterior distribution and the expected posterior distribution for unseen training data. We proceed by deriving a tractable, but inexact estimate of this difference. We model the expected posterior distribution by defining $p(\theta | \mathcal{D}_\T, \lambda) \propto \exp \left (-n_\T R_\T(\theta; \mathcal{D}_\T, \lambda) \right)$. The KL divergence, $\kld{p(\theta | S_\T, \lambda)}{p(\theta | \mathcal{D}_\T, \lambda)}$, then approximates the difference between the observed and expected posterior distributions. Using an argument similar to that of \citet{Lever2013}, we show below that minimizing the second term of Eq.~\ref{eq:ho-stability} also minimizes a PAC-Bayes bound on this quantity.

We begin by establishing that if the Langevin sampler for the training posterior is sufficiently converged, the difference between the expected and empirical training risk given $\theta \sim p(\theta_T | S_\T, \lambda)$ upper bounds $\kld{p(\theta | S_\T, \lambda)}{p(\theta | \mathcal{D}_\T, \lambda)}$.
\begin{lemma}
If $d^1_\mathcal{W}(p(\theta_T | S_\T, \lambda), p(\theta | S_\T, \lambda)) \leq \kappa$ and $\ell$ is $\gamma$-Lipschitz ,\footnote{\citet{xu2018global} provide rates of convergence in the 1-Wasserstein distance to the stationary distribution for both SGLD and LD.} then 
\begin{multline*}
     \frac{1}{n_\T}\kld{p(\theta | S_\T, \lambda)}{p(\theta | \mathcal{D}_\T, \lambda)} \leq 2\gamma \kappa +\\
     \E_{p(\theta_T | S_\T, \lambda)} [R_\T(\theta; \mathcal{D}_\T, \lambda) - \hat{R}_\T(\theta; S_\T, \lambda)]\eqperiod
\end{multline*}
\vspace{-20pt}
\label{lma:trainrw}
\end{lemma}
We can apply Theorem~\ref{thm:sqrtpbapproxdp} to minimize the bound in Lemma~\ref{lma:trainrw}. We move the empirical risk to the left-hand-side of the bound in Theorem~\ref{thm:sqrtpbapproxdp}, replace the missing subscripts with $\T$, and make the following substitutions: $Q \equiv p(\theta_T | S_\T, \lambda)$ and $P \equiv p(\theta^{(\epsilon, \delta)}_T | S_\V, \lambda)$. Denoting the terms that do not depend on $\lambda$ as $C$ results in the following bound:
\begin{multline}
    \E_{p(\theta_T | S_\T, \lambda)}[R_\T(\theta; \mathcal{D}_\T, \lambda) - \hat{R}_\T(\theta; S_\T, \lambda)] \leq\\
    \sqrt{ \frac{1}{n_\T} \kld{p(\theta_T | S_\T, \lambda)}{p(\theta^{(\epsilon, \delta)}_T | S_\V, \lambda)}  + C}\eqperiod
\label{eq:traingenbound}
\end{multline}
Corollary~\ref{cor:trainbound} follows by applying Lemma~\ref{lma:trainrw}, Corollary~\ref{cor:klbound}, and the triangle inequality to Eq.~\ref{eq:traingenbound}.

\begin{corollary}
Retain the assumptions of Corollary~\ref{cor:klbound} and Lemma~\ref{lma:trainrw}. Then, minimizing 
\begin{align*}
    \sqrt{\sum_{t = 0}^{T - 1} d^2_{2, \nu_t}(p^1(\theta | S^{(t)}_\T, \lambda), p^1(\theta | S^{(t)}_\V, \lambda))}
\end{align*}
is equivalent to minimizing an upper bound on 
\begin{align*}
    \kld{p(\theta | S_\T, \lambda)}{p(\theta | \mathcal{D}_\T, \lambda)}
\end{align*}
with respect to $\lambda$.
\label{cor:trainbound}
\end{corollary}
Corollary~\ref{cor:trainbound} establishes the desired result: Minimizing the regularizer in Eq.~\ref{eq:ho-stability} also minimizes our approximation to the difference between the in-sample and out-of-sample training posterior distributions.

\section{Algorithm} \label{sec:algorithm}
Here we show how to modify gradient-based optimization algorithms from the literature in order to minimize Eq.~\ref{eq:ho-stability}. To simplify comparison to prior work, we construct a one-sample estimator of the expectation in Eq.~\ref{eq:ho-stability} by only running a single optimization. Several previously developed methods can be used to efficiently minimize the empirical validation risk  \cite{domke12, maclaurin15, franceschi17a, shaban2018truncated, lorraine2019optimizing}. These methods can be used without modification to optimize the first term of Eq.~\ref{eq:ho-stability}.

Optimizing the regularizer in Eq.~\ref{eq:ho-stability} presents additional challenges. In our analysis below, we assume the typical setting in which the number of hyperparameters, $n$, does not dominate the number of parameters, $m$. Naively applying reverse-mode differentiation (RMD) to Eq.~\ref{eq:ho-stability} results in a prohibitively expensive algorithm with $O(mT + n)$ space complexity and $O(mT^2 + nT)$ time complexity \cite{rumelhart1986learning}. In comparison, methods developed to minimize the empirical validation risk can have as low as $O(m + n)$ space complexity and $O(mT + n)$ time complexity \cite{lorraine2019optimizing}. 

We make two modifications to the standard RMD algorithm to make computing a hyperparameter gradient feasible. To reduce the time complexity of computing a gradient with respect to the regularizer, we apply $K$-truncated-RMD to each summand in the regularizer \cite{shaban2018truncated}. Instead of backpropagating through all previous steps of the parameter optimization, we end the backpropagation at $\theta_{t - K}$ for the $(2,\nu_t)$-Fisher distance evaluated at $\theta_t$. This approximation reduces the time complexity of optimizing Eq.~\ref{eq:ho-stability} to $O(mKT + nT)$.

The space complexity of the algorithm, however, remains prohibitive; applying truncated RMD to each summand in the regularizer requires the storage of all optimization iterates. We make the following observation: The final hyperparameter gradient of the regularizer can be expressed as a scaled sum of gradients computed with respect to each summand. We can then compute the hyperparameter gradient by accumulating intermediate values in two auxiliary variables of dimension $n$ and $1$, respectively, at each iteration of the parameter optimization. 

\begin{algorithm}[ht!]
   \caption{Optimization of Eq.~\ref{eq:ho-stability}}
   \label{alg:ho_stability_alg_mt}
\begin{algorithmic}
   \STATE {\bfseries Input:} $S_\T$, $S_\V$, $T$, $\{\eta_t\}_{t = 0}^{T-1}$, $P_0$, $\lambda_0$, $K$
   \STATE {\bfseries Output:} $\lambda$
   \WHILE{not converged}
   \STATE $\theta_0 \sim P_0$
   \STATE $X \leftarrow 0, Y \leftarrow 0$
   \FOR{$t=0$ {\bfseries to} $T-1$}
   \STATE $\theta_{t + 1} \leftarrow \theta_t - \eta_t \nabla_\theta \hat{R}_\T(\theta_t; S^{(t)}_\T, \lambda)$
   \STATE $X \leftarrow X + \frac{\zeta}{2} \widetilde{\nabla}^K_\lambda \left\lVert \nabla_\theta \hat{R}_\T(\theta_t; S^{(t)}_\T, \lambda) \right .$ \\
   \hfill$ \left . - \nabla_\theta \hat{R}_\V(\theta_t; S^{(t)}_\V, \lambda) \right \rVert_2^2$
   \STATE $Y \leftarrow Y + \left \lVert\nabla_\theta \hat{R}_\T(\theta_t; S^{(t)}_\T, \lambda)  \right.$ \\
   \hfill $\left . - \nabla_\theta \hat{R}_\V(\theta_t; S^{(t)}_\V, \lambda) \right \rVert_2^2$
   \ENDFOR
   \STATE $\lambda \leftarrow \lambda - \widetilde{\nabla}^K_\lambda \hat{R}_\V(\theta_{T}; S_\V, \lambda) - X / \sqrt{Y}$ 
   \ENDWHILE
\end{algorithmic}
\end{algorithm}

Algorithm~\ref{alg:ho_stability_alg_mt} summarizes the steps required for optimization of Eq.~\ref{eq:ho-stability}. We denote the auxiliary variables using $X$ and $Y$ and the $K$-truncated-RMD gradient using $\widetilde{\nabla}_\lambda^K$. \citet{shaban2018truncated} show that under certain regularity conditions, even choosing $K=1$ yields gradient estimates that are sufficient descent directions for hyperparameter optimization. In Figure~\ref{fig:k_comp}, we observe that either selecting $K = 1$ or ignoring the implicit dependence of the summands on the inner optimization still minimizes the regularizer. Both choices result in space and time complexities of $O(m + n)$ and $O((m + n)T)$, respectively. 

Despite the apparent computational burden of optimizing a PAC-Bayes bound, this algorithm matches the optimal space and time complexity of any gradient-based algorithm that only minimizes the empirical validation risk. When we do not anticipate re-using hyperparameters for additional training set optimizations, we can further accelerate the optimization and reduce memory overhead by way of online learning of the hyperparameters. We show in Appendix~\ref{sec:algappendix} how one can view optimization of the regularizer as an instance of online convex optimization \cite{hazan2016introduction}.
\section{Results and Discussion}
Code to reproduce all experiments discussed is available at \href{https://github.com/DEShawResearch/PACBayesHyperOpt}{https://github.com/DEShawResearch/PACBayesHyperOpt}. 
\label{sec:experiments}
\subsection{Feature Selection} \label{sec:featselection}
Freedman's paradox illustrates the difficulty of preventing overfitting in hyperparameter optimization even for problems that do not involve deep neural networks \cite{freedman1983note}. We consider two versions of Freedman's problem. In the first, we generate $500$ input-label pairs, $\{(x_i, y_i)\}_{i = 1}^{500}$ where $x_i \sim \mathcal{N}(\mathbf{0_{500}}, \mathbb{I}_{\mathbf{500}})$ and $y_i \sim \mathcal{N}(0,1)$, which we split into equally sized training and validation sets. This data-generation procedure ensures that no model fit to this data set can be accurate for additional data sampled from the data-generating distribution. Our second version is more realistic: We introduce two input features to the prior experiment that have \emph{true} nonzero correlation with the labels. Details of this data-generating process are in Appendix~\ref{sec:featappendix}.

In both variations of the experiment, we do not utilize a gradient-based algorithm for hyperparameter optimization, and instead perform stepwise forward selection on the features of a linear model. Forward selection using Eq.~\ref{eq:bilevel-ho} results in a model that includes an arbitrarily large set of predictors; spuriously correlated features improve validation set goodness-of-fit. In Figure~\ref{fig:baddecision_null} (for the first experiment) and Figure~\ref{fig:baddecision} (for the second experiment), we compare the optimization objective to test set mean-squared-error (MSE) for models selected using Eq.~\ref{eq:bilevel-ho}. The supposed ``best'' model performs poorly, and, in fact, better model performance on the validation data (as measured by the hyperparameter objective, the right axis) correlates with worse performance on out-of-sample data (test set MSE, the left axis).

In contrast, when we use Eq.~\ref{eq:ho-stability} to select features, we choose models with smaller out-of-sample error than those chosen using Eq.~\ref{eq:bilevel-ho}. For both experiments, we set $\zeta = \eta/4$. For the first experiment, we show in Figure~\ref{fig:gooddecision_null} that the regularized objective is correlated with the test set MSE. In Figure~\ref{fig:gooddecision}, the optimal model includes only the truly predictive features of the second experiment. Forward selection using the regularized objective results in an accurate assessment of the poor out-of-sample performance of models containing additional features.

\begin{figure}[ht!]
     \centering
     \begin{subfigure}[b]{\columnwidth}
         \centering
         \includegraphics[width=\columnwidth]{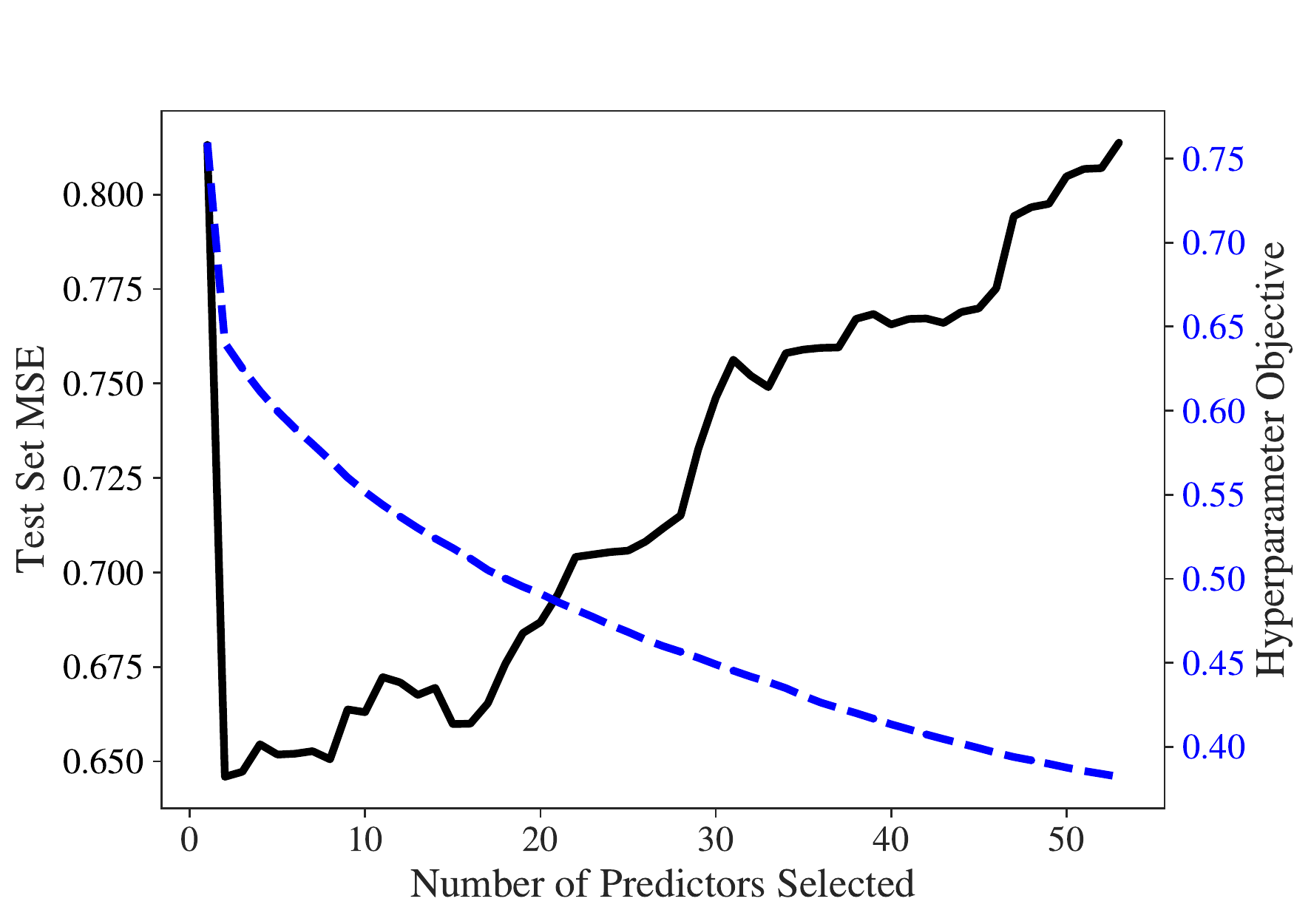}
         \caption{Test set MSE vs. objective for Eq.~\ref{eq:bilevel-ho}.}
         \label{fig:baddecision}
     \end{subfigure}
     \hfill
     \begin{subfigure}[b]{\columnwidth}
         \centering
         \includegraphics[width=\columnwidth]{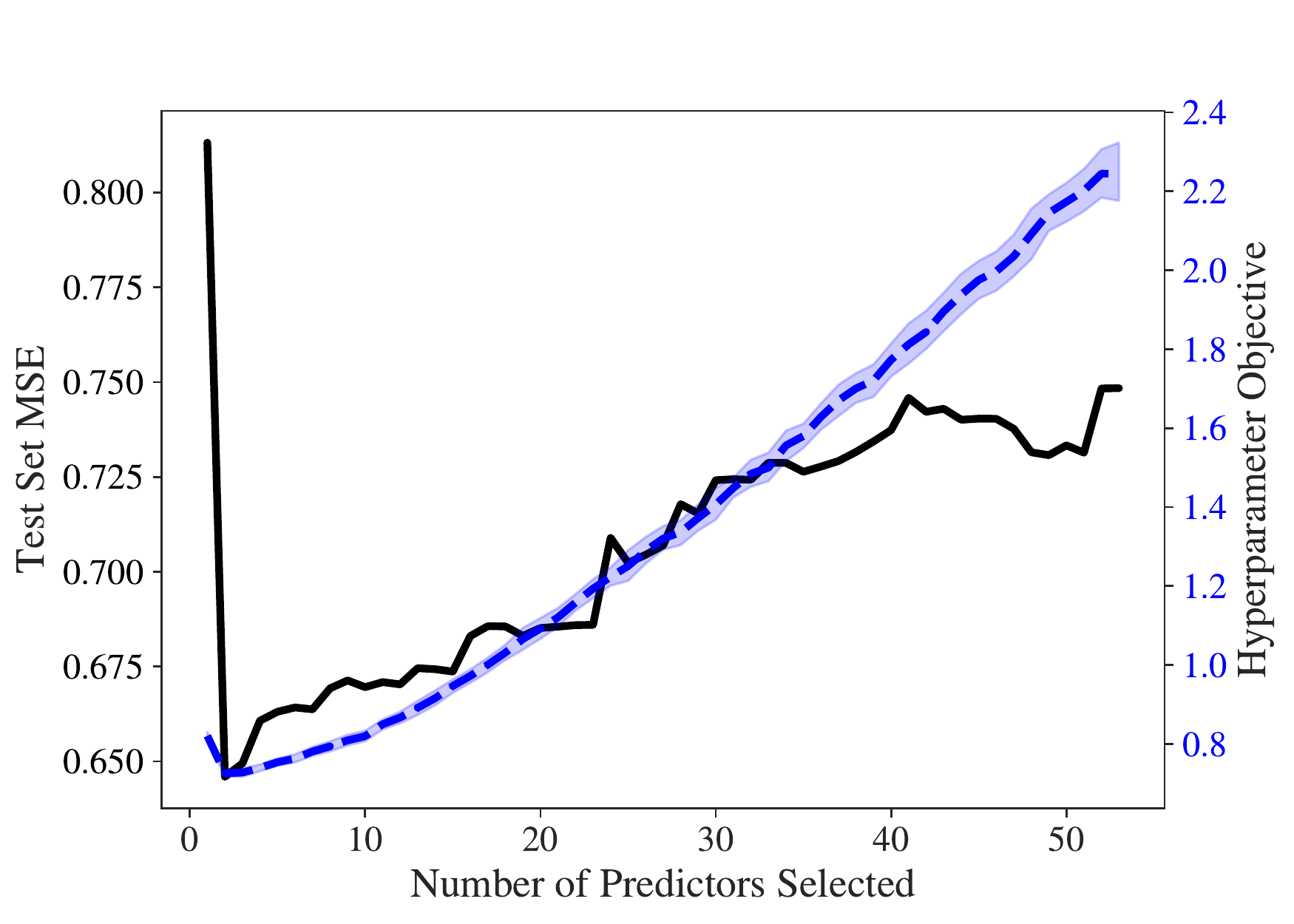}
         \caption{Test set MSE vs. objective for Eq.~\ref{eq:ho-stability}. The shaded regions are $95$\% confidence intervals constructed from $50$ optimizations initialized from randomly sampled parameters.}
         \label{fig:gooddecision}
     \end{subfigure}
     \caption{Comparing objectives to test set MSE for Freedman's paradox with two true predictors. }
     \label{fig:freedman_nonnull}
\end{figure}
\subsection{Regularization Penalty} \label{sec:regpenalty}
Next, we study an example of validation set overfitting first described in \citet{lorraine2019optimizing}. For both the MNIST and CIFAR-10 data sets, we construct a training and validation data set that includes $50$ randomly sampled images each \cite{lecun1998gradient, krizhevsky2009learning}. We then fit several classifiers to these training sets; in this section, we present results for a one-layer fully connected network (i.e., a linear classifier) and ResNet-18 \cite{he2016deep}. For each parameter in these classifiers, we introduce a weight decay hyperparameter \cite{loshchilov2018fixing}. For ResNet-18, this results in the optimization of over $12$ million hyperparameters. We provide additional details regarding the experimental setup in Appendix~\ref{sec:regappendix}.

For both the regularized and unregularized objective, we compute the hyperparameter gradient of the empirical validation risk using the $T1-T2$ approximation proposed in \citet{Luketina16}. We then minimize Eq.~\ref{eq:bilevel-ho} using gradient descent. To optimize Eq.~\ref{eq:ho-stability}, we apply Algorithm~\ref{alg:ho_stability_alg_mt} with $K = 0$. We use grid search to select the $\zeta$ that minimizes out-of-sample error, though we show in Appendix~\ref{sec:regappendix} that our results are qualitatively unchanged for a wide range of penalties. We describe the remaining optimizer settings in Appendix~\ref{sec:regappendix}.

By using small training and validation data sets, we guarantee that minimization of the empirical validation risk alone leads to dramatic overfitting. All optimized models achieve near-zero validation loss and at least 98\% top-1 validation accuracy, but the true out-of-sample accuracy of these models is substantially lower.

Evaluation on out-of-sample data reveals substantial differences between the models that result from Eq.~\ref{eq:bilevel-ho} and from Eq.~\ref{eq:ho-stability}.  In the former case, overfitting to the validation set causes substantial degradation in test set accuracy (Figure~\ref{fig:weight_decay}, dashed lines). Conversely, the models resulting from Eq.~\ref{eq:ho-stability} (Figure~\ref{fig:weight_decay}, solid lines) exhibit superior test set accuracy, which remains stable over the course of optimization. Figures~\ref{fig:linear_mnist_reg_offline_mt} and \ref{fig:resnet18_cifar10_reg_offline_mt} corroborate the theoretical connection we make between Eq.~\ref{eq:ho-stability} and generalization error bounds. These plots show a strong positive correlation between the value of the regularizer and the generalization error; the correlation remains positive even for the largest values of $\zeta$ tested. Appendix~\ref{sec:regappendix} includes additional results for experiments with other classifier-dataset pairs. 

We also compare Eq.~\ref{eq:ho-stability} to a simpler, problem-specific heuristic for preventing overfitting when optimizing Eq.~\ref{eq:bilevel-ho}; namely, we select the classifier with smallest weight-norm that also achieves the maximum top-1 validation accuracy. Averaging across all 6 classifier-dataset pairs that we evaluate, we show in Table~\ref{tbl:baseline} that the optimization of Eq.~\ref{eq:ho-stability} improves final test accuracy by 12\% when compared to this heuristic.
\begin{figure}[ht!]
     \centering
     \begin{subfigure}{\columnwidth}
         \centering
         \includegraphics[width=\columnwidth]{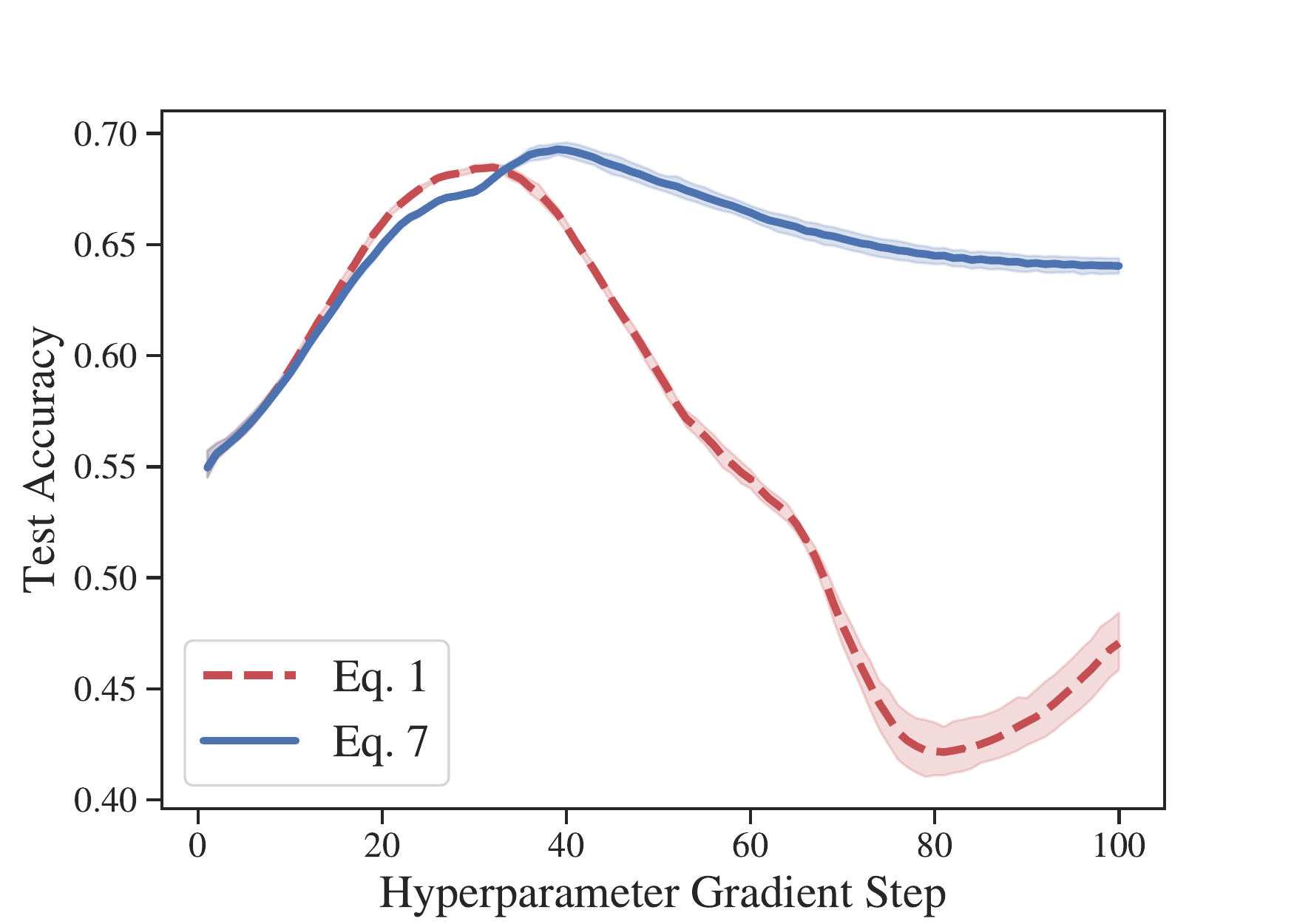}
         \caption{Fitting weight decays for a linear classifier on MNIST.}
         \label{fig:linear_mnist}
     \end{subfigure}
     \hfill
     \begin{subfigure}{\columnwidth}
         \centering
         \includegraphics[width=\columnwidth]{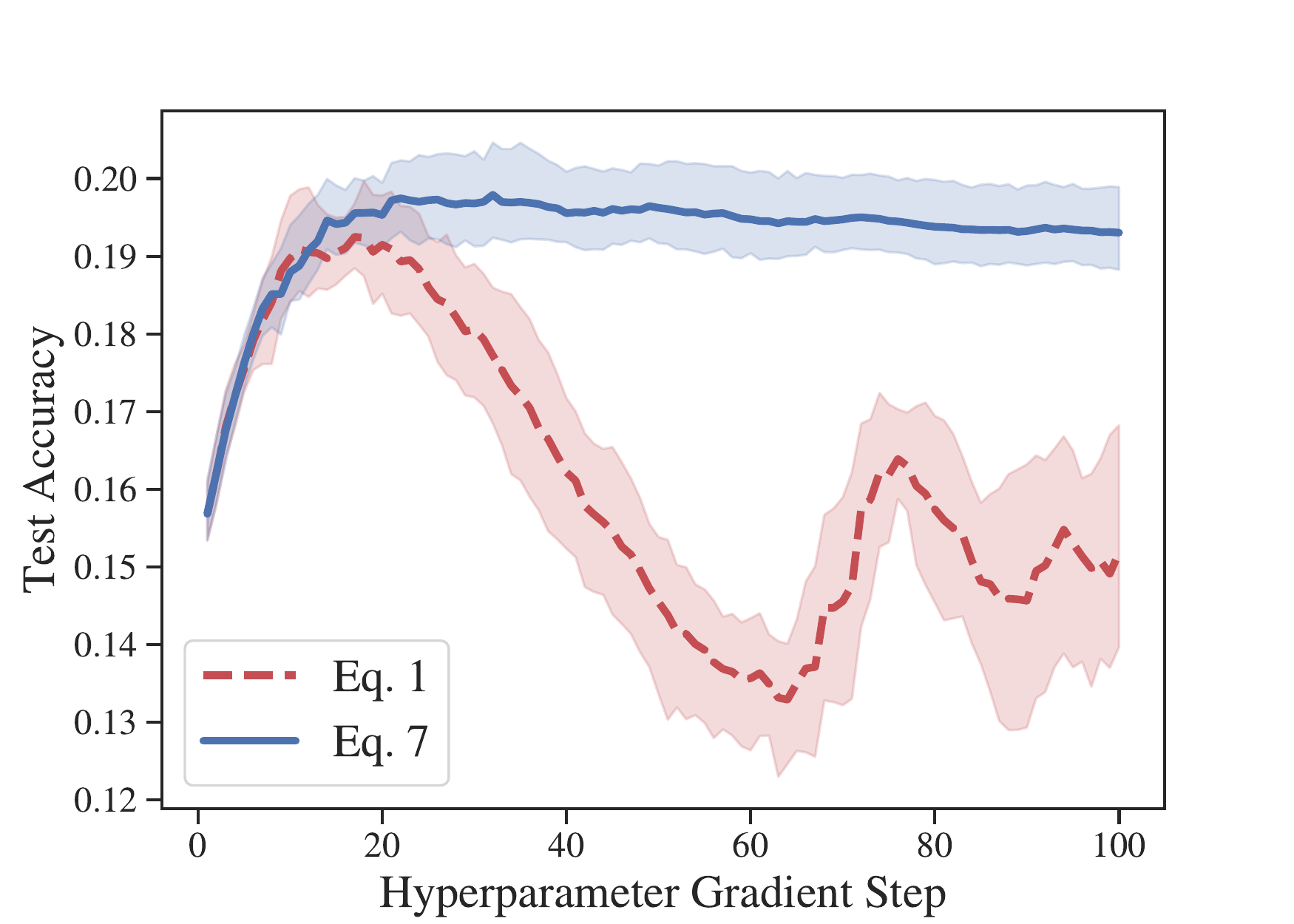}
         \caption{Fitting weight decays for ResNet-18 on CIFAR-10.}
         \label{fig:resnet18_cifar-10}
     \end{subfigure}
        \caption{Overfitting a validation set with per-parameter weight decays. The shaded regions are $95$\% confidence intervals constructed from five optimizations initialized from randomly sampled parameters.}
        \label{fig:weight_decay}
\end{figure}
\begin{figure}[ht!]
     \centering
     \begin{subfigure}{\columnwidth}
         \centering
         \includegraphics[width=\columnwidth]{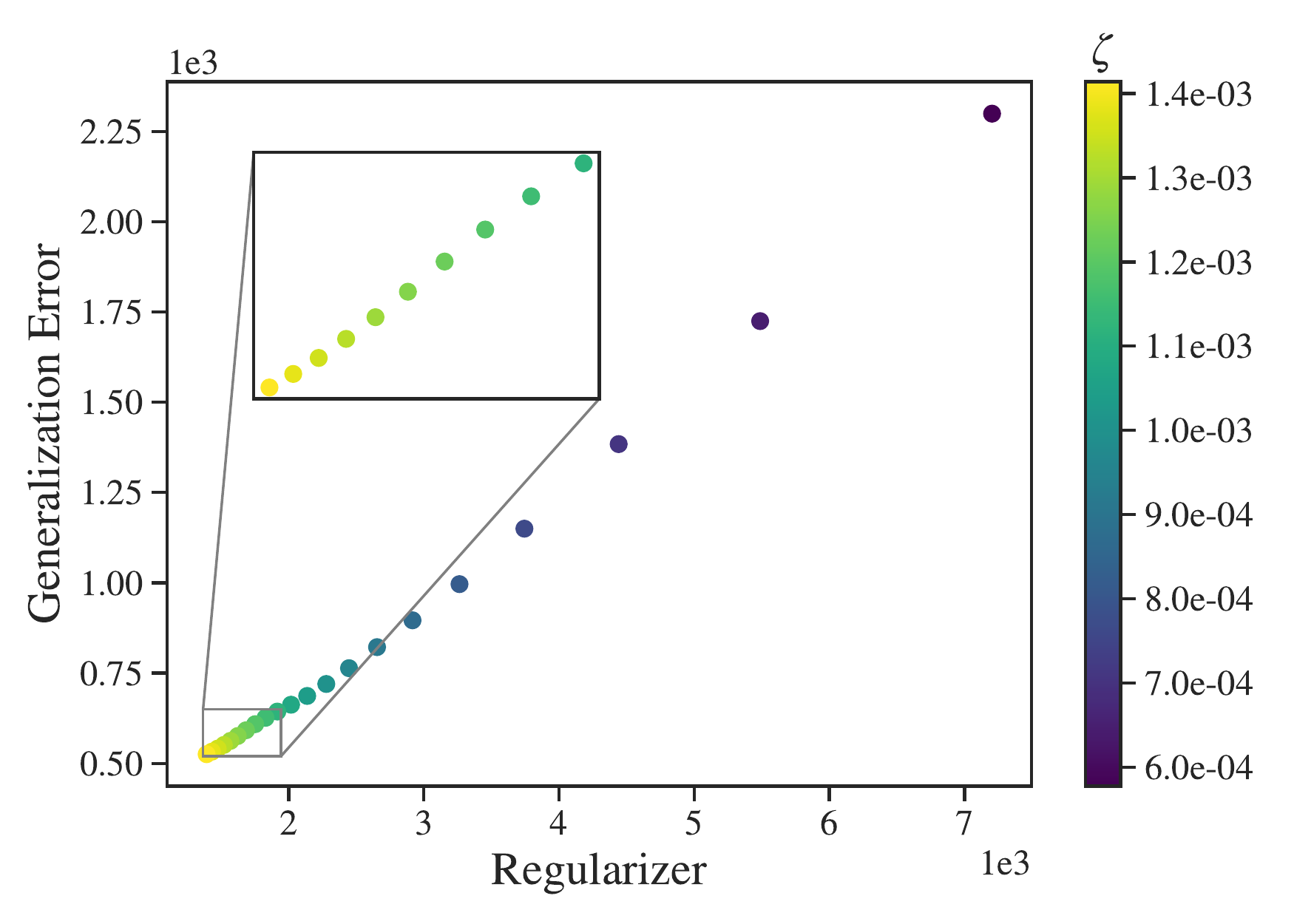}
         \caption{Error-regularizer correlation for a linear classifier on MNIST.}
         \label{fig:linear_mnist_reg_offline_mt}
     \end{subfigure}
     \hfill
     \begin{subfigure}{\columnwidth}
         \centering
         \includegraphics[width=\columnwidth]{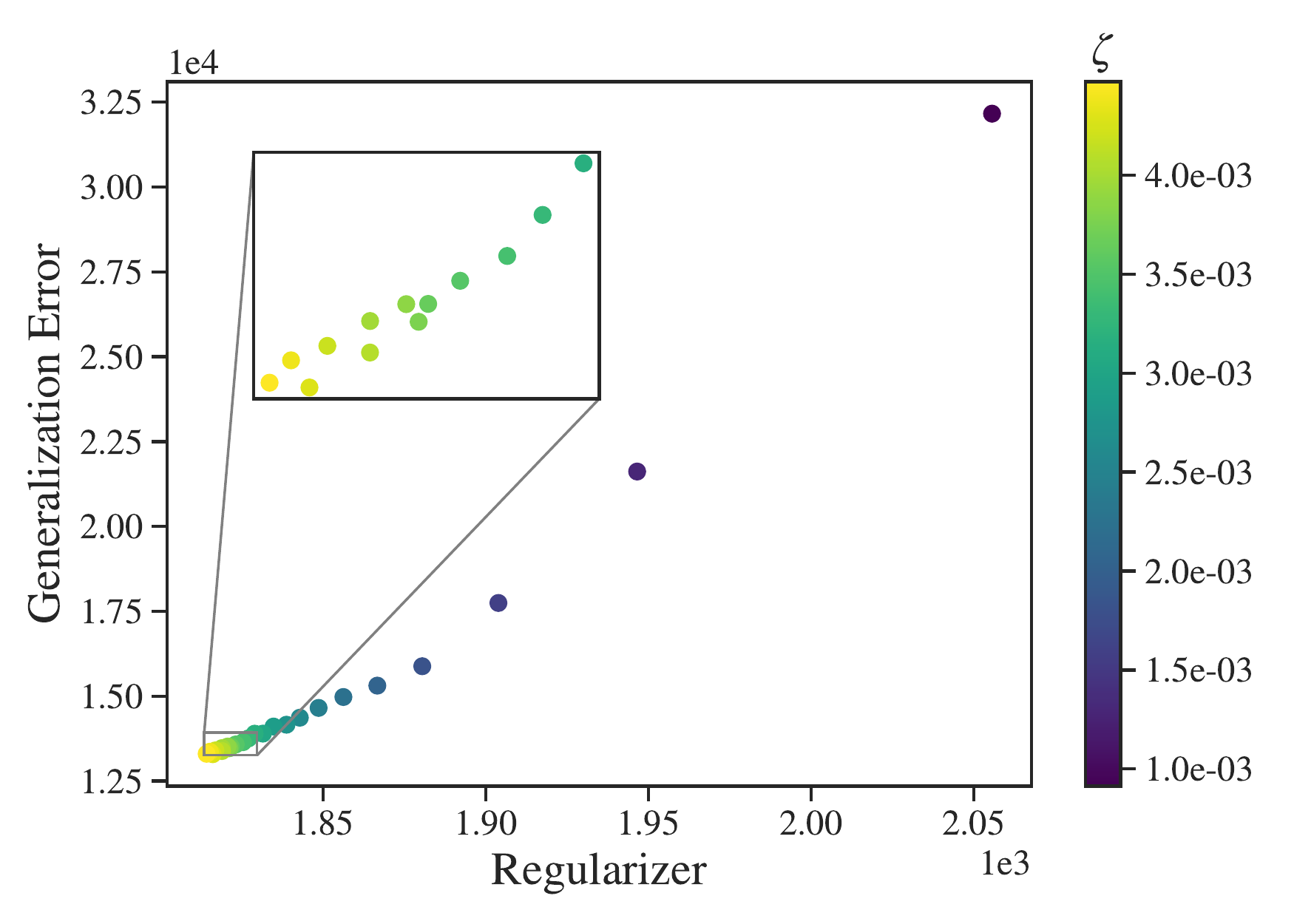}
         \caption{Error-regularizer correlation for ResNet-18 on CIFAR-10.}
         \label{fig:resnet18_cifar10_reg_offline_mt}
     \end{subfigure}
     \caption{Comparing generalization error to the value of the regularizer in Eq.~\ref{eq:ho-stability} using Algorithm~\ref{alg:ho_stability_alg_mt}.}
     \label{fig:reg_gen_offline}
\end{figure}

\subsection{Related Observations} \label{sec:relatedobs}
Our PAC-Bayes bound can be used to explain the success of recently developed methods in neural architecture search and meta-learning. Although gradient-based neural architecture search methods can be used to minimize the empirical validation risk, the architectures discovered by way of these methods do not generalize well to unseen data \cite{LiTalwalkar19}. \citet{zela2020understanding} identify a strong positive correlation between the dominant eigenvalue of the validation Hessian at the training optimum (i.e., a proxy for optima sharpness) and generalization error of the trained architecture. They then advocate for the use of heuristic methods that resemble the minimization of a PAC-Bayes bound with a Gaussian prior centered at 0 in order to reduce this eigenvalue and improve generalization.

Observation~\ref{obs:app_trhessian} provides an explanation for \citet{zela2020understanding}'s findings. When the privacy parameters, $\epsilon$ and $\delta$ are chosen to be near $0$, we show that minimizing the PAC-Bayes bound in Theorem~\ref{thm:sqrtpbapproxdp} is equivalent to minimizing another measure of loss curvature: the trace of the training Hessian. The method we describe in Algorithm~\ref{alg:ho_stability_alg_mt} can either be modified to minimize this quantity or can be used without modification to minimize Eq.~\ref{eq:ho-stability}. By directly optimizing the quantity of interest, the use of Algorithm~\ref{alg:ho_stability_alg_mt} could further reduce overfitting in network architecture search.

\citet{guiroy2019towards} study meta-learning problems in which the hyperparameter learned is the parameter initialization for unseen tasks. They evaluate several proxies for generalization to motivate their choice of regularizer. We show that these proxies can be motivated using the data-dependent PAC-Bayesian bound derived in Theorem~\ref{thm:sqrtpbapproxdp}. Rigorously proving that the methods presented by \citet{guiroy2019towards} are \emph{approximations} to PAC-Bayes bounds would require extending our analysis to the meta-learning PAC-Bayes bound derived by \citet{AmitMeir2017}. Though we leave this extension to future work, we believe that the regularizer presented in \citet{guiroy2019towards} can be improved by taking an approach similar to that of Section~\ref{sec:gentheory}. 

\section{Conclusion}
As large-scale hyperparameter optimization becomes a routine part of machine learning practice, it will also become increasingly important to prevent validation set overfitting. In this paper, we introduce a theoretical framework and practical method for selecting hyperparameters that generalize. We propose a new hyperparameter optimization objective, which we show to be a data-dependent PAC-Bayes bound on the generalization error to unseen training and validation data. The empirical results we obtain from minimizing this tractable measure of gradient incoherence indicate that we have developed a promising method for hyperparameter optimization problems susceptible to overfitting. The connections we draw between our framework and previously implemented heuristics suggest that our approach can explain and improve upon existing methods.

\section*{Acknowledgments}
We thank Kayvon Tabrizi, Hunter Nisonoff, Rian Kormos, and Benjamin Cherian for helpful discussions; and Berkman Frank for editorial assistance.
\bibliography{efficient_hyper}
\bibliographystyle{icml2020}
\onecolumn
\appendix
\section{Generalization Theory} \label{sec:genappendix}
In this section, we prove the results that we present in Section~\ref{sec:gentheory}. We first prove the data-dependent bound that we state in Theorem~\ref{thm:sqrtpbapproxdp} of the main text in Appendix~\ref{sec:dpappendix}. We then describe the prior-selection algorithm we use to derive the objective in Eq.~\ref{eq:ho-stability} and establish how one might compute the privacy parameters, $\epsilon$ and $\delta$, for this algorithm (Appendix~\ref{sec:sgldappendix}). These results then allow us to prove the remaining claims of Section~\ref{sec:gentheory} in Appendix~\ref{sec:resultappendix}. We elaborate upon the implications of our work in Appendix~\ref{sec:obsappendix}.

\subsection{Data-Dependent Bounds} \label{sec:dpappendix}
We first restate a generic PAC-Bayes bound that makes no assumptions about either the data-dependence of the PAC-Bayes prior, $P$, or boundedness of the loss function, $\ell$. This generic PAC-Bayes bound is defined in terms of the so-called ``exponential moment.'' We define this quantity as a function of both $f: (\mathcal{X} \times \mathcal{Y})^s \times \Theta \to \mathbb{R}$ and some distribution $P$ over $\Theta$:
\begin{align}
    \xi(f, P) \defeq \iint \exp \left \{ f(\theta, S) \right \} P(d\theta) \mathcal{D}^s(dS) \label{eq:xidef}\eqperiod
\end{align}
If $P$ depends on $S$, we denote the probability distribution as $P_S(d\theta)$.
Many popular PAC-Bayes bounds can be re-derived by bounding $\xi(f, P)$ for different choices of $f$ and $P$ and applying the following result.
\begin{theorem}[Theorem~2 in \citet{rivasplata19}]
Let $f(\theta, S) \defeq F(R(\theta; \mathcal{D}, \lambda), \hat{R}(\theta; S, \lambda))$ for any convex function $F$: $\mathbb{R}^2 \to \mathbb{R}$. Given some distribution $P$ over $\Theta$ and $S \sim \mathcal{D}^s$,
\begin{align*}
    F\left (\E_Q [R(\theta; \mathcal{D}, \lambda)], \E_Q[\hat{R}(\theta; S, \lambda)] \right) \leq \kld{Q}{P} + \log \left ( \frac{\xi(f, P)}{\Delta} \right)
\end{align*}
holds with probability at least $1 - \Delta$. 
\label{thm:app_genericbound}
\end{theorem}

When the prior $P$ is chosen independently of the data set $S$, deriving a bound on $\xi(f, P)$ for bounded $\ell$ is straightforward; regularizing the hyperparameters, however, so that the training posterior $Q$ is similar to data-independent $P$ can significantly harm the in-sample performance of parameters sampled from $Q$ \cite{dziugaite2018data}. Allowing $P$ to depend on $S$ would limit the negative effect of regularization on the empirical validation risk, but it would also invalidate the proof strategies used in prior work for bounding $\xi(f, P)$. 

Earlier work in data-dependent PAC-Bayes bounds required $P$ to be $(\epsilon,0)$-DP with respect to $S$ \cite{dziugaite2018data}. Using the approach suggested by \citet{rivasplata19}, we relax this requirement and only assume that samples $\theta \sim P$ are $(\epsilon,\delta)$-DP with respect to $S$. The approach we take mirrors that of \citet{rivasplata19}, though we first must reconsider their Lemma~7, leading to a correction of their result statement and proof.. Before we can derive a data-dependent PAC-Bayes bound, we thus require several preliminary definitions and results.

\begin{definition}[\citet{dwork2015generalization}]
Let $\beta \geq 0$ and let $X$ and $Y$ be random variables in arbitrary measurable spaces, and let $X^\prime$ be independent of $Y$ and equal in distribution to $X$. The $\beta$-approximate max-information between $X$ and $Y$, denoted $I^{\beta}_\infty(X; Y)$ is the smallest value $k$ such that, for all product-measurable events $E$,
\begin{align*}
    P((X, Y) \in E) \leq e^k P((X^\prime, Y) \in E) + \beta\eqperiod
\end{align*}
Then, for an algorithm $\mathcal{A}$ mapping from $(\mathcal{X} \times \mathcal{Y})^s \to T$, the $\beta$-approximate max-information of $\mathcal{A}$, denoted $I^{\beta}_\infty(\mathcal{A}, s)$, is the least value $k$ such that for all probability measures on $\mathcal{X} \times \mathcal{Y}$, $I_\infty^{\beta}(S; \mathcal{A}_S) \leq k$ when $S \sim \mathcal{D}^s$. 
\label{def:maxinfo}
\end{definition}
\citet{dwork2015generalization} show that an $(\epsilon,0)$-DP algorithm necessarily has bounded $\beta$-approximate max-information. \citet{dziugaite2018data} then use this result to derive a data-dependent PAC-Bayes bound. To prove our data-dependent PAC-Bayes result, we rely on a max-information bound for $(\epsilon, \delta)$-DP algorithms. 

Theorem~\ref{thm:maxinfobound} restates a max-information bound for $(\epsilon,\delta)$-DP algorithms. We follow \citet{rogers2016max} and define 
\begin{align}
    \beta(\epsilon, \delta, s) \defeq \exp\left\{-s \epsilon^2 \right\} + O\left(s\sqrt{\frac{\delta}{\epsilon}}\right) \label{eq:betadef}\eqperiod
\end{align}

\begin{theorem}[Theorem~3.1 in \citet{rogers2016max}]
For $\epsilon \in (0, 1/2]$ and $\delta \in (0, \epsilon)$, let $\mathcal{A}: \left(\mathcal{X} \times \mathcal{Y} \right)^s \to \Theta$ be an $(\epsilon, \delta)$-DP algorithm. Then, 
\begin{align*}
    I^{\beta(\epsilon, \delta, s)}_{\infty}(\mathcal{A}, s) \leq O\left (s\epsilon^2  + s \sqrt{\frac{\delta}{\epsilon}} \right)\eqcomma
\end{align*}
so long as the input data set, $S$, is sampled from a product distribution.
\label{thm:maxinfobound}
\end{theorem}

\begin{lemma}[Lemma~7 in \citet{rivasplata19}]
Fix $f: (\mathcal{X} \times \mathcal{Y})^s \times \Theta \to \mathbb{R}$. Then for any data-dependent distribution $P_S$ over parameters $\Theta$, and for any $\beta \in (0, 1)$, the following bound on $\xi(f, P_S)$ holds,
\begin{align*}
    \xi(f, P_S) \leq \xi_{bd}(f) \exp \left \{I^{\beta}_\infty (P_S, s) \right \} + \beta \exp \left \{||f(S, \theta)||_\infty \right \}\eqcomma
\end{align*}
given $\xi_{bd}(f) \defeq \sup_{P^\prime} \xi(f, P^\prime)$ where the supremum is taken over data-independent distributions $P^\prime$.
\label{lma:maxinfopb}
\end{lemma}

Before we prove Lemma~\ref{lma:maxinfopb}, we establish a preliminary result that addresses the error made in the statement and proof of Lemma~7 in \citet{rivasplata19}.
\begin{lemma}
If $I^{\beta}_\infty(X; Y) \leq k$, then for any function $g$: $\mathcal{X} \times \mathcal{Y} \to \mathbb{R}^+ \cup \{0\}$,
\begin{align*}
    \int_{\mathcal{X}} \int_{\mathcal{Y}} g(x, y) P_x(dy)P(dx) \leq e^k \int_{\mathcal{X}} \int_{\mathcal{X}} \int_\mathcal{Y} g(x, y) P_{x^\prime}(dy)P(dx) P(dx^\prime) + \beta ||g||_\infty \eqperiod
\end{align*}
Let $||g||_\infty \defeq \sup_{\mathcal{X} \times \mathcal{Y}} g(x, y)$. 
\label{lma:amshiftbound}
\end{lemma}
\begin{proof}
Note that, for any $\beta \geq 0$,
\begin{align*}
    \int_{\mathcal{X}} \int_{\mathcal{Y}} g(x, y) P_x(dy)P(dx) &= \int_{0}^{||g||_\infty} P(g(X, Y) \geq t) dt  \\
    &\leq \int_{0}^{||g||_\infty} \left(e^k P\left (g(X^\prime, Y) \geq t \right) + \beta \right)  dt \\
    &= e^k \int_{0}^{||g||_\infty} P\left (g(X^\prime, Y) \geq t \right) dt + \beta ||g||_\infty \\
    &= e^k \int_{\mathcal{X}} \int_{\mathcal{Y}} g(x^\prime, y) \left (\int_{\mathcal{X}} P_x(dy) P(dx)\right) P(dx^\prime) + \beta ||g||_\infty \\
    &= e^k \int_{\mathcal{X}} \int_{\mathcal{X}} \int_{\mathcal{Y}} g(x, y) P_{x^\prime}(dy) P(dx) P(dx^\prime) + \beta ||g||_\infty\eqperiod
\end{align*}
The first inequality follows from the definition of the max-information. The last equality follows from Tonelli's theorem and the fact that $X^\prime \stackrel{d}{=} X$.
\end{proof}

We can now prove Lemma~\ref{lma:maxinfopb} using Lemma~\ref{lma:amshiftbound}.
\begin{proof}[Proof of Lemma~\ref{lma:maxinfopb}]
For notational convenience, we define $\mathcal{Z} \defeq (\mathcal{X} \times \mathcal{Y})^s$. We motivate our proof strategy as follows. Classical PAC-Bayes bounds on $\xi(f, P)$ are derived by interchanging the order of integration in $\xi(f, P)$'s definition (Eq.~\ref{eq:xidef}). For data-dependent PAC-Bayes bounds, this interchange is not possible as $P$ depends on $S$. Instead, we use the definition of the max-information to bound $\xi(f, P_S)$ in terms of an integral in which the data, $S$, upon which the prior depends is replaced with an independent and identically-distributed $S^\prime$. This replacement allows us to apply the interchange of integral approach to bounding $\xi(f, P_S)$.
\begin{align*}
\xi(f, P_S) &= \int_{\mathcal{Z}} \int_\Theta e^{f(S, \theta)} P_S(d\theta) \mathcal{D}^s(dS) \\
&\leq e^{I^\beta(P_S, s)} \int_{\mathcal{Z}} \int_{\mathcal{Z}} \int_{\Theta} e^{f(S, \theta)} P_{S^\prime}(d\theta) \mathcal{D}^s(dS)\mathcal{D}^s(dS^\prime) + \beta e^{||f(S, \theta)||_\infty} \\
&\leq e^{I^\beta(P_S, s)} \int_\mathcal{Z} \xi_{bd}(f) \mathcal{D}^s(dS^\prime) + \beta e^{||f(S, \theta)||_\infty} \\
&= e^{I^\beta(P_S, s)} \xi_{bd}(f) + \beta e^{||f(S, \theta)||_\infty}\eqperiod
\end{align*}
The first inequality follows from Lemma~\ref{lma:amshiftbound}. The second follows from the definition of $\xi_{bd}(f)$ in the lemma statement.
\end{proof}

Note that this result differs from the result proven in Lemma~7 of \citet{rivasplata19} due to the change in the prefactor applied to $\beta$ from $1$ to $e^{||f(S, \theta)||_\infty}$. 

In Theorem~\ref{thm:app_pbapproxdp}, we use Theorem~\ref{thm:maxinfobound} and Lemma~\ref{lma:maxinfopb} to prove a more general version of the data-dependent PAC-Bayes bound stated in Theorem~\ref{thm:sqrtpbapproxdp}. 

\begin{theorem}
We assume that $\ell$ is bounded in $[0,1]$ and that $P$ is chosen such that samples drawn from it are $(\epsilon, \delta)$-DP with respect to $S \sim \mathcal{D}^s$ without the use of any non-trivial composition. We require $\epsilon \in \left(0,\frac{1}{2}\right]$ and $\delta \in (0, \epsilon)$ such that $\beta(\epsilon, \delta, s) < 1$.  Then, for $S \sim \mathcal{D}^s$ and for all distributions $Q$ over $\Theta$,
\begin{multline*}
    \E_Q [R(\theta; \mathcal{D}, \lambda)] \leq \E_Q[\hat{R}(\theta; S, \lambda)] + \\ \left\{\frac{1}{2s-1} \left (\kld{Q}{P} + \log \frac{4s\exp \left \{s \left ( c_1 \epsilon^2 + c_2 \sqrt{\frac{\delta}{\epsilon}} \right) \right \} + \beta(\epsilon, \delta, s) \exp\{2s\}}{\Delta} \right) \right\}^{1/2}
\end{multline*}
holds with probability at least $1 - \Delta$ for some positive constants $c_1$ and $c_2$.

If we further choose $\epsilon \in \left(0,\frac{1}{2}\right]$ and $\delta \in (0, \epsilon)$ such that $\beta(\epsilon, \delta, s) < \min \left(1, s \exp \left \{s \left( c_1\epsilon^2 + c_2 \sqrt{\frac{\delta}{\epsilon}} - 2\right)\right \} \right)$, then for $S \sim \mathcal{D}^s$ and for all distributions $Q$ over $\Theta$:
\begin{align*}
    \E_Q [R(\theta; \mathcal{D}, \lambda)] \leq \E_Q[\hat{R}(\theta; S, \lambda)] + \left\{\frac{1}{2s-1} \left (\kld{Q}{P} + \log \frac{5s}{\Delta} \right) + c_1 \epsilon^2 + c_2 \sqrt{\frac{\delta}{\epsilon}} \right\}^{1/2}
\end{align*}
holds with probability at least $1 - \Delta$ for some positive constants $c_1$ and $c_2$.
\label{thm:app_pbapproxdp}
\end{theorem}
\begin{proof}
We prove this result as a special case of Theorem~\ref{thm:app_genericbound} where $F(x,y) \defeq (2s - 1)(x - y)^2$. Then, $f(\theta, S) \defeq (2s - 1)\left (R(\theta; \mathcal{D}, \lambda) -  \hat{R}(\theta; S, \lambda)\right)^2$. Using this definition of $f$ and rearranging terms yields the following bound:
\begin{align*}
    \E_Q [R(\theta; \mathcal{D}, \lambda)] \leq \E_Q[\hat{R}(\theta; S, \lambda)] + \left\{\frac{1}{2s-1} \left (\kld{Q}{P} + \log \frac{\xi(f, P)}{\Delta} \right)\right\}^{1/2}\eqperiod
\end{align*}
We can then use Lemma~\ref{lma:maxinfopb} to bound $\xi(f, P)$. Lemma~2 of \citet{mcallester2003pac} shows that $\xi_{bd}(f) \leq 4s$ for this choice of $f$. We also observe that $\norm{\exp\{f\}}_\infty = \sup_{x \in [0,1], y \in [0,1]} \exp \{(2s - 1)(x - y)^2\}$ is bounded above by $\exp\{2s\}$. We can then rewrite the bound above as follows:
\begin{multline*}
    \E_Q [R(\theta; \mathcal{D}, \lambda)] \leq \E_Q[\hat{R}(\theta; S, \lambda)] + \\ \left\{\frac{1}{2s-1} \left (\kld{Q}{P} + 
    \log \frac{4s\exp\left \{I^{\beta(\epsilon, \delta, s)}_\infty(P_S, s) \right\} + \beta(\epsilon, \delta, s) \exp \{2s\}}{\Delta} \right)\right\}^{1/2}   \eqperiod 
\end{multline*}
Applying the bound in Theorem~\ref{thm:maxinfobound} to $I^{\beta(\epsilon, \delta, s)}_\infty(P_S, s)$ proves the first result of the theorem statement. The second result follows if it is possible to choose $\epsilon, \delta$ such that $\beta(\epsilon, \delta, s) \exp \{2s\} \leq s\exp\left \{I^{\beta(\epsilon, \delta, s)}_\infty(P_S, s) \right\}$. If that condition is met, then we can upper bound the numerator in that term by $5s \exp\left \{I^{\beta(\epsilon, \delta, s)}_\infty(P_S, s) \right\}$. The second result then immediately follows.
\end{proof}
\begin{theorem}
We assume that $\ell$ is bounded in $[a,b]$ and that $P$ is chosen such that samples drawn from it are $(\epsilon, \delta)$-DP with respect to $S \sim \mathcal{D}^s$ without the use of any non-trivial composition. We require $\epsilon \in \left(0,\frac{1}{2}\right]$ and $\delta \in (0, \epsilon)$ such that $\beta(\epsilon, \delta, s) < 1$.  Then, for $S \sim \mathcal{D}^s$ and for all distributions $Q$ over $\Theta$,
\begin{multline*}
    \E_Q [R(\theta; \mathcal{D}, \lambda)] \leq \E_Q[\hat{R}(\theta; S, \lambda)] + \\ \left\{\frac{(b - a)^2}{2s-1} \left (\kld{Q}{P} + \log \frac{4s\exp \left \{s \left ( c_1 \epsilon^2 + c_2 \sqrt{\frac{\delta}{\epsilon}} \right) \right \} + \beta(\epsilon, \delta, s) \exp\{2s\}}{\Delta} \right) \right\}^{1/2}
\end{multline*}
holds with probability at least $1 - \Delta$ for some positive constants $c_1$ and $c_2$.

If we further choose $\epsilon \in \left(0,\frac{1}{2}\right]$ and $\delta \in (0, \epsilon)$ such that $\beta(\epsilon, \delta, s) < \min \left(1, s \exp \left \{s \left( c_1\epsilon^2 + c_2 \sqrt{\frac{\delta}{\epsilon}} - 2\right)\right \} \right)$, then for $S \sim \mathcal{D}^s$ and for all distributions $Q$ over $\Theta$:
\begin{align*}
    \E_Q [R(\theta; \mathcal{D}, \lambda)] \leq \E_Q[\hat{R}(\theta; S, \lambda)] + (b - a)\left\{\frac{1}{2s-1} \left (\kld{Q}{P} + \log \frac{5s}{\Delta} \right) + c_1 \epsilon^2 + c_2 \sqrt{\frac{\delta}{\epsilon}} \right\}^{1/2}
\end{align*}
holds with probability at least $1 - \Delta$ for some positive constants $c_1$ and $c_2$.
\label{thm:app_pbapproxdpbd}
\end{theorem}
\begin{proof}
Let $F(x, y) \defeq (2s - 1)\left(\frac{x - a}{b - a} - \frac{y - a}{b - a} \right)^2$. The LHS of the bound in Theorem~\ref{thm:app_genericbound} can then be rewritten as $\frac{2s - 1}{(b - a)^2} \left (\E_Q[R(\theta; \mathcal{D}, \lambda)] - \E_Q[R(\theta; S, \lambda)] \right)^2$. The remainder of the proof follows from the proof of Theorem~\ref{thm:app_pbapproxdp}.
\end{proof}
\subsection{Differentially Private SGLD} \label{sec:sgldappendix}
The results in Section~\ref{sec:gentheory} assume the use of a particular $(\epsilon, \delta)$-DP algorithm that we introduce here. To define the algorithm that satisfies the assumptions of Theorem~\ref{thm:app_pbapproxdp}, we make three modifications to standard SGLD \cite{welling2011bayesian}. The first two modifications were used by \citet{li2017connecting} for their $(\epsilon,\delta)$-DP SGLD algorithm.

We first scale the gradient computed at each step by ``clipping'' it to have norm no greater than some input parameter $\gamma$. For the purpose of computing the privacy parameters, $\epsilon$ and $\delta$, gradient clipping is equivalent to $\gamma$-Lipschitz continuity for the optimized function. $(\epsilon,\delta)$-DP gradient-based optimization algorithms commonly assume Lipschitz continuity because the gradients of such functions have bounded sensitivity to modifications of the data set. We will denote gradient clipping by
\begin{align*}
    \clip_\gamma(g) = \frac{g}{\max\left(1, \frac{\norm{g}_2}{\gamma} \right)}\eqperiod
\end{align*}
In addition to gradient clipping, we bound the step size, $\eta_t$, of SGLD. Consider the limit in which $\eta_t$ approaches $0$; then SGLD will simply output samples from its initial (data-independent) distribution. More generally, the smaller the step sizes are, the closer the distribution of $\theta_T$ is to the initial distribution. The step size bound thus makes the output sample less sensitive to the data used at each iteration.

Last, to satisfy the product distribution assumption of Theorem~\ref{thm:maxinfobound}, we sample $h$ data points \emph{without} replacement at each iteration; the algorithm can then be run for no longer than one epoch (i.e., one full pass through the data set). If the same data points were to be reused in later iterations of Algorithm~\ref{alg:dpsgld_onechain}, the distribution of those data points, conditioned on prior iterates, would no longer be a product distribution. We speculate that this restriction is not necessary (i.e., Theorem~\ref{thm:maxinfobound} can be proven with a weaker assumption), but we leave a proof of this claim to future work. If this conjecture is true, we could select a PAC-Bayes prior by running the $(\epsilon,\delta)$-DP SGLD algorithm introduced by \citet{li2017connecting} for as many epochs as we desire.

\begin{algorithm}[ht!]
   \caption{$(\epsilon,\delta)$-DP SGLD Algorithm}
   \label{alg:dpsgld_onechain}
\begin{algorithmic}
   \STATE {\bfseries Input:} $S \sim \mathcal{D}^s$, $T$, $\epsilon$, $\delta$, $\{\eta_t\}_{t = 0}^{T-1}$, $P_0$, $h$, $\gamma$
   \STATE {\bfseries Output:} $\theta_T$
   \STATE $\theta_0 \sim P_0$
   \FOR{$t=0$ {\bfseries to} $T-1$}
   \STATE $J_t \leftarrow$ simple random sample of size $h$ (without replacement) of $[0,\dots,s]$ 
   \STATE $z_t \leftarrow \mathcal{N}\left (\mathbf{0}_m, \frac{2 \eta_t}{s} \mathbb{I}_m \right)$
   \STATE $\theta_{t + 1} \leftarrow \theta_t - \frac{\eta_t}{h} \clip_\gamma\left (\nabla_\theta \hat{R}(\theta_t; S_{J_t}) \right)  + z_t$
   \ENDFOR
\end{algorithmic}
\end{algorithm}

Theorem~\ref{thm:sgld_guarantees} describes how $\gamma$ and $\eta_t$ affect the privacy parameters $\epsilon$ and $\delta$ of Algorithm~\ref{alg:dpsgld_onechain}.
\begin{theorem}
Algorithm~\ref{alg:dpsgld_onechain} is $(\epsilon, \delta)$-DP with respect to $S$ if 
\begin{align*}
    \eta_t \leq \frac{h^2\epsilon^2 }{s \gamma^2 \log \frac{1.25}{\delta}}\eqperiod
\end{align*}
\label{thm:sgld_guarantees}
\end{theorem}
\vspace{-15pt}Proving Theorem~\ref{thm:sgld_guarantees} requires the following results from the differential privacy literature. 

\begin{lemma}[The Parallel Composition Theorem \cite{barthe2012probabilistic}]
If $T$ $(\epsilon_t, \delta_t)$-DP algorithms $\mathcal{A}_t$ are applied to disjoint subsets of $S$, then $\mathcal{A} \defeq \mathcal{A}_1 \circ \cdots \circ \mathcal{A}_T$ is $(\max_t \epsilon_t, \max_t \delta_t)$-DP with respect to $S$.
\label{lma:parallelcomp}
\end{lemma}
The next result establishes the differential privacy parameters of an algorithm that computes a bounded function of a data set and then adds Gaussian noise. This algorithm for producing differentially private outputs of a bounded function is often referred to as the ``Gaussian mechanism.''
\begin{lemma}[\citet{dwork2014algorithmic}]
If $f: (\mathcal{X} \times \mathcal{Y})^s \to \mathbb{R}^m$ and $\norm{f}_2 \leq 1$, then if
\begin{align*}
    \sigma^2 \geq \frac{2 \log \frac{1.25}{\delta}}{\epsilon^2}\eqcomma
\end{align*}
a sample from $f(S) + \mathcal{N}(0, \sigma^2 \mathbb{I}_m)$ is $(\epsilon, \delta)$-DP.
\label{lma:gaussianmech}
\end{lemma}
\begin{proof}[Proof of Theorem~\ref{thm:sgld_guarantees}]
Let $\mathcal{A}_{t+1}$ denote the noisy gradient step taken at iteration $t$: $\mathcal{A}_{t + 1}$ takes $\theta_t$ and $S_{J_t}$ as inputs and outputs $\theta_{t + 1}$. $\theta_T$ (i.e., the output of Algorithm~\ref{alg:dpsgld_onechain}) is then the result of composing $\mathcal{A}_1 \circ \cdots \circ \mathcal{A}_T$. Because we assume $J_t$ is sampled without replacement, Lemma~\ref{lma:parallelcomp} implies that this algorithm is $(\epsilon, \delta)$-DP if every $\mathcal{A}_t$ is $(\epsilon, \delta)$-DP with respect to $S_{J_t}$. To show this, we rewrite the gradient update in $\mathcal{A}_{t+1}$ (i.e., the only part of the algorithm that depends on $S_{J_t}$) and prove that it is an instance of the Gaussian mechanism analyzed in Lemma~\ref{lma:gaussianmech}: 
\begin{align*}
    -\frac{\eta_t \gamma}{h} \left [\frac{1}{\gamma} \clip_\gamma\left (\nabla_\theta \hat{R}(\theta_t; S_{J_t}) \right) + \mathcal{N}\biggl (0, \underbrace{\frac{2 \eta_t}{s} \frac{h^2}{\eta^2_t \gamma^2}}_{\sigma^2_t} \mathbb{I}_m \biggl) \right] \eqperiod
\end{align*}
Factoring out $\left(-\frac{\eta_t \gamma}{h}\right)$ ensures that the clipped gradient has norm no greater than $1$; the gradient thus satisfies the assumption regarding $f$ in Lemma~\ref{lma:gaussianmech}. Lemma~\ref{lma:gaussianmech} implies that we can only guarantee $(\epsilon, \delta)$-DP for $\mathcal{A}_t$ if the following bound on $\sigma^2_t$ holds:
\begin{align*}
    \frac{2 h^2}{s \eta_t \gamma^2} \geq \frac{2 \log \frac{1.25}{\delta}}{\epsilon^2}\eqperiod
\end{align*}
Solving for $\eta_t$ yields the desired result
\begin{equation*}
    \frac{\epsilon^2 h^2}{s\gamma^2 \log \frac{1.25}{\delta}} \geq \eta_t \eqperiod\shoveright{\qedhere}
\end{equation*}
\end{proof}
Note that Theorem~\ref{thm:sgld_guarantees} does not establish a single pair of privacy parameters, $\epsilon$ and $\delta$, given a user's choice of $\eta_t$ and $\gamma$. We prove that there exists a \emph{set} of $\epsilon$ and $\delta$ for which the inequality in Theorem~\ref{thm:sgld_guarantees} is tight. This set is akin to a Pareto frontier because any other choice of $\epsilon$ and $\delta$ results in a PAC-Bayes bound that can be improved by simply reducing either $\epsilon$ or $\delta$.

Algorithm~\ref{alg:dpsgld_onechain} and Theorem~\ref{thm:sgld_guarantees} establish a method for producing one sample from an $(\epsilon, \delta)$-DP version of SGLD. We could collect more samples by running $C$ independent instantiations of Algorithm~\ref{alg:dpsgld_onechain}. Inspired by the Markov Chain Monte Carlo literature, we refer to these parallel samplers as chains. Although running $C$ chains ostensibly reuses the data set  $C$ times, this composition is \emph{parallel}, not sequential. Conditioning on the output of the first chain does not affect the distribution of the output of any other chain. The algorithm we describe is then $(C\epsilon, C\delta)$-DP and operates on a product distribution, implying that it satisfies the assumptions of Theorem~\ref{thm:maxinfobound}.
\subsection{Proof of Remaining Results in Section~\ref{sec:gentheory}} \label{sec:resultappendix}
We now use the theory established in the prior sections of this appendix to prove the remaining results we describe in Section~\ref{sec:gentheory}. We first recall the well-known KL divergence between two multivariate Gaussian distributions.
\begin{lemma}[\citet{duchi}]
Given two multivariate Gaussian distributions $\mathcal{N}\left ( \mu_1, \Sigma_1 \right)$ and $\mathcal{N}\left ( \mu_2, \Sigma_2 \right)$ where $\mu_{(\cdot)} \in \mathbb{R}^m$ and $\Sigma_{(\cdot)} \in \mathbb{R}^{m \times m}$,
\begin{align*}
    \kld{\mathcal{N}\left ( \mu_1, \Sigma_1 \right)}{\mathcal{N}\left ( \mu_2, \Sigma_2) \right)} &= \frac{1}{2} \left (\tr\left(\Sigma_1^{-1} \Sigma_0 \right) + (\mu_1 - \mu_0)^T \Sigma_1^{-1} (\mu_1 - \mu_0) - m +  \log {\left ( \frac{\det \Sigma_1}{\det \Sigma_0} \right)}\right)\eqperiod 
\end{align*}
\label{lma:gaussian}
\end{lemma}
We next prove Corollary~\ref{cor:klbound}, which establishes a tractable upper bound on the original PAC-Bayes bound. Because we assume that the validation risk is $\gamma$-Lipschitz in the statement of Corollary~\ref{cor:klbound}, $\clip_\gamma(\cdot)$ becomes the identity function. Even though Algorithm~\ref{alg:dpsgld_onechain} ostensibly still includes gradient clipping, this assumption allows us to omit $\clip_\gamma(\cdot)$ from the derived bound.

For convenience, we will use the following symbols for the mean of Gaussians defined below,
\begin{equation}
    \begin{split}
    \mu_{t}^{\T} &\defeq \theta_{t-1} - \eta^{\T}_t \nabla_\theta \hat{R}(\theta_{t-1}; S^{(t-1)}_{\T},\lambda)\eqcomma\\
    \mu_{t}^{\V} &\defeq \theta_{t-1} - \eta^{\V}_t \clip_\gamma\left(\nabla_\theta \hat{R}(\theta_{t-1}; S^{(t-1)}_{\V})\right)\eqperiod
\end{split}
\label{eq:gaussianparams}
\end{equation}
As mentioned above, when we assume the validation risk is $\gamma$-Lipschitz, we can suppress the clipping function in Eq.~\ref{eq:gaussianparams}.
{\renewcommand{\thesection}{2}
\setcounter{restatecorollary}{0}
\begin{restatecorollary}
We assume that the same constant step size $\eta$ is used to define both $p(\theta_T | S_\T, \lambda)$ and $p(\theta^{(\epsilon, \delta)}_T | S_\V, \lambda)$, $\hat{R}_\V(\theta; S_\V, \lambda)$ is $\gamma$-Lipschitz, and that both iterative methods are initialized with $\theta_0 \sim P_0$. Then, given that $\nu_t$ is the distribution of the $t$-th iterate of the Langevin sampler applied to $S_\T$,
\begin{align*}
    \kldbig{p(\theta_T | S_\T, \lambda)}{p\left (\theta^{(\epsilon, \delta)}_T | S_\V, \lambda \right )} \leq B + \sum_{t = 0}^{T - 1} \frac{n_\V \eta}{4} \E_{\nu_t} \left [\norm{\nabla_\theta \hat{R}_\T(\theta; S^{(t)}_\T, \lambda) - \nabla_\theta \hat{R}_\V(\theta; S^{(t)}_\V, \lambda)}_2^2 \right]\eqcomma
\end{align*}
for some constant $B(\eta,n_\T,n_\V,m)$.
\label{cor:app_klbound}
\end{restatecorollary}}
\begin{proof}
We observe that the distribution of $\theta_t$ conditioned on $\theta_0,\dots,\theta_{t - 1}$ is the following Gaussian:
\begin{align*}
    \mathcal{N} \left (\mu_{t}^{(\cdot)}, \frac{2\eta}{n_{(\cdot)}} \mathbb{I}_m\right)\eqperiod
\end{align*}
Applying Lemma~\ref{lma:klchainrule} to $\kldbig{p(\theta_T | S_\T, \lambda)}{p\left(\theta^{(\epsilon, \delta)}_T | S_\V, \lambda, \lambda \right)}$ then yields the following bound.
\begin{align*}
    \kld{p(\theta_T | S_\T, \lambda)}{p(\theta^{(\epsilon, \delta)}_T | S_\V, \lambda, \lambda)} \leq \sum_{t = 1}^T \E_{\theta_0,...,\theta_{t - 1}} \left [\kldbig{\mathcal{N}\left (\mu_{t}^\T, \frac{2\eta}{n_{\T}}\mathbb{I}_m\right)}{\mathcal{N} \left ( \mu_{t}^\V, \frac{2\eta}{n_{\V}}\mathbb{I}_m \right)} \right]\eqperiod
\end{align*}
Applying Lemma~\ref{lma:gaussian} to expand the KL divergence summands on the right-hand-side of the inequality yields:
\begin{align*}
    \text{RHS} &= \sum_{t = 1}^T \frac{n_\V \eta}{4} \E_{\nu_{t - 1}} \left [ \norm{\nabla_\theta \hat{R}_\T(\theta_{t-1}; S^{(t - 1)}_\T, \lambda) - \nabla_\theta \hat{R}_\V(\theta_{t-1}; S^{(t - 1)}_\V, \lambda)}_2^2 \right] + B(\eta, n_\T, n_\V, m) \\
    &= \sum_{t = 0}^{T-1} \frac{n_\V \eta}{4} \E_{\nu_t} \left [ \norm{\nabla_\theta \hat{R}_\T(\theta_{t}; S^{(t)}_\T, \lambda) - \nabla_\theta \hat{R}_\V(\theta_t; S^{(t)}_\V, \lambda)}_2^2 \right]+ B(\eta, n_\T, n_\V, m)\eqperiod\qedhere
\end{align*}
\end{proof}

We generalize Corollary~\ref{cor:app_klbound} by relaxing two assumptions. Because we may wish to use training gradient step sizes that exceed the bound required for $(\epsilon,\delta)$-DP SGLD on the validation set, we allow for different step sizes on the training and validation samplers. We also relax the $\gamma$-Lipschitz assumption on the validation risk by using clipped gradients as defined in Section~\ref{sec:sgldappendix}.
\begin{corollary}
We assume that $p(\theta_T | S_\T, \lambda)$ is defined by a Langevin sampler with step sizes $\{\eta^{\T}_t\}_{t = 0}^{T - 1}$ and that $p(\theta^{(\epsilon, \delta)}_T | S_\V, \lambda)$ is defined by a Langevin sampler with step sizes $\{\eta^{\V}_t\}_{t = 0}^{T - 1}$. As in Corollary~\ref{cor:app_klbound}, both iterative methods are initialized with $\theta_0 \sim P_0$. Then, $\kld{p(\theta_T | S_\T, \lambda)}{p(\theta^{(\epsilon, \delta)}_T | S_\V, \lambda)}$ is less than or equal to
\begin{align*}
   \sum_{t = 0}^{T-1} \frac{n_\V }{4\eta^{\V}_t} \E_{\nu_t} \left [\norm{\eta^{\T}_t \nabla_\theta \hat{R}_\T(\theta_{t}; S^{(t)}_\T, \lambda) - \eta^{\V}_t \clip_\gamma\left (\nabla_\theta \hat{R}_\V(\theta_t; S^{(t)}_\V, \lambda) \right)}_2^2 \right] + m \left [ \frac{n_\V\eta_t^{\T}}{n_\T \eta_t^{\V}}  - \log \left (\frac{n_\V \eta_t^{\T}}{n_\T \eta_t^{\V}} \right) - 1\right]\eqperiod
\end{align*}
\label{cor:app_klgenbound}
\end{corollary}
\begin{proof}
Applying Lemma~\ref{lma:klchainrule} to $\kld{p(\theta_T | S_\T, \lambda)}{p(\theta^{(\epsilon, \delta)}_T | S_\V, \lambda)}$ yields the following bound.
\begin{align*}
    \kld{p(\theta_T | S_\T, \lambda)}{p(\theta^{(\epsilon, \delta)}_T | S_\V, \lambda)} \leq \sum_{t = 1}^T \E_{\theta_0,\ldots,\theta_{t - 1}} \left [ \kldbig{\mathcal{N}\left (\mu_{t}^\T, \frac{2\eta^\T_t}{n_\T}\mathbb{I}_m \right)}{\mathcal{N} \left ( \mu_t^\V, \frac{2\eta^{\V}_t}{n_\V}\mathbb{I}_m\right)} \right]\eqperiod
\end{align*}
Substituting the KL divergence between multivariate Gaussians and simplifying yields:
\begin{align*}
    \sum_{t = 0}^{T-1} \frac{n_\V }{4 \eta^{\V}_t} \E_{\nu_t} \left [ \norm{\eta^{\T}_t \nabla_\theta \hat{R}_\T(\theta_{t}; S^{(t)}_\T, \lambda) - \eta^{\V}_t \clip_\gamma \left (\nabla_\theta \hat{R}_\V(\theta_t; S^{(t)}_\V, \lambda) \right)}_2^2 \right] + m \left [ \frac{n_\V}{n_\T} \frac{\eta_t^{\T}}{\eta_t^{\V}}  - \log \left (\frac{n_\V}{n_\T} \frac{\eta_t^{\T}}{\eta_t^{\V}} \right) - 1\right]\eqperiod\mbox{\qedhere}
\end{align*}
\end{proof}

Last, we prove the results presented in Section~\ref{sec:gentheory} that describe how the objective in Eq.~\ref{eq:ho-stability} also bounds the generalization error attributable to the use of out-of-sample training data.
{\renewcommand{\thesection}{2}
\setcounter{restatelemma}{1}
\begin{restatelemma}
If $\ell$ is $\gamma$-Lipschitz and $d^1_{\mathcal{W}}(p(\theta_T | S_\T, \lambda), p(\theta | S_\T, \lambda)) \leq \kappa$, then 
\begin{align*}
    \frac{1}{n_\T} \kld{p(\theta | S_\T, \lambda)}{p(\theta | \mathcal{D}_\T, \lambda)} \leq \E_{p(\theta_T | S_\T, \lambda)} [R_\T(\theta; \mathcal{D}_\T, \lambda) - \hat{R}_\T(\theta; S_\T, \lambda)] + 2\gamma \kappa\eqperiod
\end{align*}
\label{lma:app_trainrw}
\end{restatelemma}}
\begin{proof}\vspace{-15pt}
Applying the definition of the 1-Wasserstein distance (Eq.~\ref{eq:wassdist}) to 
\begin{align*}
\left|\E_{p(\theta_T | S_\T, \lambda)} [R_\T(\theta; \mathcal{D}_\T, \lambda) - \hat{R}_\T(\theta; S_\T, \lambda)] - \E_{p(\theta | S_\T, \lambda)} [R_\T(\theta; \mathcal{D}_\T, \lambda) - \hat{R}_\T(\theta; S_\T, \lambda)]\right|\eqcomma
\end{align*}
as well as the $\gamma$-Lipschitz assumption for $\ell$ imply the following inequality:
\begin{align*}
    \left|\E_{p(\theta_T | S_\T, \lambda)} [R_\T(\theta; \mathcal{D}_\T, \lambda) - \hat{R}_\T(\theta; S_\T, \lambda)] - \E_{p(\theta | S_\T, \lambda)} [R_\T(\theta; \mathcal{D}_\T, \lambda) - \hat{R}_\T(\theta; S_\T, \lambda)]\right| \leq 2 \gamma \kappa\eqperiod
\end{align*}
This inequality allows us to compute a bound on the generalization error under samples from $p(\theta | S_\T, \lambda)$.
\begin{align*}
    \E_{p(\theta | S_\T, \lambda)} [R_\T(\theta; \mathcal{D}_\T, \lambda) - \hat{R}_\T(\theta; S_\T, \lambda)] \leq \E_{p(\theta_T | S_\T, \lambda)} [R_\T(\theta; \mathcal{D}_\T, \lambda) - \hat{R}_\T(\theta; S_\T, \lambda)] + 2\gamma\kappa\eqperiod
\end{align*}
We also show that $\E_{p(\theta | S_\T, \lambda)} [R(\theta; \mathcal{D}_\T, \lambda) - \hat{R}(\theta; S_\T, \lambda)]$ can be rewritten as the KL divergence we aim to bound:
\begin{align*}
    n_\T \left (\E_{p(\theta | S_T, \lambda)}[R_\T(\theta ; \mathcal{D}_\T, \lambda) - \hat{R}_\T(\theta ; S_\T, \lambda)] \right)&= \E_{p(\theta | S_\T, \lambda)} \left [ \log \frac{p(\theta | S_\T, \lambda)}{p(\theta | \mathcal{D}_\T, \lambda)} \right] \\
    &= \kld{p(\theta | S_\T, \lambda)}{p(\theta | \mathcal{D}_\T, \lambda)}\eqperiod
\end{align*}
Substituting $\frac{1}{n_\T}\kld{p(\theta | S_\T, \lambda)}{p(\theta | \mathcal{D}_\T, \lambda)}$ into the generalization error bound above yields the desired result.
\end{proof}
{\renewcommand{\thesection}{2}
\setcounter{restatecorollary}{1}
\begin{restatecorollary}
Retain the assumptions of Corollary~\ref{cor:app_klbound} and Lemma~\ref{lma:app_trainrw}. Then, minimizing 
\begin{align*}
    \sqrt{\sum_{t = 0}^{T - 1} d^2_{2, \nu_t}(p^1(\theta | S^{(t)}_\T, \lambda), p^1(\theta | S^{(t)}_\V, \lambda))}
\end{align*}
is equivalent to minimizing an upper bound on 
\begin{align*}
    \kld{p(\theta | S_\T, \lambda)}{p(\theta | \mathcal{D}_\T, \lambda)}
\end{align*}
with respect to $\lambda$.
\label{cor:app_trainbound}
\end{restatecorollary}}
\begin{proof}
\begin{align*}
    \kld{p(\theta | S_\T, \lambda)}{p(\theta | \mathcal{D}_\T, \lambda)} &\leq n_\T \E_{p(\theta_T | S_\T, \lambda)} \left [ R_\T(\theta; \mathcal{D}_\T, \lambda) - \hat{R}_\T(\theta; S_\T, \lambda)\right] + 2 n_\T \gamma \kappa \\
    &\leq  n_\T \left \{\frac{1}{2n_\V} \kld{p(\theta_T | S_\T, \lambda)}{p(\theta^{(\epsilon, \delta)}_T | S_\V, \lambda)} + A \right \}^{1/2} + 2 n_\T \gamma \kappa  \\
    &\leq D(\eta, n_\V, n_\T) \left \{\sum_{t = 0}^{T - 1} d^2_{2, \nu_t}(p^1(\theta | S^{(t)}_\T, \lambda), p^1(\theta | S^{(t)}_\V, \lambda)) \right\}^{1/2} + n_\T \left(A^{1/2} + B + 2 \gamma \kappa\right)\eqperiod
\end{align*}
The first inequality follows from Lemma~\ref{lma:app_trainrw}, the second from Eq.~\ref{eq:pb_bound}, and the third from applying Corollary~\ref{cor:app_klbound} as well as the triangle inequality. 
\end{proof}

\subsection{Observations} \label{sec:obsappendix}
Here we summarize various characteristics of the objective we have derived; we use these observations to motivate and propose improvements to methods developed in previously published work for improving generalization.

Even if the validation risk is $\gamma$-Lipschitz, we may still wish to clip the validation gradient in our objective to have norm less than $\gamma$. Additional gradient clipping reduces the privacy parameters, $\epsilon$ and $\delta$, of Algorithm~\ref{alg:dpsgld_onechain}, which potentially tightens the bound that we minimize in Eq.~\ref{eq:ho-stability} (Theorems~\ref{thm:sgld_guarantees} and \ref{thm:sqrtpbapproxdp}). Though clipping the \emph{training} gradient does not tighten the PAC-Bayes bound, \citet{menon2020can} prove that gradient clipping can accelerate training convergence when optimizing deep neural networks.  We expand a summand of the regularizer in Eq.~\ref{eq:ho-stability} when both the training and validation gradients are clipped to have unit norm (i.e., $\gamma = 1$).
\begin{multline}
    \norm{\clip_1\left (\nabla_\theta \hat{R}_\T(\theta_t; S^{(t)}_\T, \lambda) \right) - \clip_1\left (\nabla_\theta \hat{R}_\V(\theta_t; S^{(t)}_\V, \lambda) \right)}_2^2 = \\2 -  2\clip_1\left (\nabla_\theta \hat{R}_\T(\theta_t; S^{(t)}_\T, \lambda) \right)^T \clip_1\left (\nabla_\theta \hat{R}_\V(\theta_t; S^{(t)}_\V, \lambda) \right)  \label{eq:cossim}\eqperiod
\end{multline}
When $\norm{\nabla_\theta\hat{R}_\T(\theta_t; S_\T, \lambda)}_2 \geq 1$ and $\norm{\nabla_\theta\hat{R}_\V(\theta_t; S_\V, \lambda)}_2 \geq 1$, the second term of the RHS of Eq.~\ref{eq:cossim} is equal to a familiar measure of vector agreement: the cosine similarity. We summarize this finding in Observation~\ref{obs:normclip}.
\begin{observation}
Assuming $\norm{\nabla_\theta\hat{R}_\T(\theta_t; S_\T, \lambda)}_2 \geq 1$ and $\norm{\nabla_\theta\hat{R}_\V(\theta_t; S_\V, \lambda)}_2 \geq 1$, minimizing the regularizer that results from applying $\clip_1(\cdot)$ to both gradients is equivalent to maximizing the cosine similarity between the training and validation gradients.
\label{obs:normclip}
\end{observation}
In Section~\ref{sec:relobsappendix}, we show that Observation~\ref{obs:normclip} can be used to motivate a recently proposed method for improving generalization in meta-learning \cite{guiroy2019towards}.

Next, we provide an additional interpretation of the regularizer in Eq.~\ref{eq:ho-stability} when the privacy parameters, $\epsilon$ and $\delta$, of Algorithm~\ref{alg:dpsgld_onechain} approach 0 (i.e., Algorithm~\ref{alg:dpsgld_onechain} samples from a data-independent distribution). Recall that in Appendix~\ref{sec:sgldappendix}, we showed that the privacy parameters of Algorithm~\ref{alg:dpsgld_onechain} are controlled by our choices for the step size $\eta_t$ and clipped norm parameter $\gamma$. We then study the implications of choosing these parameters so that $\epsilon$ and $\delta$ for Algorithm~\ref{alg:dpsgld_onechain} can be chosen to approach $0$.

In the limit of infinitesimal step size, Theorem~\ref{thm:sgld_guarantees} establishes that we are free to choose infinitesimally small privacy parameters for Algorithm~\ref{alg:dpsgld_onechain}. To simplify our re-derivation of the regularizer in this setting, we assume that the step sizes of the training and validation Langevin samplers are constant, but not necessarily the same (i.e., $\eta^\V_t = \eta^\V$ and $\eta^\T_t = \eta^\T$ for all $t$). 

Without any additional changes to the definition of the PAC-Bayes prior, the value of the bound we prove in Corollary~\ref{cor:app_klbound} explodes when $\eta^\V$ approaches $0$. We restate the summands of the bound in terms of the KL divergences derived in the proof of Corollary~\ref{cor:app_klbound} to understand why:
\begin{align}
\sum_{t = 1}^T \E_{\nu_{t - 1}} \left [ \kldbig{\mathcal{N}\left (\mu_{t}^\T, \frac{2\eta^\T}{n_\T}\mathbb{I}_m \right)}{\mathcal{N} \biggl ( \underbrace{\mu_{t}^\V}_{\to \theta_{t - 1}}, \underbrace{\frac{2\eta^{\V}}{n_\V}}_{\to 0}\mathbb{I}_m\biggr)} \right] \label{eq:deltakl}\eqperiod
\end{align}
As $\eta^\V \to 0$, the summands in Eq.~\ref{eq:deltakl} converge to the KL divergence between a Gaussian and a $\delta$-function (i.e., Eq.~\ref{eq:deltakl} approaches $\infty$).

We can address this pathology while still selecting arbitrarily small $\eta^\V$ by modifying the $\tau$ parameter of the Gibbs posterior (Eq.~\ref{eq:posteriordef}) of the validation set. Though we had previously assumed for the sake of notational simplicity that $\tau = n_\V$ for the validation Gibbs posterior, we are, in fact, free to vary this parameter of the Gibbs posterior. By choosing $\tau = n_\V \eta^\V / \eta^T$, the distribution of $\theta_T$ now approaches the (data-independent) distribution induced by $T$ steps of a Gaussian random walk initialized with $\theta_0 \sim P_0$. Recomputing the sum of KL divergences in Eq.~\ref{eq:deltakl} results in a finite expression:
\begin{align*}
\sum_{t = 1}^T \E_{\nu_{t - 1}} \left [ \kldbig{\mathcal{N}\left (\mu_{t}^\T, \frac{2\eta^\T}{n_\T}\mathbb{I}_m \right)}{\mathcal{N} \biggl ( \underbrace{\mu_t^\V}_{\to \theta_{t - 1}}, \frac{2\eta^{\T}}{n_\V}\mathbb{I}_m\biggr)} \right]\eqperiod
\end{align*}

Completing the proof of Corollary~\ref{cor:app_klbound} with this choice of validation posterior, we recover the following limiting expression for the regularizer in Eq.~\ref{eq:ho-stability}:
\begin{align}
    \sqrt{\frac{\eta^\T}{4} \sum_{t = 0}^{T - 1}\E_{\nu_t} \left [ \norm{\nabla_\theta \hat{R}_\T(\theta_t; S^{(t)}_\T, \lambda)}_2^2 \right]} \label{eq:limitreg}\eqperiod
\end{align}
To better understand this expression, we approximate $\nu_t$, the distribution of the $t$-th Langevin iterate, with $p(\theta | S^{(t)}_\T, \lambda)$. We rewrite (and re-scale) the $t$-th summand of Eq.~\ref{eq:limitreg} as follows. Recall that $n_\T \hat{R}(\theta; S^{(t)}_\T, \lambda) = - \log p(\theta | S^{(t)}_\T, \lambda)$.
\begin{align}
    \E_{p(\theta | S^{(t)}_\T, \lambda)} \left [ \norm{\nabla_\theta \log p(\theta | S^{(t)}_\T, \lambda)}_2^2 \right] &= \E_{p(\theta | S^{(t)}_\T, \lambda)} \left [ \text{tr} \left (\nabla_\theta \log p(\theta | S^{(t)}_\T, \lambda) \nabla_\theta \log p(\theta | S^{(t)}_\T, \lambda)^T \right) \right] 
\label{eq:not_quite_fisher}\eqperiod
\end{align}
The expression we derive above is reminiscent of, although not identical to, the Fisher information matrix. A method for re-expressing the Fisher information thus helps to elucidate what minimizing Eq.~\ref{eq:not_quite_fisher} accomplishes. 
\begin{align*}
        \E_{p(\theta | S^{(t)}_\T, \lambda)} \left [\nabla^2_\theta \log p(\theta | S^{(t)}_\T, \lambda) \right] =& \E_{p(\theta | S^{(t)}_\T, \lambda)} \left [ \frac{\nabla^2_\theta p(\theta | S^{(t)}_\T, \lambda)}{p(\theta | S^{(t)}_\T, \lambda)} - \frac{\nabla_\theta p(\theta | S^{(t)}_\T, \lambda) \nabla_\theta p(\theta | S^{(t)}_\T, \lambda)^T}{p(\theta | S^{(t)}_\T, \lambda)^2} \right] \\
        =& \E_{p(\theta | S^{(t)}_\T, \lambda)} \left [ \frac{\nabla^2_\theta p(\theta | S^{(t)}_\T, \lambda)}{p(\theta | S^{(t)}_\T, \lambda)} \right] \\&- \E_{p(\theta | S^{(t)}_\T, \lambda)} \left [ \nabla_\theta \log p(\theta | S^{(t)}_\T, \lambda) \nabla_\theta \log p(\theta | S^{(t)}_\T, \lambda)^T \right] \\
        =& - \E_{p(\theta | S^{(t)}_\T, \lambda)} \left [ \nabla_\theta \log p(\theta | S^{(t)}_\T, \lambda) \nabla_\theta \log p(\theta | S^{(t)}_\T, \lambda)^T \right]\eqperiod
\end{align*}
The last step follows from a standard regularity condition on $p(\theta | S, \lambda)$: we assume that we are able to interchange the gradient operator and the integral. Applying the trace to both sides of this expression and combining with Eq.~\ref{eq:not_quite_fisher},
\begin{align*}
    \E_{p(\theta | S^{(t)}_\T, \lambda)} \left [ \norm{\nabla_\theta \log p(\theta | S^{(t)}_\T, \lambda)}_2^2 \right] &= \E_{p(\theta | S^{(t)}_\T, \lambda)} \left [ \text{tr} \left (\nabla_\theta \log p(\theta | S^{(t)}_\T, \lambda) \nabla_\theta \log p(\theta | S^{(t)}_\T, \lambda)^T \right) \right]  \\
    &= - \E_{p(\theta | S^{(t)}_\T, \lambda)} \left [\text{tr} \left (\nabla^2_\theta \log p(\theta | S^{(t)}_\T, \lambda) \right) \right] \\
    &\propto \E_{p(\theta | S^{(t)}_\T, \lambda)} \left [\text{tr} \left (\nabla^2_\theta \hat{R}_\T(\theta; S^{(t)}_\T, \lambda) \right) \right]\eqperiod
\end{align*}
Minimizing the regularizer in Eq.~\ref{eq:ho-stability} is thus equivalent to minimizing the sum of the training risk Hessian's eigenvalues evaluated at $\theta$ sampled from the training posterior. Assuming that the parameters drawn from the posterior are near the minimizer of the (assumed to be) locally strongly convex training risk, the sampled Hessian will always be positive definite. Minimizing the sum of the Hessian's eigenvalues is then equivalent to minimizing the curvature of the training optima.

We can derive a similar result regarding the regularizer in Eq.~\ref{eq:ho-stability} when the clipped norm parameter $\gamma$ approaches 0. Solving for $\gamma$ in the bound of Theorem~\ref{thm:sgld_guarantees} demonstrates that Algorithm~\ref{alg:dpsgld_onechain} is $(\epsilon, \delta)$-DP if:
\begin{align*}
    \gamma \leq \sqrt{\frac{\tau^2 \epsilon^2}{s \eta_t \log \frac{1.25}{\delta}}}\eqperiod
\end{align*}
Given some choice of step size, this result implies that $\gamma \to 0$ guarantees that we can select privacy parameters for Algorithm~\ref{alg:dpsgld_onechain} that approach 0 as well. Though we omit the full re-derivation for brevity, we can use our typical choice of $\tau = n_\V$ for the validation Gibbs posterior and repeat the steps shown above to derive a limiting expression for the regularizer that is identical to Eq.~\ref{eq:limitreg}. 
\begin{observation}
When we modify the parameters of Algorithm~\ref{alg:dpsgld_onechain} so that $\epsilon$ and $\delta$ approach $0$, minimizing the validation data-independent regularizer in Eq.~\ref{eq:ho-stability} is equivalent to minimizing the curvature of the training risk near its optima. Equivalently, minimizing Eq.~\ref{eq:ho-stability} promotes the selection of hyperparameters that lead to ``flat'' training optima.
\label{obs:app_trhessian}
\end{observation}

We last consider hyperparameter optimization problems for which $\mathcal{D}_\T = \mathcal{D}_\V$. For these cases, the use of Algorithm~\ref{alg:dpsgld_onechain} to select the prior is not necessary because we can derive a bound on the expected (training or validation) risk by modifying the substitutions we make in the main text in Theorem~\ref{thm:sqrtpbapproxdp} (i.e., we fill in the missing subscripts with $\T$ and $P \equiv p(\theta | S_\V, \lambda)$). The validation posterior, $p(\theta | S_\V, \lambda)$, yields a valid PAC-Bayes bound because the distribution is independent of the data set $S_\T$ used to compute the empirical risk. 

These substitutions recover the bound minimized by \citet{ambroladze2007tighter}. We can generalize their result by applying the Fisher distance bound from Lemma~\ref{lma:klchainrule}. As we allude to in Section~\ref{sec:relobsappendix}, minimizing this upper bound is feasible even when the PAC-Bayes prior and posterior are only implicitly defined by an iterative method. Observation~\ref{obs:app_ambroupdate} states the objective that we have derived under this simpler data-generating distribution assumption.
\begin{observation}
When $\mathcal{D}_\T = \mathcal{D}_\V$, the following objective minimizes a PAC-Bayes bound on the expected risk.
\begin{align*}
    \min_\lambda \E_{p(\theta_T | S_\T, \lambda)} \left [ \hat{R}_\T(\theta; S_\T, \lambda) \right] + \zeta \left \{ \sum_{t = 0}^{T - 1} d^2_{\nu_t}(p^1(\theta | S^{(t)}_\T, \lambda), p^1(\theta | S^{(t)}_\V, \lambda)) \right \}^{1/2}\eqperiod
\end{align*}
\label{obs:app_ambroupdate}
\end{observation}

\section{Algorithm} \label{sec:algappendix}
We investigate the effect of the choice of $K$ on the accuracy of the estimated hyperparameter gradient for the regularizer. While \citet{shaban2018truncated} show that optimization with gradients estimated using $K$-truncated-RMD can converge to a stationary point, verifying the conditions under which their result holds is impractical for the problems we consider. We must instead empirically evaluate the effect of small choices of $K$ on optimization of the regularizer.

To demonstrate that choosing the truncation $K = 1$ or $K = 0$ for the regularizer gradient is reasonable, in Figure~\ref{fig:k_comp}, we rerun one of the experiments we describe in Section~\ref{sec:regpenalty}: the optimization of per-parameter weight decays for a linear classifier fit to the MNIST data set.

To study the effect of choosing $K$ larger than 1, we use the Higher auto-differentiation package, which enables the calculation of higher-order gradients over training optimizations in PyTorch \cite{grefenstette2019generalized}. Computing the exact $K$-truncated-RMD gradient for any $K > 1$ using Higher would incur significant space and time overhead. We instead follow the approach of \citet{Metz19}, who compute a truncated-RMD gradient for each term in their hyperparameter optimization objective by partitioning the inner optimization $\theta_1,\dots,\theta_T$ into windows of size $W$. For each window, e.g. $\theta_{1},\dots,\theta_W$, the first parameter iterate is assumed to be constant. Using this approximation, RMD applied to a Fisher distance evaluated at, for instance, $\theta_i$ between $\theta_1$ and $\theta_W$ will be truncated after $i - 1$ steps of backpropagation. Generalizing this example, the gradient of the $i$-th term of a window is computed with $(i-1)$-truncated-RMD. Though there is no precise correspondence between using this windowed gradient approximation and applying $K$-truncated-RMD to every term in the regularizer for some fixed $K$, we observe that for a window of size $W$, the RMD algorithm is truncated, on average, after $(W-1)/2$ steps. In Figure~\ref{fig:k_comp}, we denote the choice of window size $w$ as $W = w$. 

Figure~\ref{fig:regularizer} shows how different choices of $K$ and $W$ affect optimization of the regularization term. In this experiment, to isolate the effect of the $K$-truncated-RMD gradient approximation, we optimized the regularization term alone. We then plot the square root of the regularization term (i.e., our estimate of the Fisher distance) computed over the $1000$ steps of inner optimization that take place between each hyperparameter gradient step. While the true hyperparameter gradient, denoted by $K = \infty$ (calculated in practice by setting $W$ to be the length of the inner optimization), appears to improve optimization of the regularizer at later iterations, more practical choices of $W$ and $K$ yield loss curves that lie on top of each other. Because the larger choices of $W$ appear to lead to slower optimization, the decrease in the estimated Fisher distance for $K = \infty$ at later iterations only offsets the initially slow optimization. Figure~\ref{fig:objective}, which plots the running value of the regularizer over the course of optimization (i.e., the quantity that we show to be equivalent to a PAC-Bayes bound), shows that the final value of the regularization term in Eq.~\ref{eq:ho-stability} appears to have minimal dependence on the gradient approximation used.

\begin{figure}[ht!]
     \centering
     \begin{subfigure}[t]{0.49\linewidth}
         \centering
         \includegraphics[width=\linewidth]{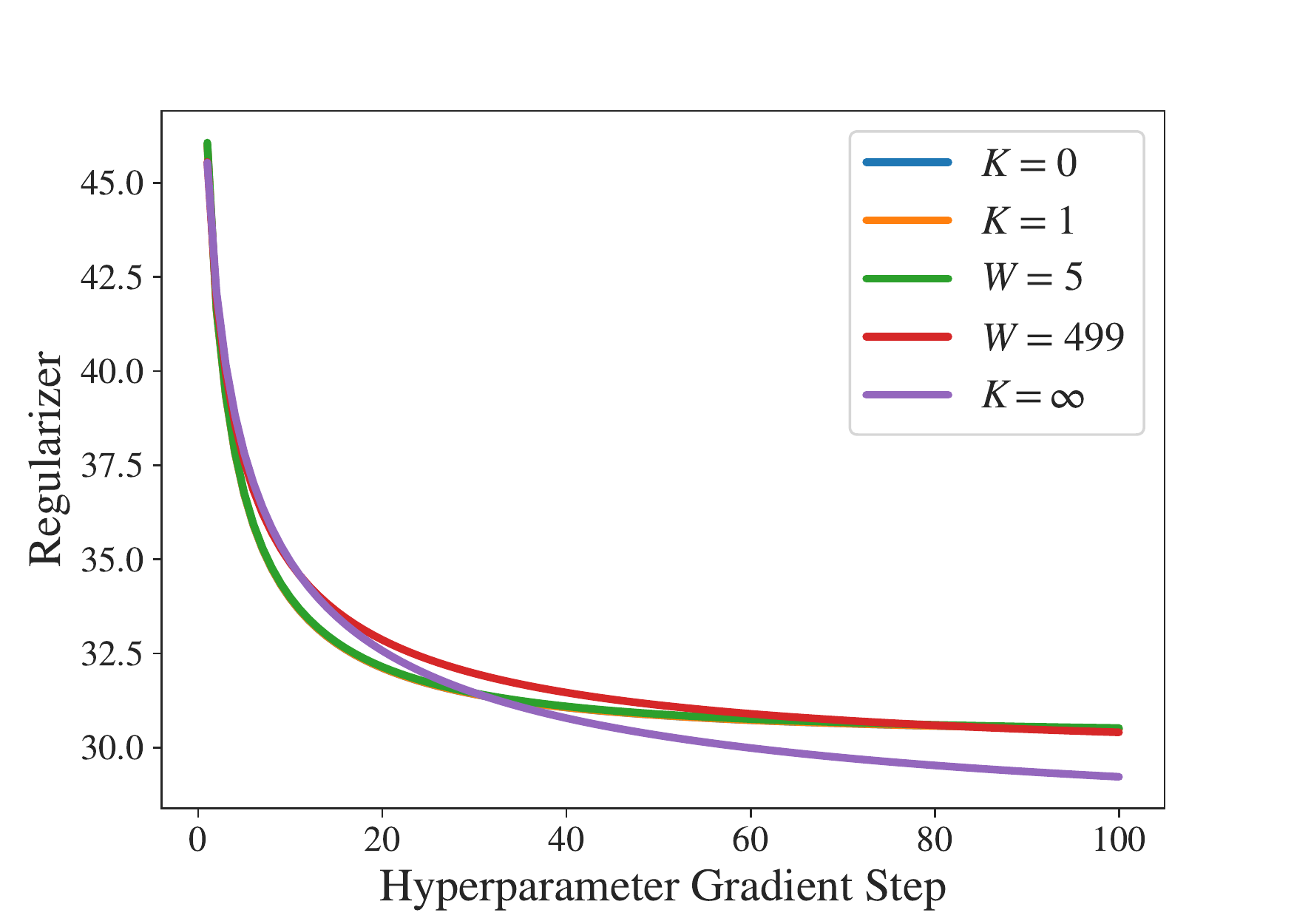}
         \caption{Regularizer value computed at each outer step.}
         \label{fig:regularizer}
     \end{subfigure}
     \hfill
     \begin{subfigure}[t]{0.49\linewidth}
         \centering
         \includegraphics[width=\linewidth]{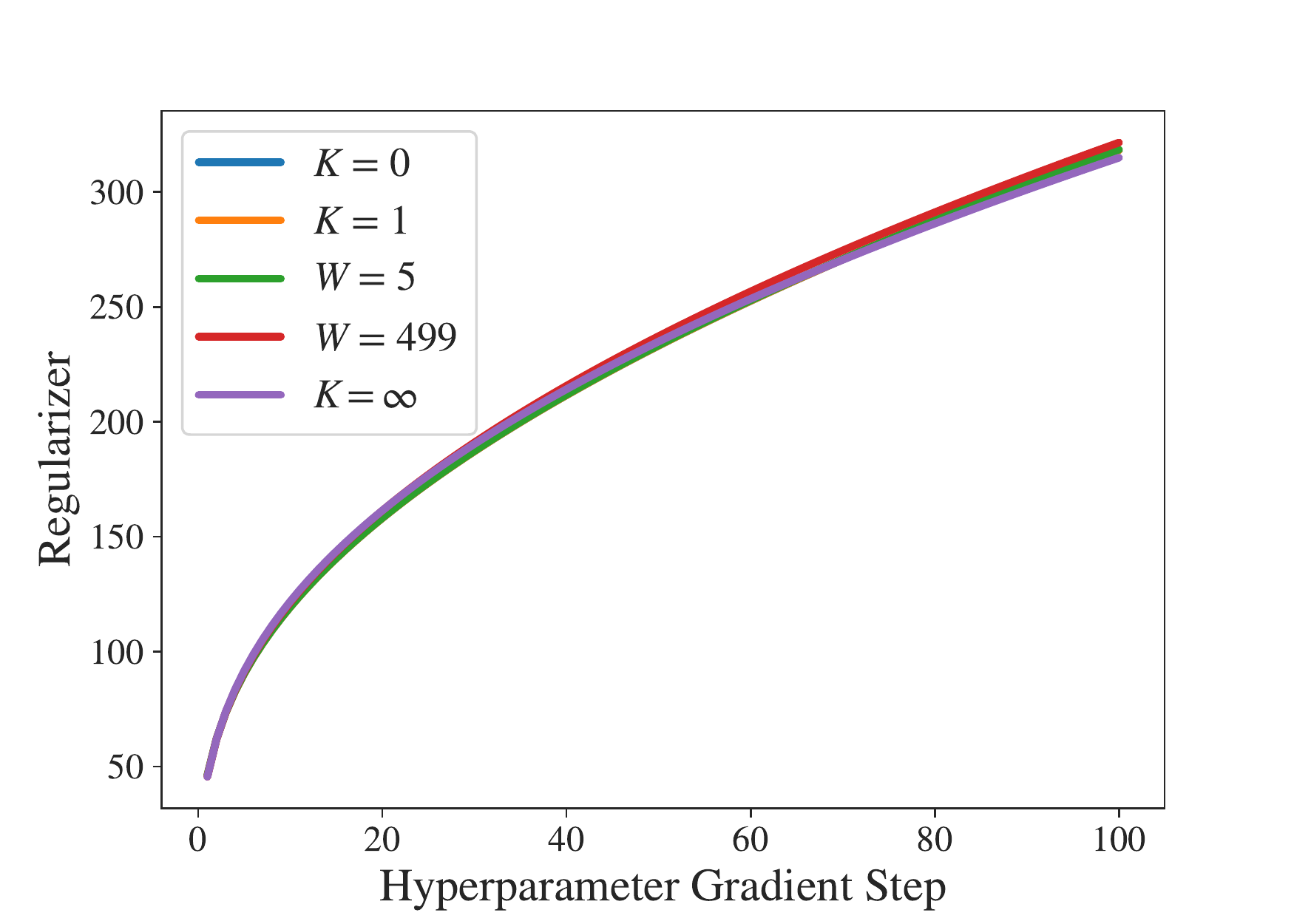}
         \caption{Running value of the regularizer over optimization.}
         \label{fig:objective}
     \end{subfigure}
     \caption{Comparing regularizer minimization using different gradient approximations.}
     \label{fig:k_comp}
\end{figure}

In Algorithm~\ref{alg:ho_stability_alg}, we provide complete pseudocode for our algorithm to optimize Eq.~\ref{eq:ho-stability}. Unlike in Algorithm~\ref{alg:ho_stability_alg_mt}, here we allow for the possibility of obtaining multiple samples, denoted by $C$, with which to estimate the objective. We assume for the sake of simplicity that the learning rate used for hyperparameter gradient descent is 1 and that the user has chosen some convergence criterion for the hyperparameter optimization (e.g., that the norm of the hyperparameter gradient is less than some threshold).
\begin{algorithm}[ht!]
   \caption{(Offline) Optimization of Eq.~\ref{eq:ho-stability}}
   \label{alg:ho_stability_alg}
\begin{algorithmic}
   \STATE {\bfseries Input:} $S_\T$, $S_\V$, $T$, $\{\eta_t\}_{t = 0}^{T-1}$, $P_0$, $C$, $\lambda_0$, $K$
   \STATE {\bfseries Output:} $\lambda$
   \WHILE{not converged}
   \STATE $\theta^{(1)}_0,\dots,\theta^{(C)}_0 \sim P_0$
   \STATE $X_{1:C} \leftarrow 0, Y_{1:C} \leftarrow 0$
   \FOR{$t=0$ {\bfseries to} $T-1$}
   \FOR{$c=1$ {\bfseries to} $C$}
   \STATE $\theta^{(c)}_{t + 1} \leftarrow \theta^{(c)}_t - \eta_t \nabla_\theta \hat{R}(\theta^{(c)}_t; S^{(t)}_\T, \lambda)$
   \STATE $X_c \leftarrow X_c + \frac{\zeta}{2} \widetilde{\nabla}^K_\lambda \norm{\nabla_\theta \hat{R}(\theta^{(c)}_t; S^{(t)}_\T, \lambda) - \nabla_\theta \hat{R}(\theta^{(c)}_t; S^{(t)}_\V, \lambda)}_2^2$
   \STATE $Y_c \leftarrow Y_c + \norm{\nabla_\theta \hat{R}(\theta^{(c)}_t; S^{(t)}_\T, \lambda) - \nabla_\theta \hat{R}(\theta^{(c)}_t; S^{(t)}_\V, \lambda)}_2^2$
   \ENDFOR
   \ENDFOR
   \STATE $\lambda \leftarrow \lambda - \widetilde{\nabla}^K_\lambda \frac{1}{C}\sum_{c = 1}^C \hat{R}(\theta^{(c)}_{T-1}; S_\V, \lambda) - \frac{1}{C} \sum_{c = 1}^C X_c / \sqrt{Y_c}$ 
   \ENDWHILE
\end{algorithmic}
\end{algorithm}

We next show how one might approach optimization of the regularizer in Eq.~\ref{eq:ho-stability} using the online convex optimization framework. To see how the regularizer might be optimized online, we apply the triangle inequality to the definition of the regularizer in Eq.~\ref{eq:ho-stability} to derive Eq.~\ref{eq:trianglereg}.
\begin{align}
    \sqrt{\sum_{t = 0}^{T-1} d^2_{2,\nu_t}\left (p^1(\theta | S^{(t)}_\T, \lambda), p^1(\theta | S^{(t)}_\V, \lambda)\right)} \leq\sum_{t = 0}^{T - 1} d_{2,\nu_t}\left(p^1(\theta_t | S^{(t)}_\T, \lambda), p^1(\theta_t | S^{(t)}_\V, \lambda) \right) \label{eq:trianglereg}\eqperiod
\end{align}

If we assume that $d_{2,\nu_t}\left(p^1(\theta_t | S^{(t)}_\T, \lambda), p^1(\theta_t | S^{(t)}_\V, \lambda) \right)$ is convex for any choice of $t$, we can justify the use of online gradient descent to minimize the value of Eq.~\ref{eq:trianglereg} \cite{hazan2016introduction}. By updating $\lambda$ on-the-fly for every parameter step, we might hope to reduce the value of Eq.~\ref{eq:trianglereg} more rapidly than if we were to limit ourselves to one hyperparameter gradient step every $T$ parameter iterations.

We note, however, that the final $\lambda$ output by an online convex optimization algorithm is not necessarily well-suited for use over all $T$ steps of optimization. Guarantees on the optimality of the hyperparameters selected using an online convex optimization algorithm can only be provided when so-called ``online-to-batch'' conversion is feasible \cite{hazan2016introduction}. This would require assuming that the $\theta_t$ observed at each iteration are independent and identically distributed; this assumption is clearly violated in Eq.~\ref{eq:trianglereg} since $\theta_{t + 1}$ directly depends upon $\theta_t$.

In Algorithm~\ref{alg:ho_stability_alg_online}, we provide pseudocode for an online version of Algorithm~\ref{alg:ho_stability_alg}. Analogous to the offline case, we assume that we use online gradient descent with learning rate 1 to optimize the hyperparameters and that the user has chosen some convergence criterion for the hyperparameter optimization.
\begin{algorithm}[ht!]
   \caption{Online Optimization of Eq.~\ref{eq:ho-stability}}
   \label{alg:ho_stability_alg_online}
\begin{algorithmic}
   \STATE {\bfseries Input:} $S_\T$, $S_\V$, $T$, $\{\eta_t\}_{t = 0}^{T-1}$, $P_0$, $C$, $\lambda_0$, $K$
   \STATE {\bfseries Output:} $\lambda$
   \WHILE{not converged}
   \STATE $\theta^{(1)}_0,\dots,\theta^{(C)}_0 \sim P_0$
   \FOR{$t=0$ {\bfseries to} $T-1$}
   \FOR{$c=1$ {\bfseries to} $C$}
   \STATE $\theta^{(c)}_{t + 1} \leftarrow \theta^{(c)}_t - \eta_t \nabla_\theta \hat{R}(\theta^{(c)}_t; S^{(t)}_\T, \lambda)$
   \STATE $g_c \leftarrow \zeta \widetilde{\nabla}^K_\lambda \norm{\nabla_\theta \hat{R}(\theta^{(c)}_t; S^{(t)}_\T, \lambda) - \nabla_\theta \hat{R}(\theta^{(c)}_t; S^{(t)}_\V, \lambda)}_2$
   \ENDFOR
   \STATE $\lambda \leftarrow \lambda - \frac{1}{C}\sum_{c = 1}^C g_c$
   \ENDFOR
   \STATE $\lambda \leftarrow \lambda - \widetilde{\nabla}^K_\lambda \frac{1}{C}\sum_{c = 1}^C \hat{R}(\theta^{(c)}_{T-1}; S_\V, \lambda)$
   \ENDWHILE
\end{algorithmic}
\end{algorithm}
\section{Experiments}
All experiments requiring backpropagation were performed using the PyTorch computational framework \cite{torchcitation}. In Section~\ref{sec:regappendix}, the ResNet-18 and ResNet-34 experiments were run on Nvidia GeForce GTX 980 Ti GPUs, and the linear classifier experiments were run on a Nvidia GeForce GTX 1080 GPU.

\subsection{Feature Selection} \label{sec:featappendix}
As we describe in Section~\ref{sec:featselection}, we investigate two versions of Freedman's problem. In the first, we generate $500$ input-label pairs, $\{(x_i, y_i)\}_{i = 1}^{500}$ where $x_i \sim \mathcal{N}(\mathbf{0_{500}}, \mathbb{I}_{\mathbf{500}})$ and $y_i \sim \mathcal{N}(0,1)$, which we split into equally sized training and validation sets. In the second, we also generate $500$ input-label pairs, $\{(x_i, y_i)\}_{i = 1}^{500}$ where $x_i \sim \mathcal{N}(\mathbf{0_{500}}, \mathbb{I}_{\mathbf{500}})$. In this version, however, we sample $y_i$ so that there are two true predictors $x^{(1)}_i$ and $x^{(2)}_i$. We thus define $y_i = \frac{1}{\sqrt{6}} \left (x^{(1)}_i + x^{(2)}_i + \mathcal{N}(0, 2) \right)$. Note that the superscript on $x_i$ denotes an index for the vector $x_i$.

The choice of $\sigma^2$ for the noise added to $y_i$ in the second experiment determines how much signal is present in the two ``true'' predictors provided to the feature selection algorithm. $\sigma^2 = 2$ was chosen arbitrarily; the results are qualitatively similar for other choices of $\sigma^2$.

To estimate Eq.~\ref{eq:ho-stability} with $\zeta^2 = 0.025$, we run $50$ chains of Langevin Dynamics with learning rate $0.1$ on the training set for $50$ steps each. To improve the bound that we are implicitly minimizing, we also tune $\tau$ for the Gibbs posterior sampled by each chain of Langevin Dynamics. By contrast, when we estimate Eq.~\ref{eq:bilevel-ho}, we evaluate the validation risk at the exact training optimum. For both objectives, we then select predictors using forward selection.  

Figure~\ref{fig:freedman_null} shows that unlike standard hyperparameter optimization (Eq.~\ref{eq:bilevel-ho}), forward selection using Eq.~\ref{eq:ho-stability} optimizes an objective that is positively correlated with the true out-of-sample error. Similar to Figure~\ref{fig:freedman_nonnull}, the test set mean-squared-error (MSE) and the hyperparameter objective are plotted using the left and right axes, respectively; the shaded regions on the plots are also $95$\% confidence intervals constructed from $50$ optimizations initialized from randomly sampled parameters. To simplify comparison, Figure~\ref{fig:freedman_nonnull_app} replicates Figure~\ref{fig:freedman_nonnull}.
\begin{figure}[p]
     \centering
     \begin{subfigure}[t]{0.49\linewidth}
         \centering
         \includegraphics[width=\linewidth]{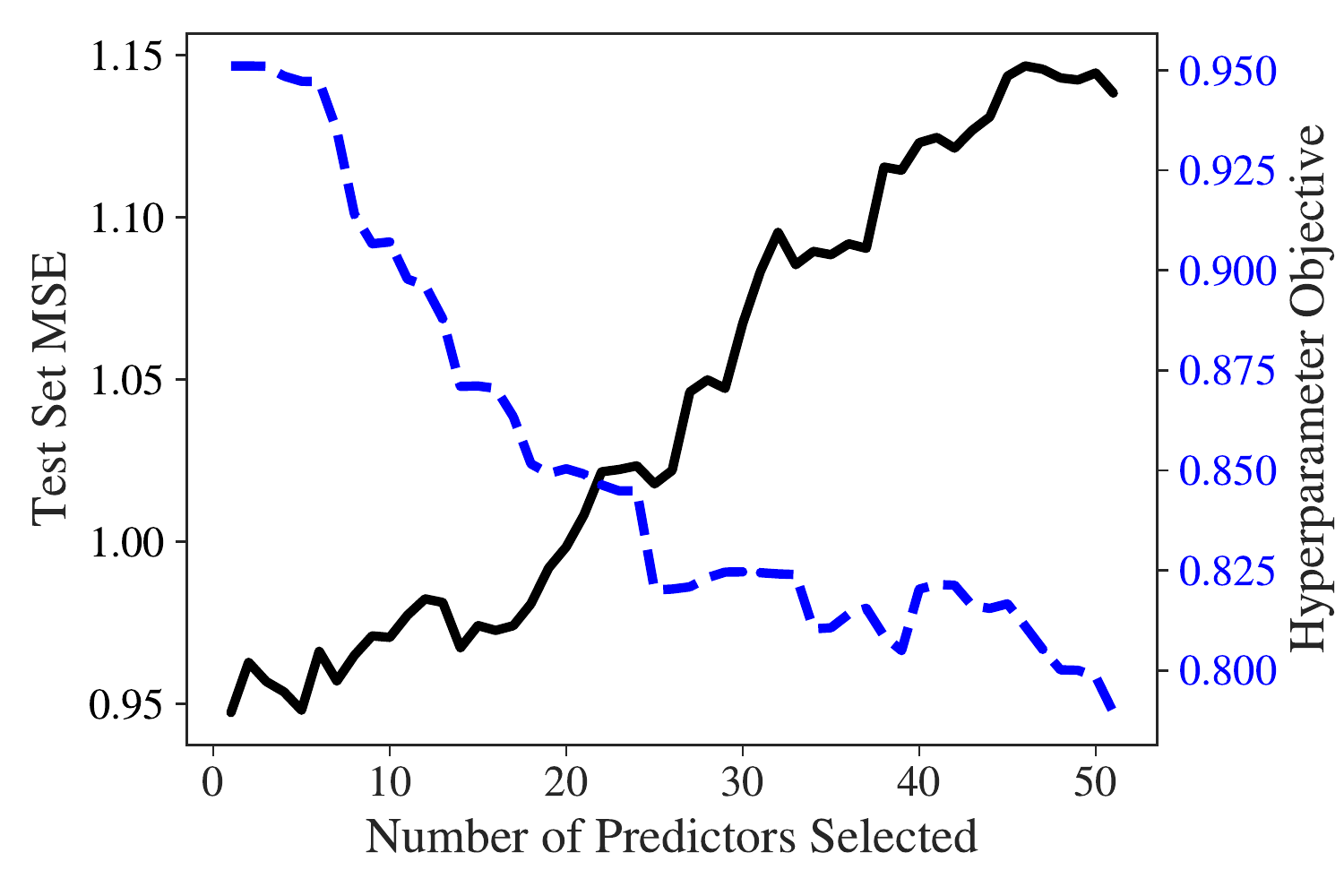}
         \caption{Test set MSE vs. objective for Eq.~\ref{eq:bilevel-ho}}
         \label{fig:baddecision_null}
     \end{subfigure}
     \hfill
     \begin{subfigure}[t]{0.49\linewidth}
         \centering
         \includegraphics[width=\linewidth]{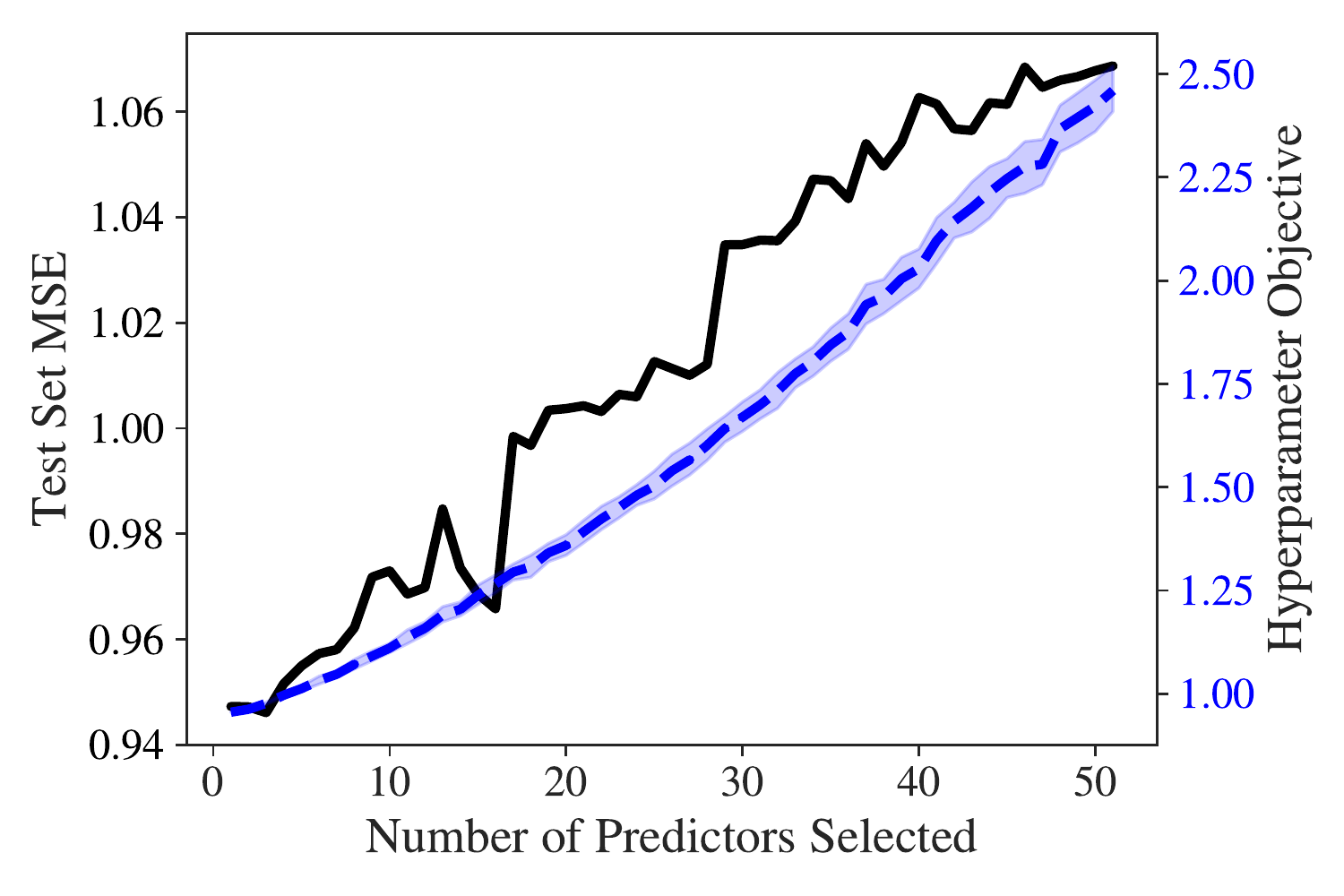}
         \caption{Test set MSE vs. objective for Eq.~\ref{eq:ho-stability}.}
         \label{fig:gooddecision_null}
     \end{subfigure}
     \caption{Comparing objectives to test set MSE for Freedman's paradox with no true predictors.}
     \label{fig:freedman_null}
\end{figure}
\begin{figure}[p]
     \centering
     \begin{subfigure}[t]{0.49\linewidth}
         \centering
         \includegraphics[width=\linewidth]{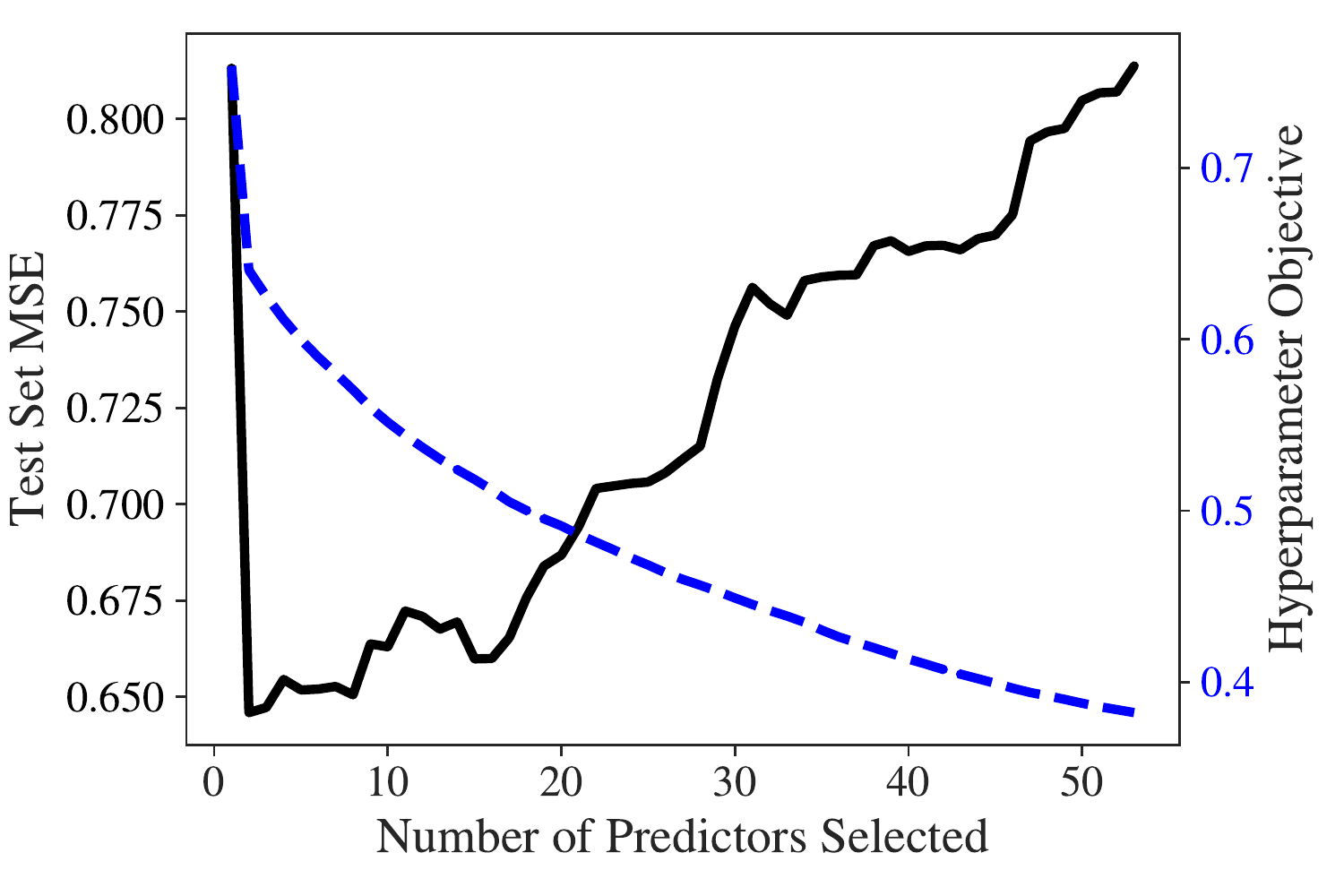}
         \caption{Test set MSE vs. objective for Eq.~\ref{eq:bilevel-ho}.}
         \label{fig:baddecision_nonnull_app}
     \end{subfigure}
     \hfill
     \begin{subfigure}[t]{0.49\linewidth}
         \centering
         \includegraphics[width=\linewidth]{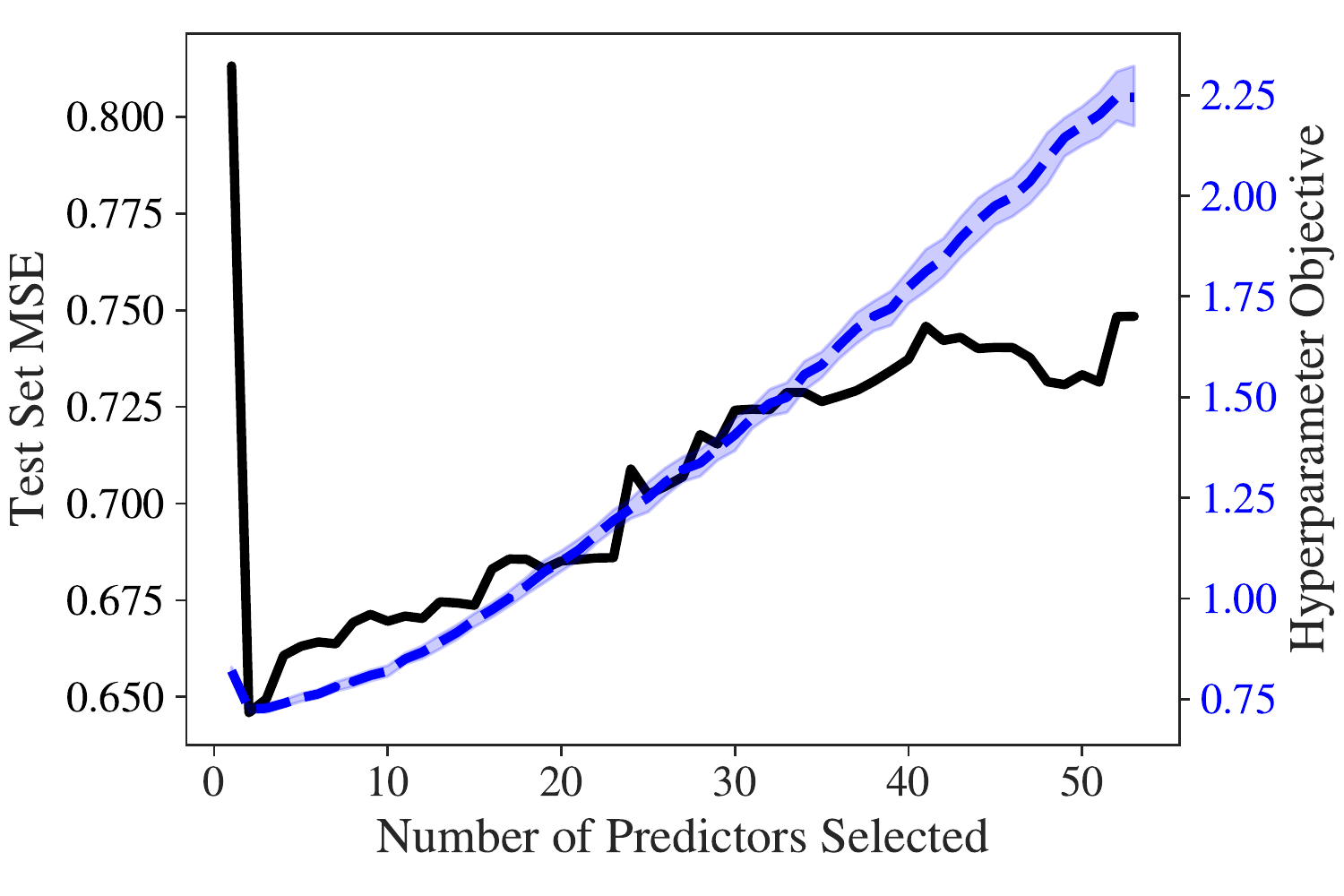}
         \caption{Test set MSE vs. objective for Eq.~\ref{eq:ho-stability}. }
         \label{fig:gooddecision_nonnull_app}
     \end{subfigure}
     \caption{Comparing objectives to test set MSE for Freedman's paradox with two true predictors. }
     \label{fig:freedman_nonnull_app}
\end{figure}
\begin{figure}[p]
     \centering
     \begin{subfigure}[t]{0.49\linewidth}
         \centering
         \includegraphics[width=\linewidth]{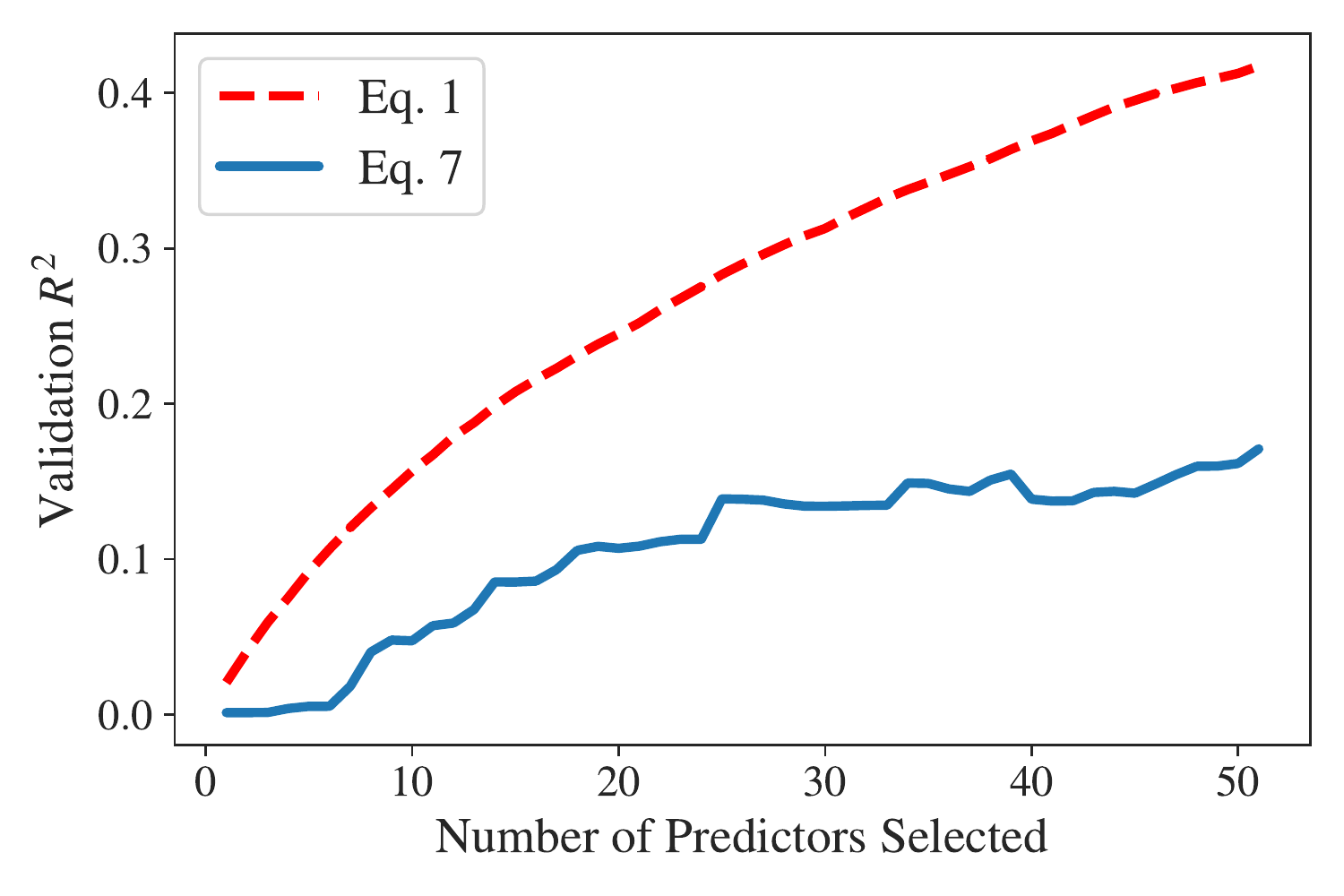}
         \caption{Validation $R^2$ for the experiment with no true predictors.}
         \label{fig:validation_r2_null}
     \end{subfigure}
     \hfill
     \begin{subfigure}[t]{0.49\linewidth}
         \centering
         \includegraphics[width=\linewidth]{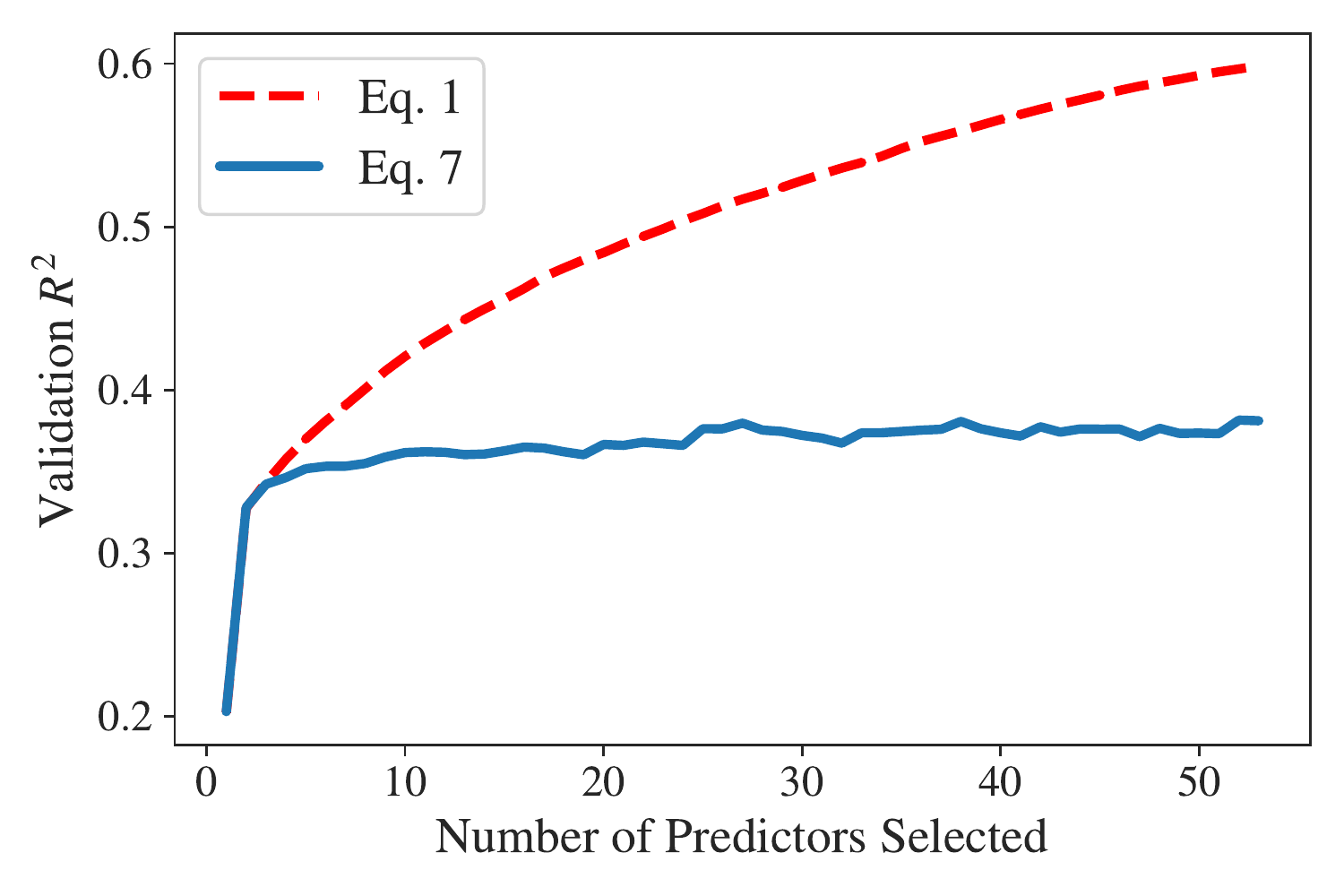}
         \caption{Validation $R^2$ for the experiment with two true predictors.}
         \label{fig:validation_r2_nonnull}
     \end{subfigure}
     \caption{Validation $R^2$ vs. the number of predictors selected for both versions of Freedman's experiment. }
     \label{fig:validation_r2}
\end{figure}

In the paper introducing this experiment, \citet{freedman1983note} uses the $R^2$ of the linear model to assess the extent of the model's overfitting. Figure~\ref{fig:validation_r2} demonstrates that selecting features using Eq.~\ref{eq:bilevel-ho} (Figures~\ref{fig:validation_r2_null} and \ref{fig:validation_r2_nonnull}, dashed lines) results in a model whose fit to the validation set improves dramatically (and misleadingly) over the course of forward selection. By contrast, when selecting features using Eq.~\ref{eq:ho-stability}, the model's fit to the validation set improves more slowly (Figure~\ref{fig:validation_r2_null}, solid line) or not at all once the truly predictive features are incorporated (Figure~\ref{fig:validation_r2_nonnull}, solid line).

Using Eq.~\ref{eq:ho-stability} to select features achieves similar accuracy when compared to using a more problem-specific measure of performance such as the Akaike Information Criterion (AIC). Since we aim to improve model performance on unseen validation data, we use the measure of out-of-sample error as a baseline where $p$ denotes the number of features included:
\begin{align*}
    \text{AIC}(\theta) = 2p + n_\V \hat{R}_\V(\theta; S_\V, \lambda)\eqperiod
\end{align*}
Note that unlike Eq.~\ref{eq:ho-stability}, this proxy for out-of-sample performance can only be applied to problems in which the hyperparameter selection affects the number of features in the model. Figure~\ref{fig:freedman_AIC} shows that the AIC can be used to select the correct model in both instances of Freedman's paradox we study (Figure~\ref{fig:freedman_AIC}). But while the validation $R^2$ of models selected using Eq.~\ref{eq:ho-stability} remains flat as we include spuriously correlated features (Figure~\ref{fig:validation_r2_nonnull}), minimizing the AIC does \emph{not} reduce validation set overfitting for models for which we deliberately include spuriously correlated predictors.

\begin{figure}[ht!]
     \centering
     \begin{subfigure}[t]{0.49\linewidth}
         \centering
         \includegraphics[width=\linewidth]{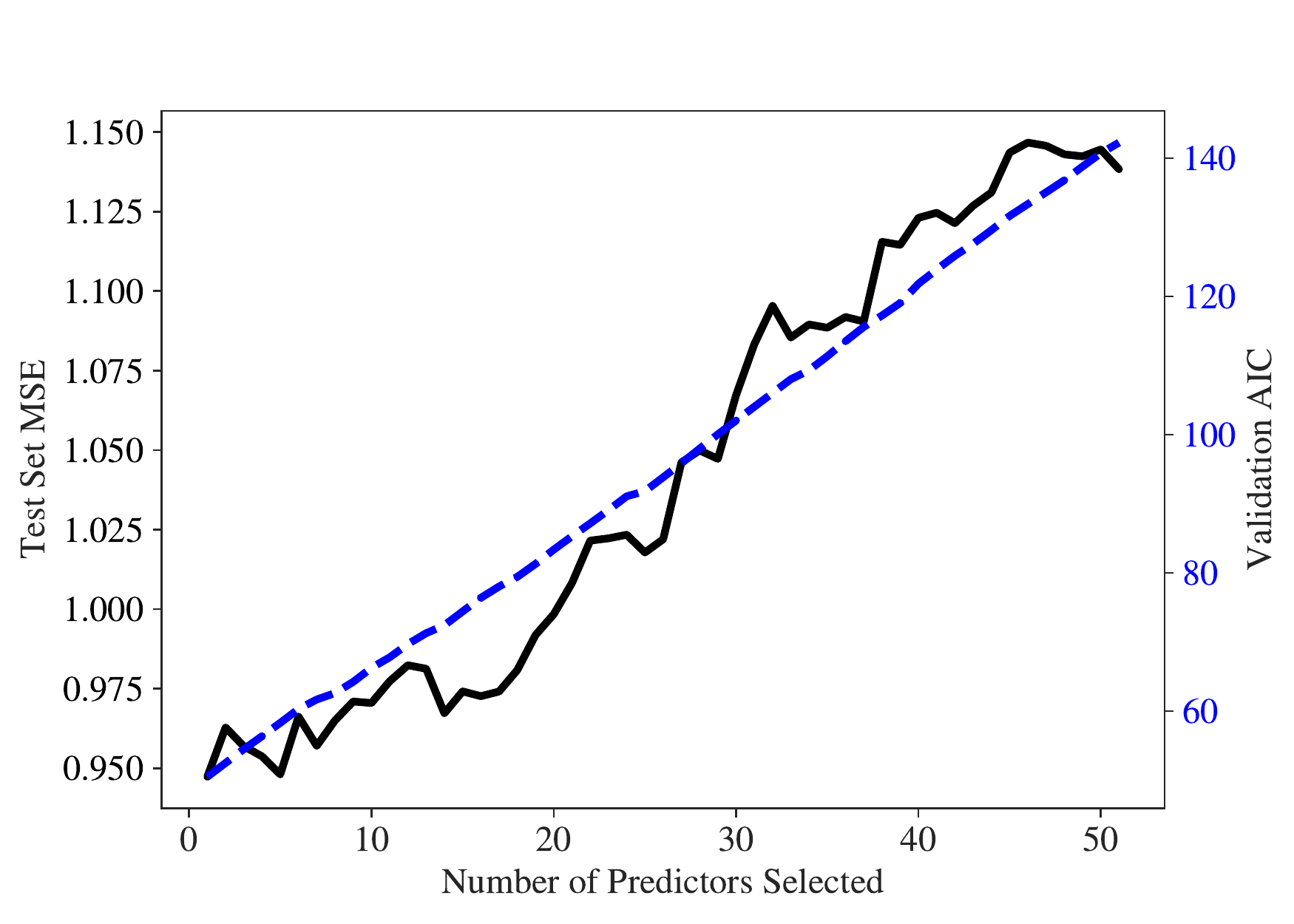}
         \caption{Test set MSE vs. Validation AIC with no true predictors.}
         \label{fig:validation_aic_null}
     \end{subfigure}
     \hfill
     \begin{subfigure}[t]{0.49\linewidth}
         \centering
         \includegraphics[width=\linewidth]{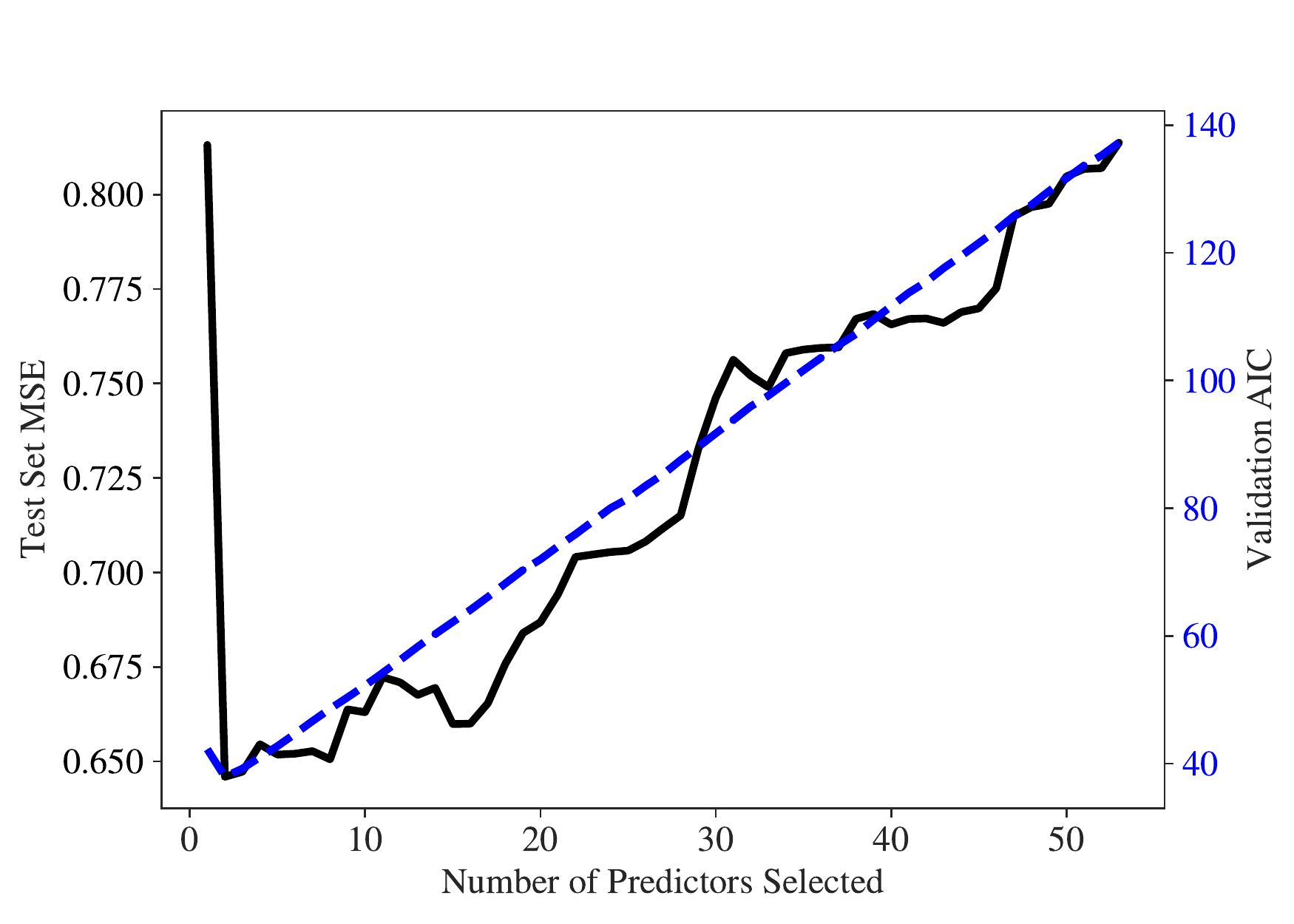}
         \caption{Test set MSE vs. Validation AIC with two true predictors.}
         \label{fig:validation_aic_nonnull}
     \end{subfigure}
     \caption{Comparing the validation AIC to test set MSE for Freedman's paradox.}
     \label{fig:freedman_AIC}
\end{figure}

\subsection{Regularization Penalty} \label{sec:regappendix}
We also evaluate our approach using an example of validation set overfitting first described in Section 5.2 of \citet{lorraine2019optimizing}. We run a similar experiment by fitting a per-parameter weight decay hyperparameter for the following classifiers: a one-layer fully connected network (i.e., a linear classifier), ResNet-18, and ResNet-34 \cite{he2016deep}. We fit each of these classifiers to 50 randomly sampled training images from the MNIST and CIFAR-10 data sets and we use 50 randomly sampled validation images to evaluate and train the hyperparameters \cite{lecun1998gradient, krizhevsky2009learning}. We measure test set error using the standard testing partition provided by the creators of both data sets. Note that links to these data sets can be found in the references below.

The details of our optimization setup follow. To optimize the neural network parameters, we minimize a cross-entropy loss using the Adam optimizer with a learning rate of $10^{-4}$ \cite{kingma2014adam}. We then run $1000$ inner gradient steps on the parameters per outer step on the hyperparameters. To ensure that our results are comparable to those of \citet{lorraine2019optimizing}, we do not re-initialize the neural network parameters for each outer step (matching Algorithm~1 in \citet{lorraine2019optimizing}).

For both the regularized and unregularized objective, we compute the hyperparameter gradient of the empirical validation risk using the $T1-T2$ approximation proposed in \citet{Luketina16}. To optimize Eq.~\ref{eq:ho-stability}, we apply a modified Algorithm~\ref{alg:ho_stability_alg_mt} with $K = 0$. As explained above, Algorithm~\ref{alg:ho_stability_alg_mt} is modified to not resample $\theta_0 \sim P_0$ at each outer step. We use grid search to select the $\zeta$ that minimizes out-of-sample error, though we show in Figures~\ref{fig:linear_scan}--\ref{fig:resnet34_scan} that our results are qualitatively unchanged for a wide range of penalties. For both Eq.~\ref{eq:bilevel-ho} and Eq.~\ref{eq:ho-stability}, we perform gradient descent using the RMSProp optimizer with a learning rate of $10^{-2}$ \cite{tieleman2012lecture}.

For all plots below, the shaded regions are $95$\% confidence intervals constructed from $5$ optimizations initialized from randomly sampled parameters. When applicable, the $\zeta$ used for optimizing Eq.~\ref{eq:ho-stability} is included in the caption of the figure.

In Section~\ref{sec:regpenalty}, we assert that optimizing Eq.~\ref{eq:ho-stability} not only increases out-of-sample accuracy relative to Eq.~\ref{eq:ho-stability}, but also leads to more stable validation loss minimization. Figures~\ref{fig:linear_valid_offline}--\ref{fig:resnet34_valid_offline} substantiate this unintuitive claim. We believe that this result can be traced to the inaccuracy of the $T_1-T_2$ hyperparameter gradient; small errors in the gradient near convergence appear to dramatically affect the validation loss. Recall that in Section~\ref{sec:motivation}, we observed that minimizing the regularizer could be interpreted as minimizing a measure of distance between the training and validation Gibbs posterior distributions. Because this objective is consistent with, though not identical to, the goal of minimizing the validation risk, we speculate that adding the hyperparameter gradient of the regularizer at each outer step hides small errors associated with the $T_1-T_2$ approximation and thus improves convergence. 

Improved (and more expensive) hyperparameter gradient approximations would presumably resolve the instability we observe in Figures~\ref{fig:linear_valid_offline}--\ref{fig:resnet34_valid_offline}. But even when substantial computational resources are available, accurately estimating the hyperparameter gradient with respect to the validation risk remains extremely challenging \cite{Metz19}. A method for improving optimization stability with minimal computational expense is thus desirable. The empirical effect of the regularizer on validation loss convergence strengthens our argument for optimizing Eq.~\ref{eq:ho-stability}. 

\begin{figure}[p]
     \centering
     \begin{subfigure}[t]{0.49\linewidth}
         \centering
         \includegraphics[width=\linewidth]{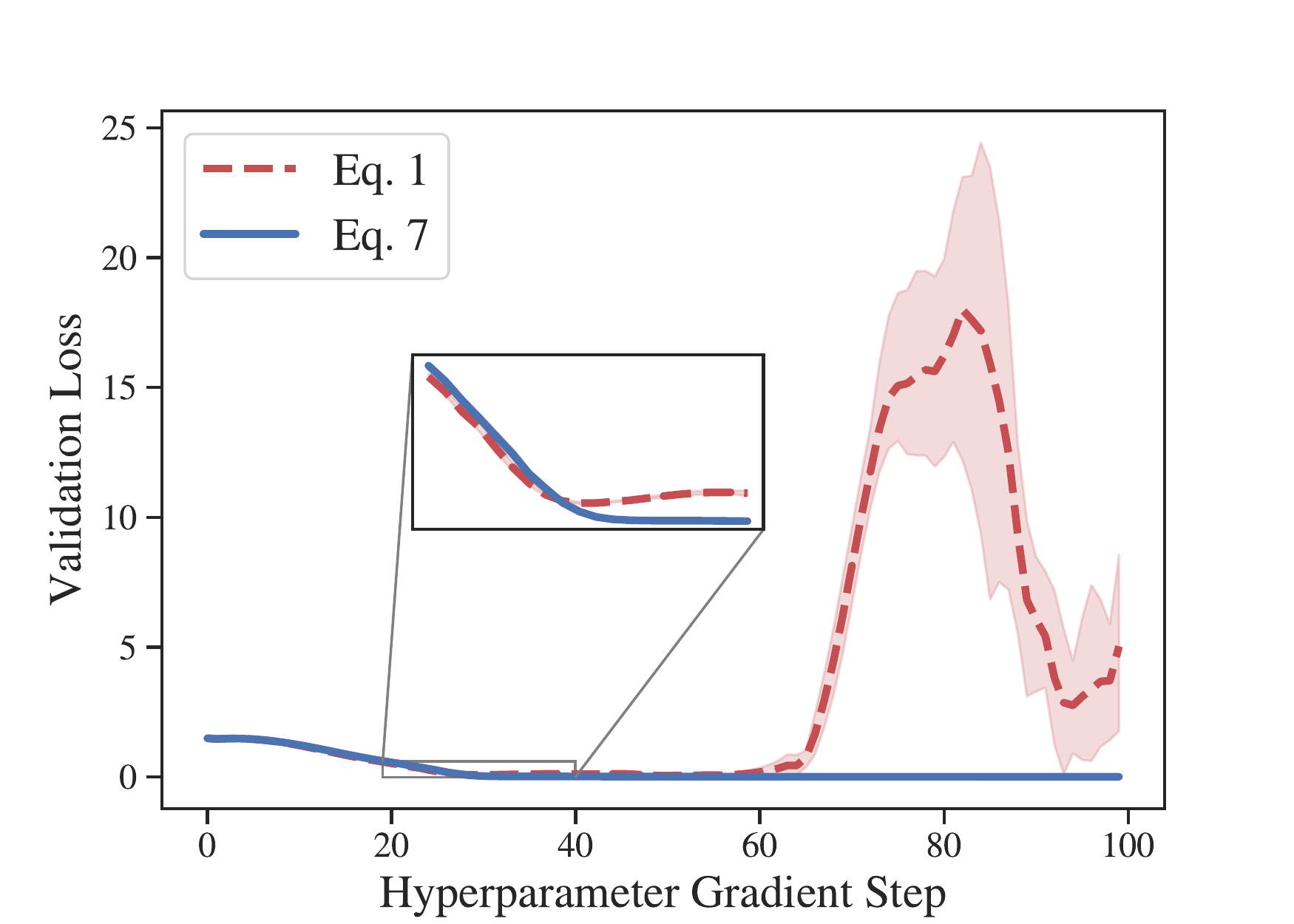}
         \caption{MNIST $\left (\zeta = \num{1.41e-3} \right)$}
         \label{fig:linear_mnist_valid_offline}
     \end{subfigure}
     \hfill
     \begin{subfigure}[t]{0.49\linewidth}
         \centering
         \includegraphics[width=\linewidth]{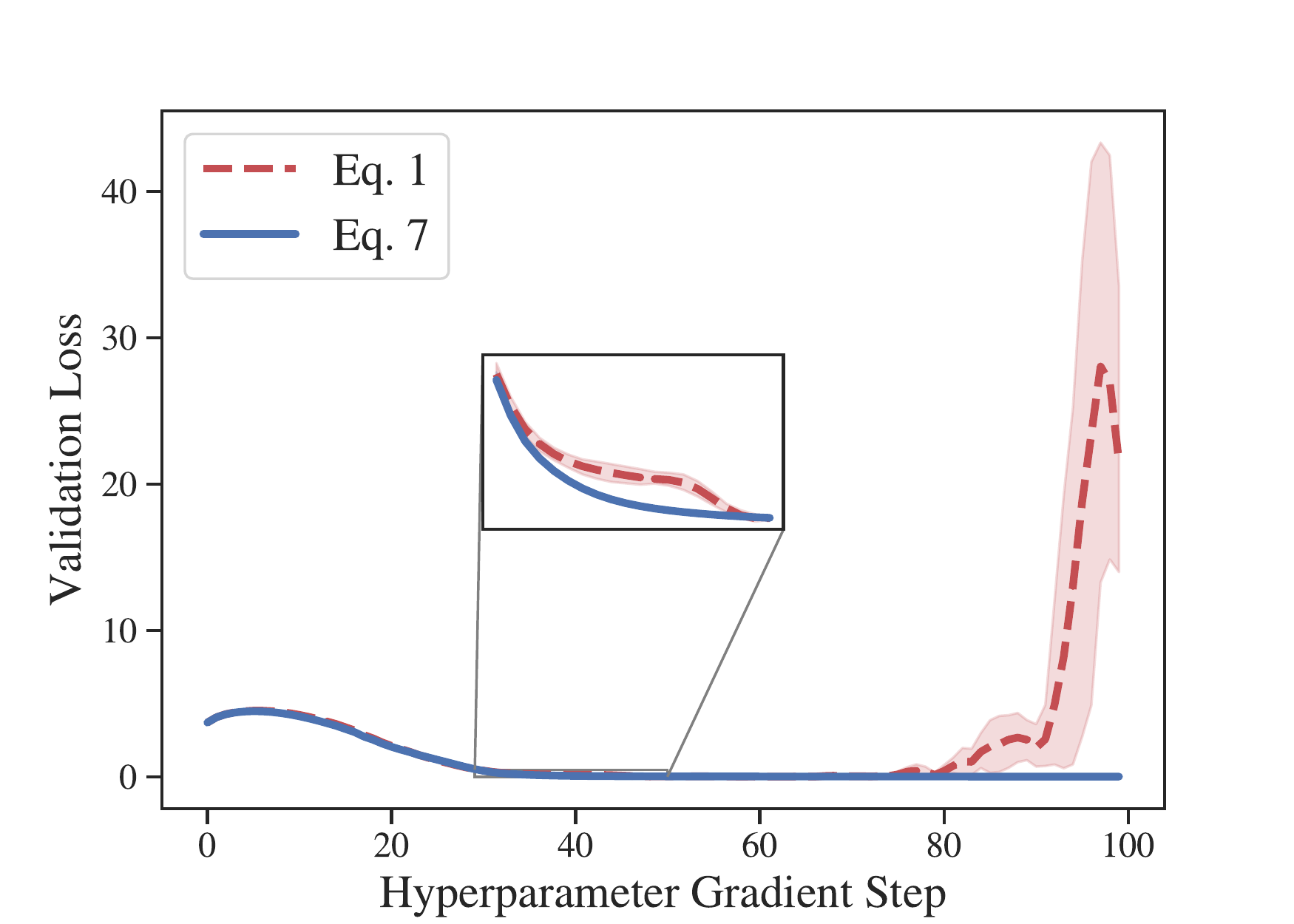}
         \caption{CIFAR-10 $\left (\zeta = \num{3.17e-3} \right)$}
         \label{fig:linear_cifar10_valid_offline}
     \end{subfigure}
     \caption{Validation loss minimization with per-parameter weight decays on a linear classifier using Algorithm~\ref{alg:ho_stability_alg}.}
     \label{fig:linear_valid_offline}
\end{figure}

\begin{figure}[p]
     \centering
     \begin{subfigure}[t]{0.49\linewidth}
         \centering
         \includegraphics[width=\linewidth]{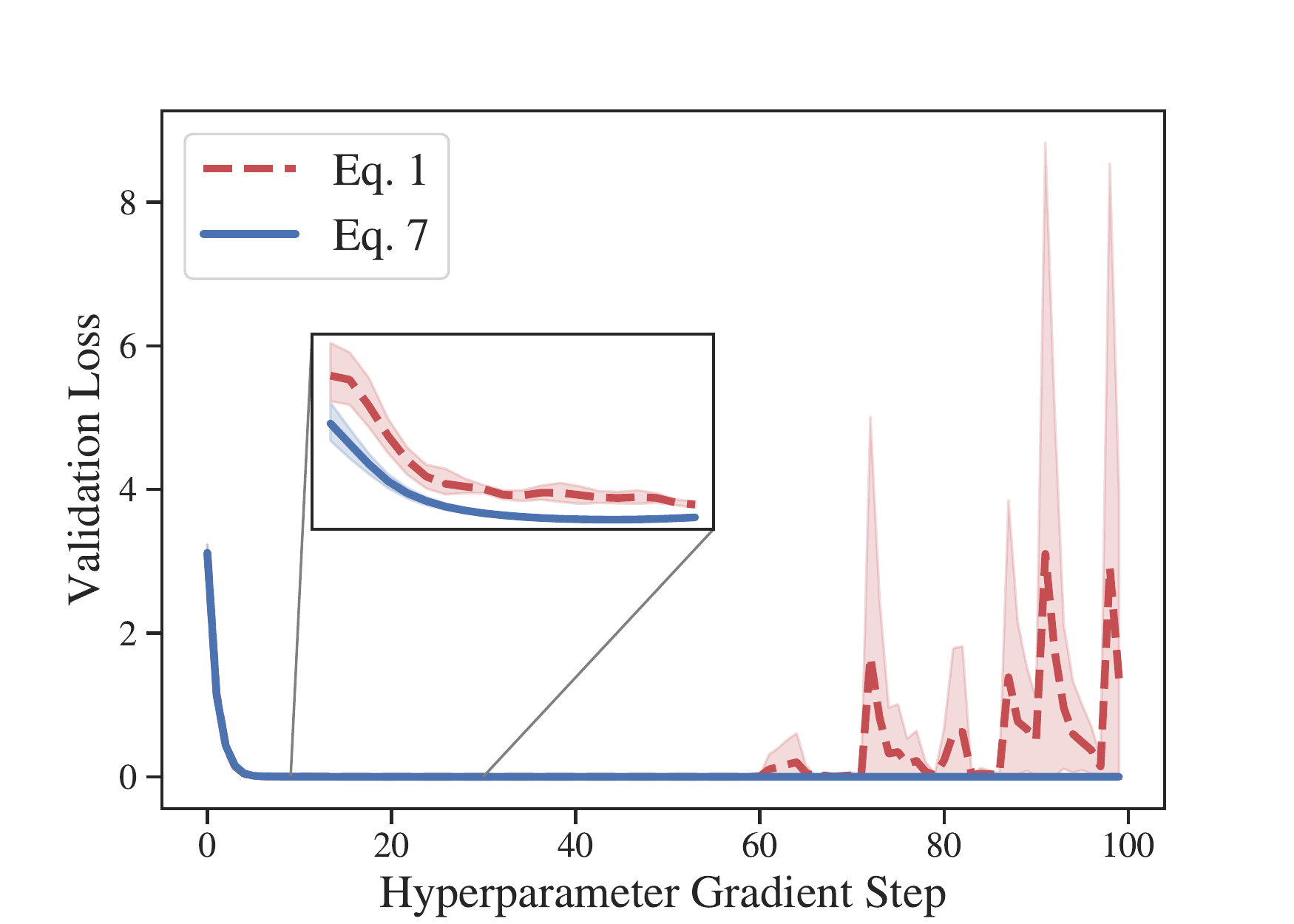}
         \caption{MNIST $\left (\zeta = \num{3.41e-3} \right)$}
         \label{fig:resnet18_mnist_valid_offline}
     \end{subfigure}
     \hfill
     \begin{subfigure}[t]{0.49\linewidth}
         \centering
         \includegraphics[width=\linewidth]{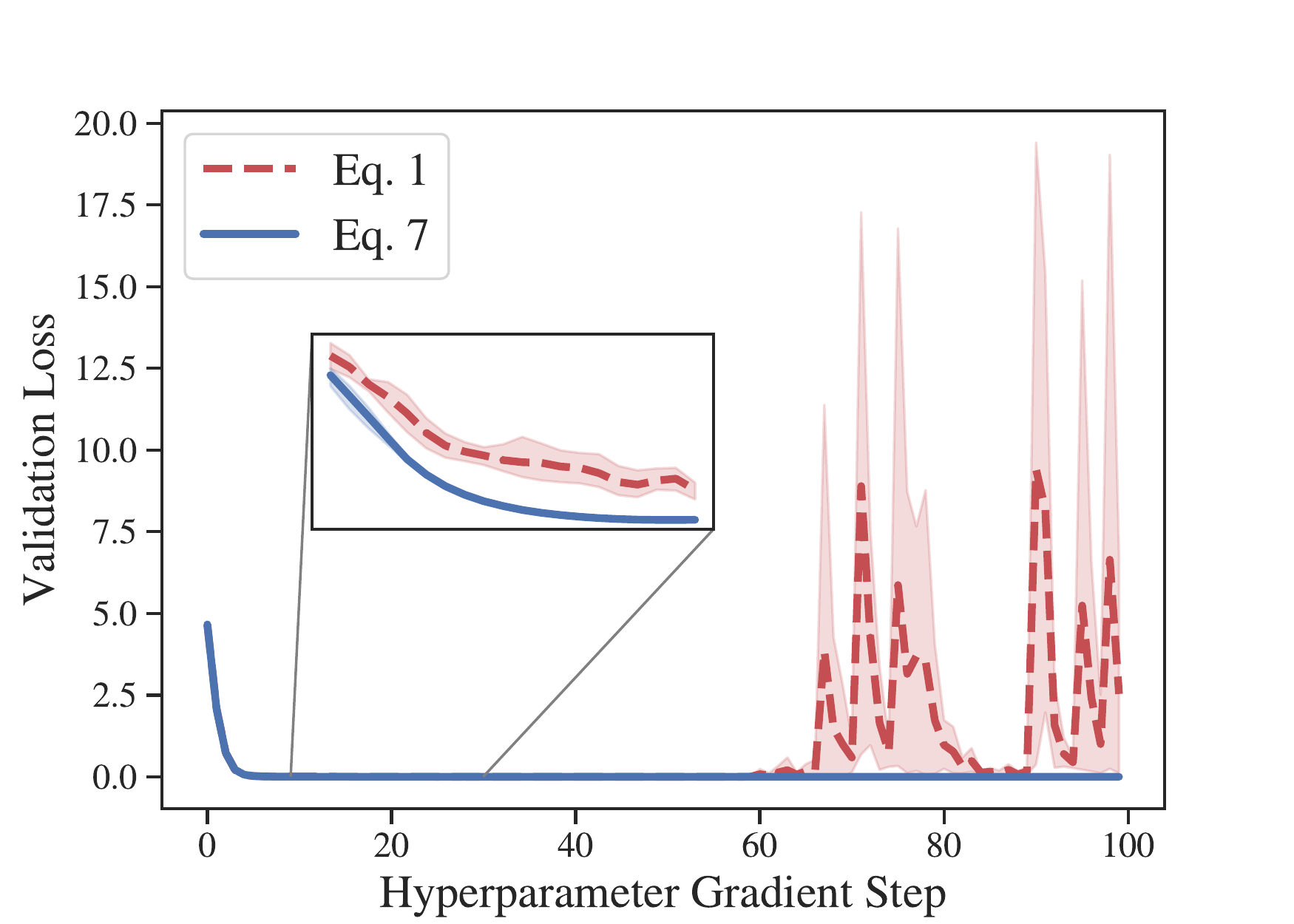}
         \caption{CIFAR-10 $\left (\zeta = \num{2.59e-3} \right)$}
         \label{fig:resnet18_cifar10_valid_offline}
     \end{subfigure}
     \caption{Validation loss minimization with per-parameter weight decays on ResNet-18 using Algorithm~\ref{alg:ho_stability_alg}.}
     \label{fig:resnet18_valid_offline}
\end{figure}

\begin{figure}[p]
     \centering
     \begin{subfigure}[t]{0.49\linewidth}
         \centering
         \includegraphics[width=\linewidth]{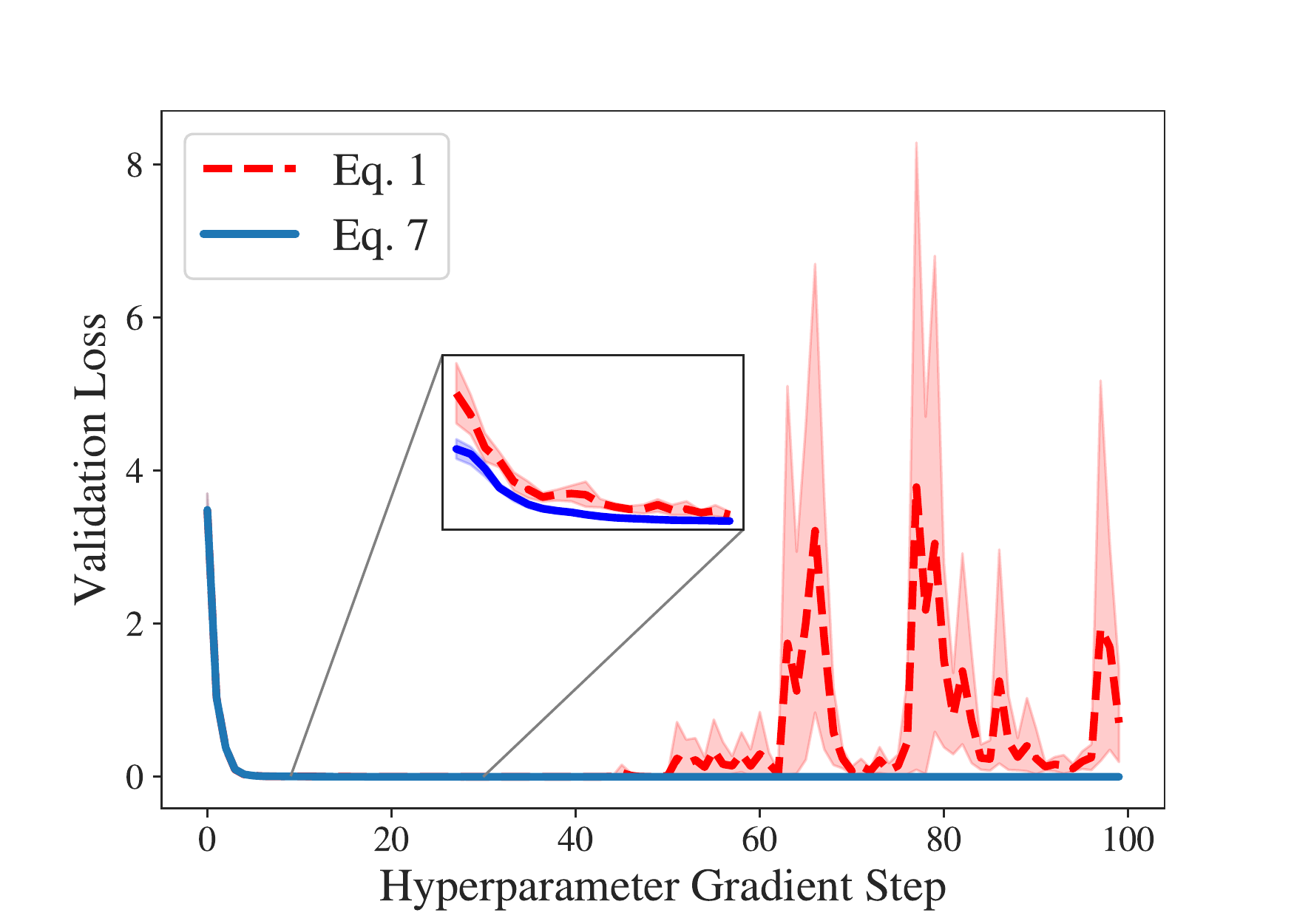}
         \caption{MNIST ($\zeta = \num{3.41e-3}$)}
         \label{fig:resnet34_mnist_valid_offline}
     \end{subfigure}
     \hfill
     \begin{subfigure}[t]{0.49\linewidth}
         \centering
         \includegraphics[width=\linewidth]{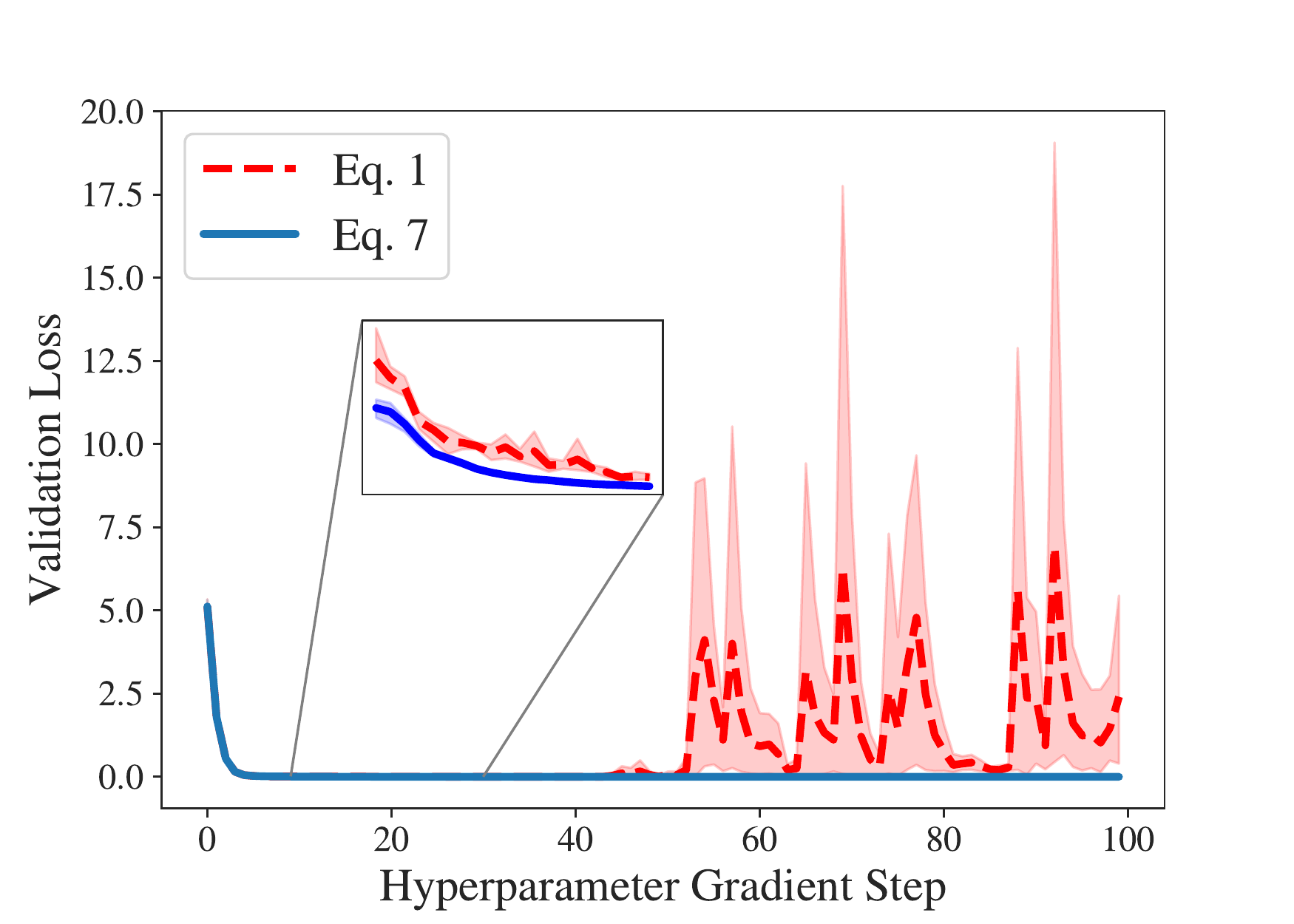}
         \caption{CIFAR-10 ($\zeta = \num{3.17e-3}$)}
         \label{fig:resnet34_cifar10_valid_offline}
     \end{subfigure}
     \caption{Validation loss minimization with per-parameter weight decays on ResNet-34 using Algorithm~\ref{alg:ho_stability_alg}.}
     \label{fig:resnet34_valid_offline}
\end{figure}

Figures~\ref{fig:linear_test_offline}--\ref{fig:resnet34_test_offline} and \ref{fig:linear_test_online}--\ref{fig:resnet34_test_online} extend the results presented in Section~\ref{sec:regpenalty} to additional classifier-dataset pairs. Figures~\ref{fig:linear_test_offline}--\ref{fig:resnet34_test_offline} demonstrate that applying the modified Algorithm~\ref{alg:ho_stability_alg} with $C = 1$ and the $K = 0$ approximation to the hyperparameter gradient of the regularizer substantially improves test set accuracy under a variety of conditions. In Table~\ref{tbl:baseline}, we also compare the results to the following alternative strategy: when optimizing Eq.~\ref{eq:bilevel-ho}, select the classifier with smallest weight norm that achieves the maximum top-1 validation accuracy observed. 

Figures~\ref{fig:linear_test_offline}--\ref{fig:resnet34_test_offline} show a similar improvement in out-of-sample accuracy when Algorithm~\ref{alg:ho_stability_alg_online} is run using the same inputs (i.e., $C = 1$ and $K = 0$) and modifications.

\begin{figure}[p]
     \centering
     \begin{subfigure}[t]{0.49\linewidth}
         \centering
         \includegraphics[width=\linewidth]{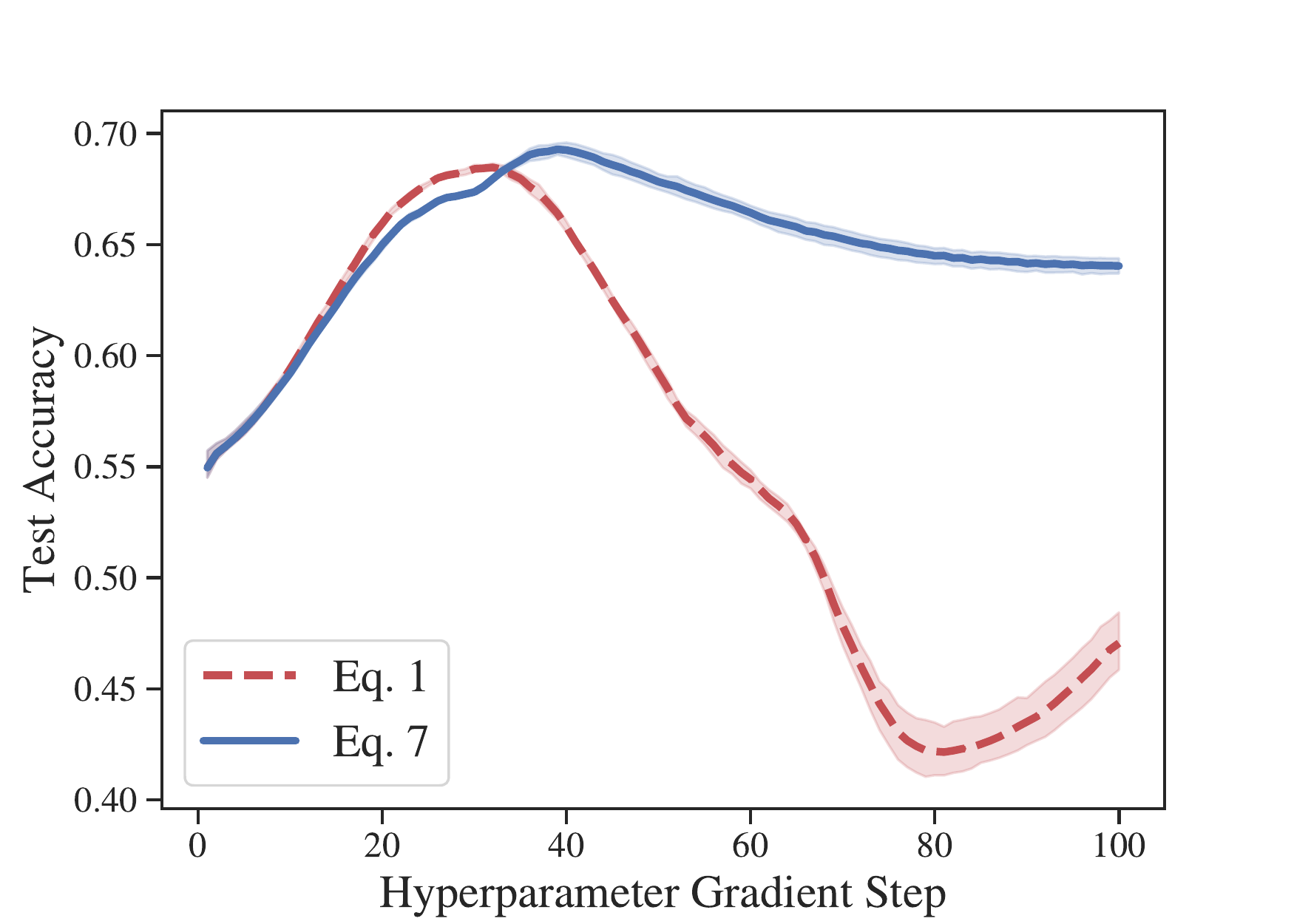}
         \caption{MNIST $\left (\zeta = \num{1.41e-3} \right)$}
         \label{fig:linear_mnist_test_offline}
     \end{subfigure}
     \hfill
     \begin{subfigure}[t]{0.49\linewidth}
         \centering
         \includegraphics[width=\linewidth]{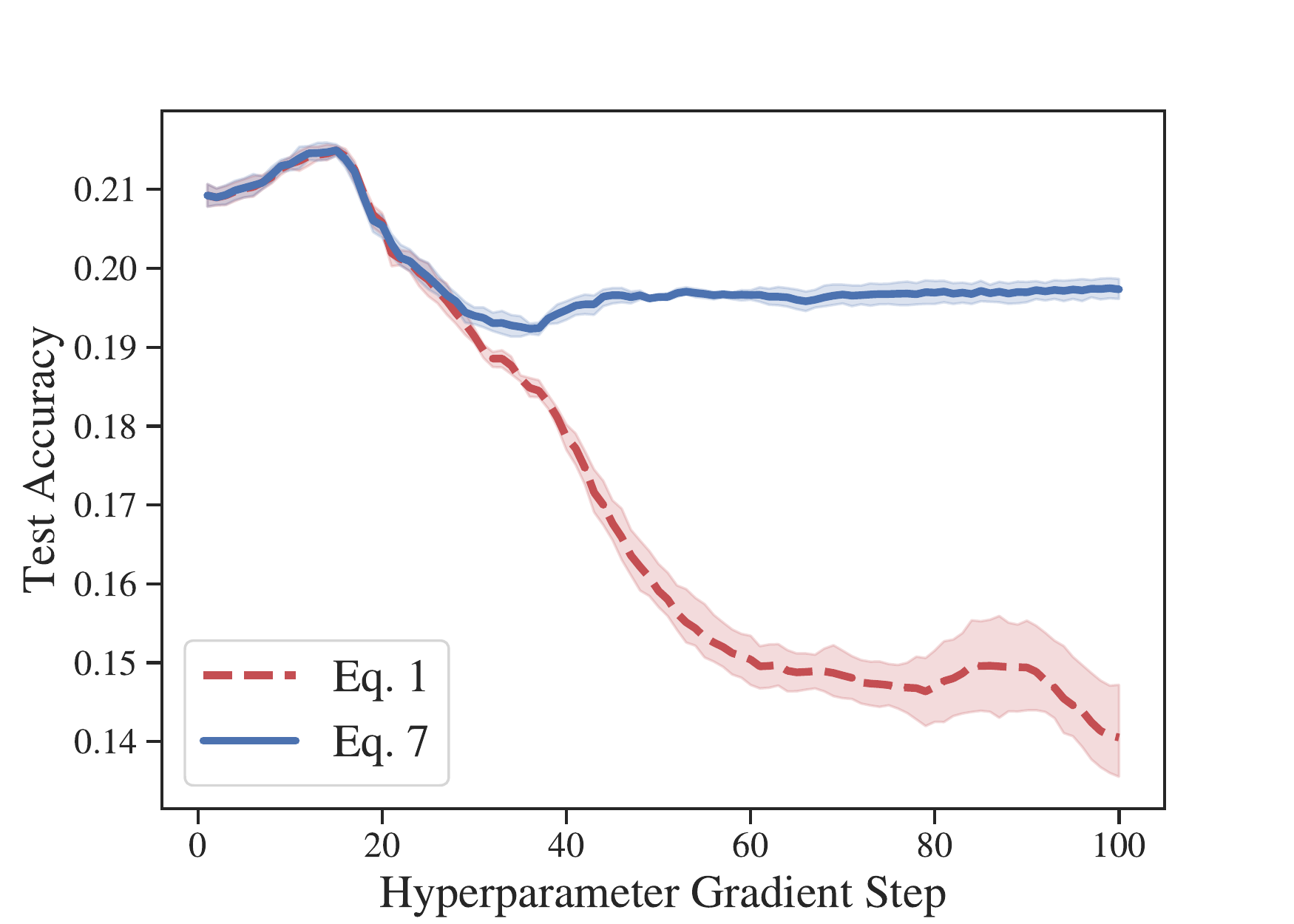}
         \caption{CIFAR-10 $\left (\zeta = \num{3.17e-3} \right)$}
         \label{fig:linear_cifar10_test_offline}
     \end{subfigure}
     \caption{Overfitting a validation set with per-parameter weight decays on a linear classifier using Algorithm~\ref{alg:ho_stability_alg}.}
     \label{fig:linear_test_offline}
\end{figure}

\begin{figure}[p]
     \centering
     \begin{subfigure}[t]{0.49\linewidth}
         \centering
         \includegraphics[width=\linewidth]{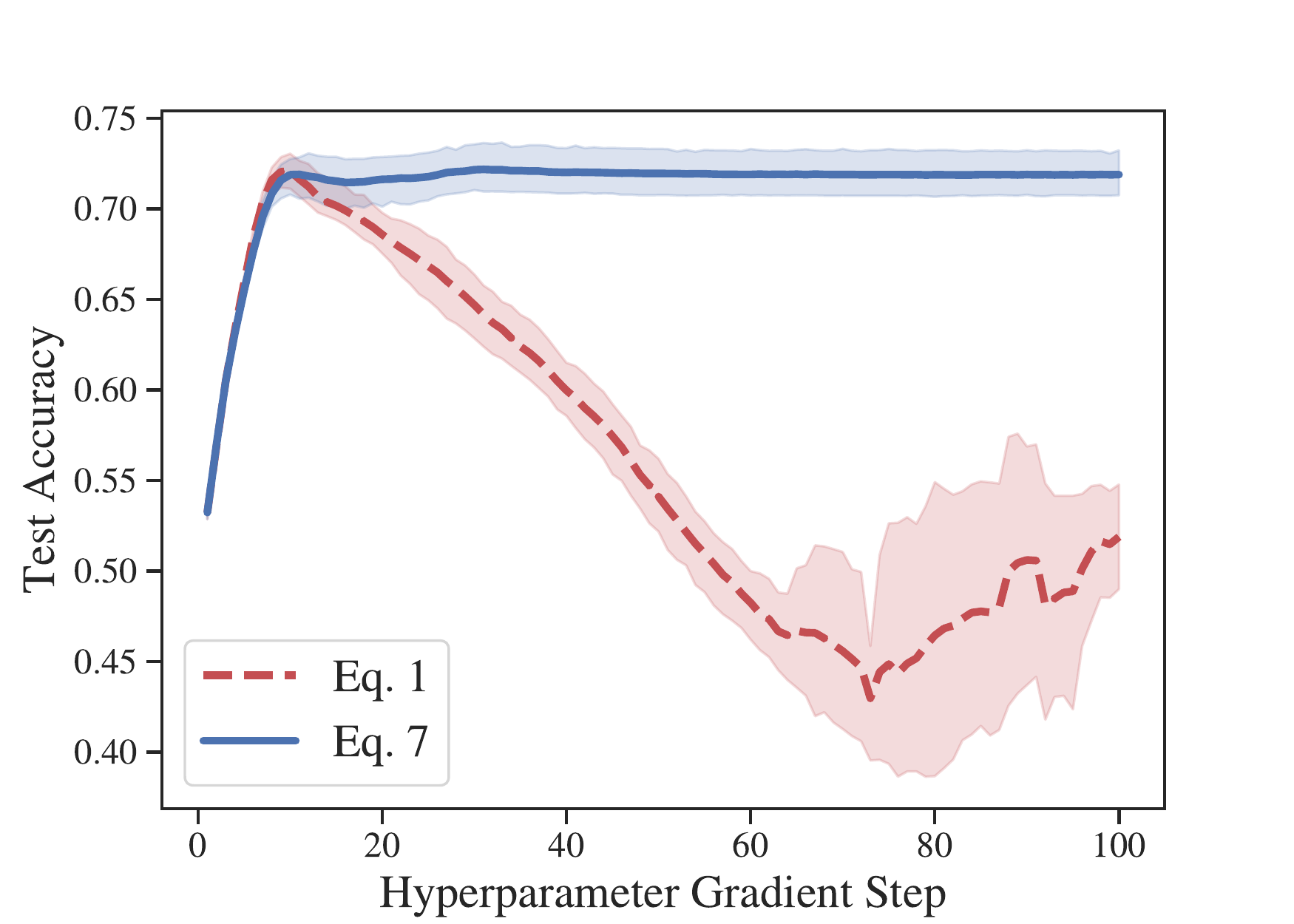}
         \caption{MNIST $\left (\zeta = \num{3.41e-3} \right)$}
         \label{fig:resnet18_mnist_test_offline}
     \end{subfigure}
     \hfill
     \begin{subfigure}[t]{0.49\linewidth}
         \centering
         \includegraphics[width=\linewidth]{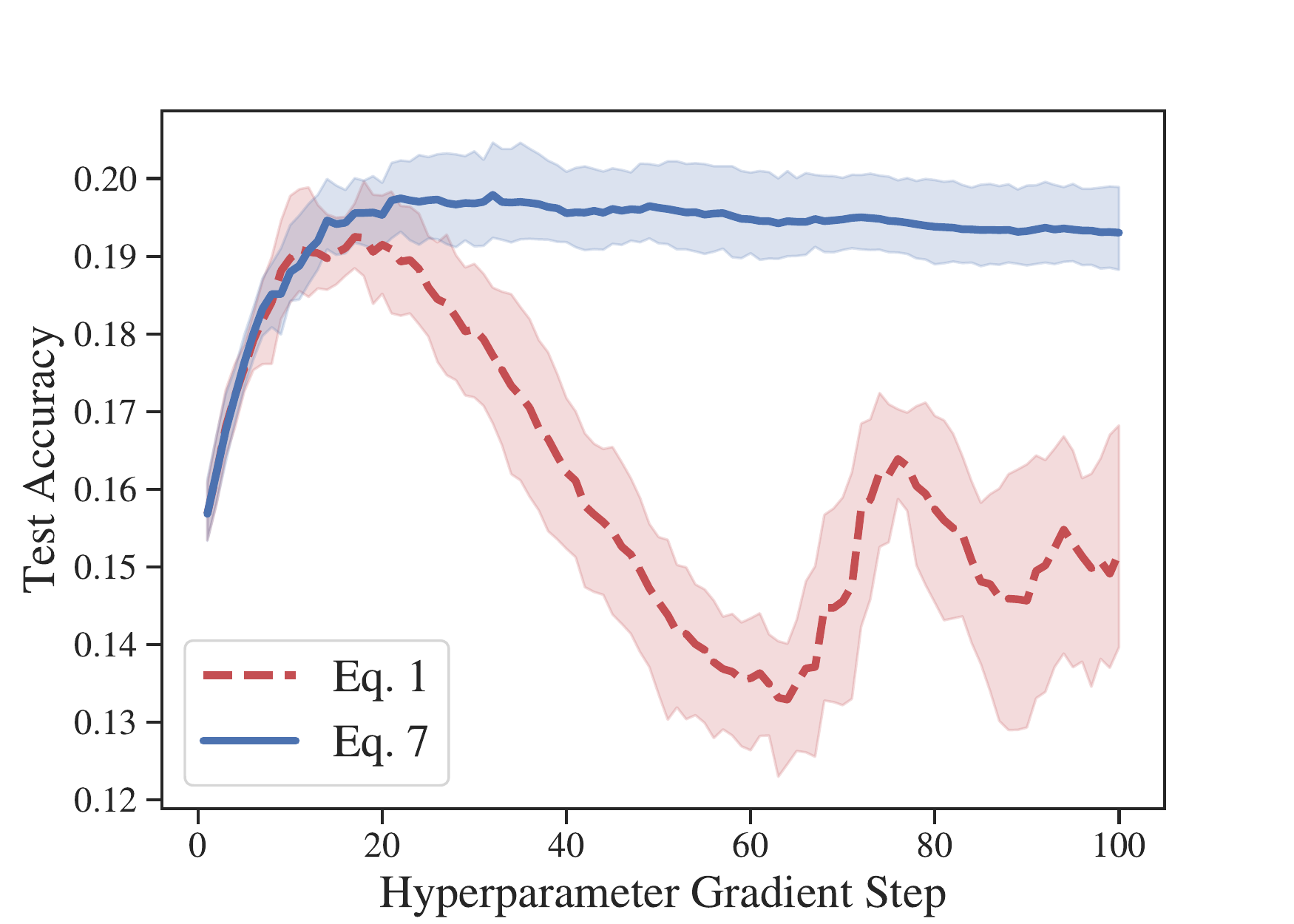}
         \caption{CIFAR-10 $\left (\zeta = \num{2.59e-3} \right)$}
         \label{fig:resnet18_cifar10_test_offline}
     \end{subfigure}
     \caption{Overfitting a validation set with per-parameter weight decays on ResNet-18 using Algorithm~\ref{alg:ho_stability_alg}.}
     \label{fig:resnet18_test_offline}
\end{figure}

\begin{figure}[p]
     \centering
     \begin{subfigure}[t]{0.49\linewidth}
         \centering
         \includegraphics[width=\linewidth]{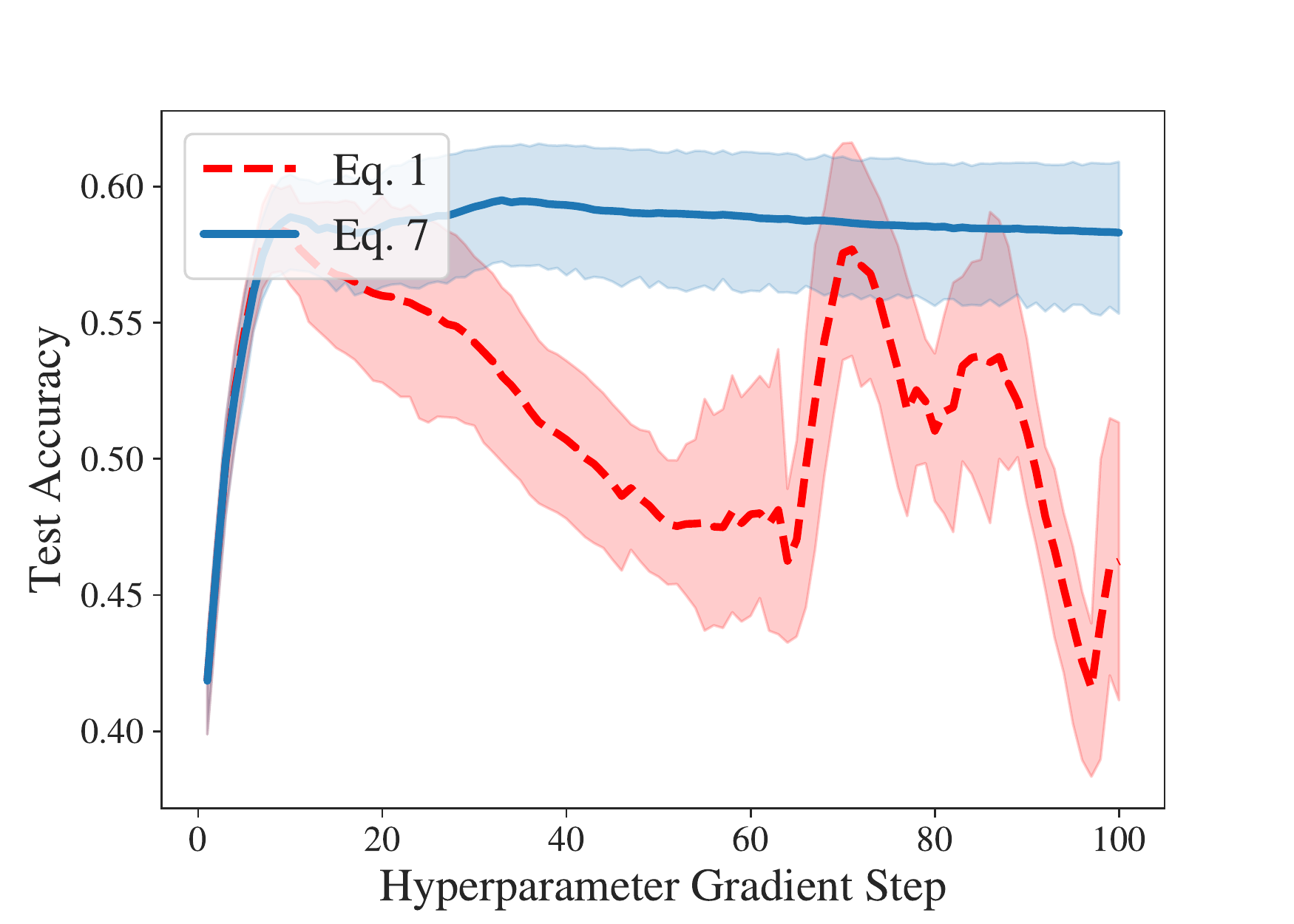}
         \caption{MNIST ($\zeta = \num{3.41e-3}$)}
         \label{fig:resnet34_mnist_test_offline}
     \end{subfigure}
     \hfill
     \begin{subfigure}[t]{0.49\linewidth}
         \centering
         \includegraphics[width=\linewidth]{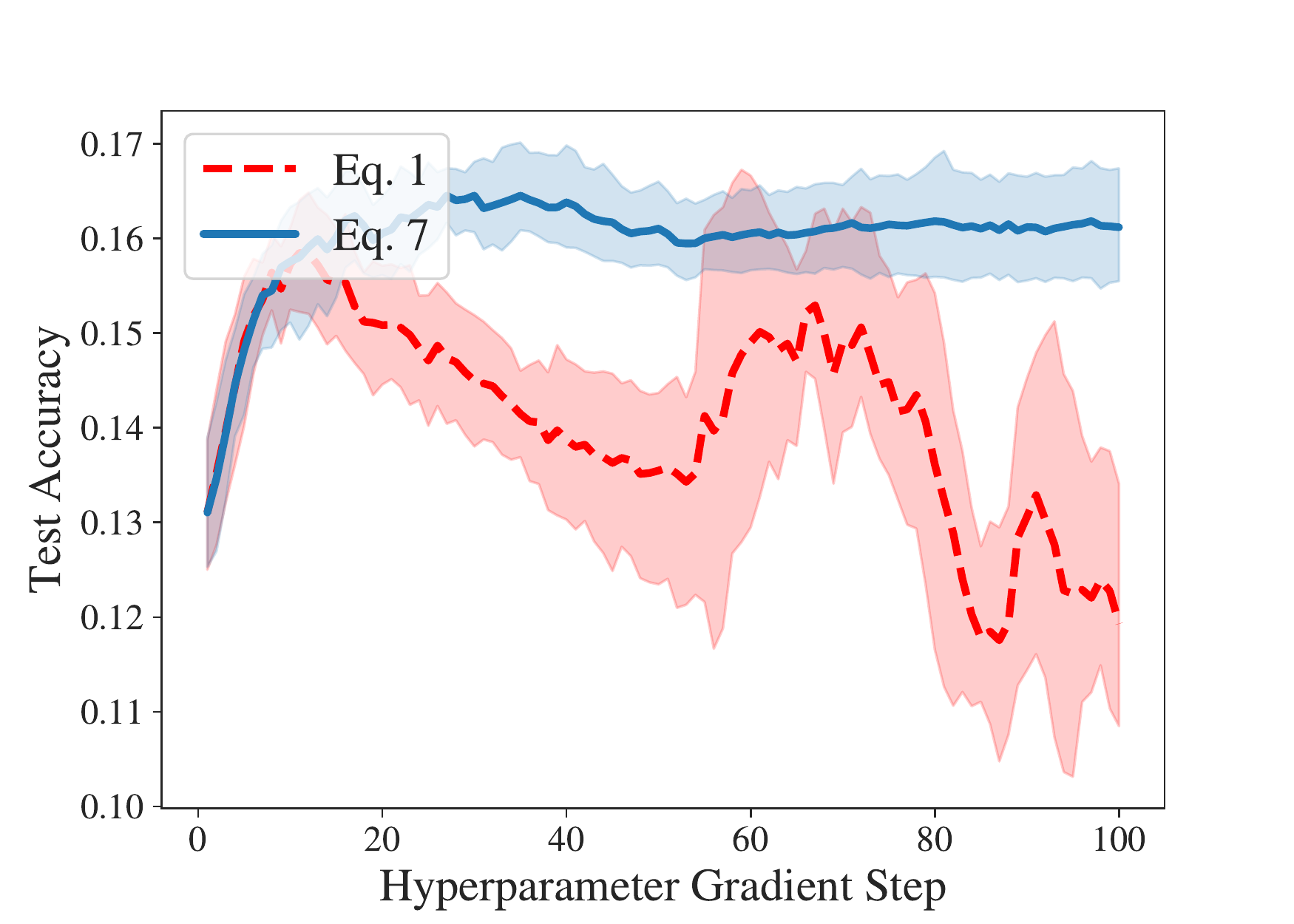}
         \caption{CIFAR-10 ($\zeta = \num{3.17e-3}$)}
         \label{fig:resnet34_cifar10_test_offline}
     \end{subfigure}
     \caption{Overfitting a validation set with per-parameter weight decays on ResNet-34 using Algorithm~\ref{alg:ho_stability_alg}.}
     \label{fig:resnet34_test_offline}
\end{figure}

\begin{table}[t]
\caption{Top-1 test accuracy comparison between classifiers obtained by way of either the minimum weight norm strategy or optimization of Eq.~\ref{eq:ho-stability}. The numbers in the parentheses correspond to 95\% confidence intervals obtained using the bootstrap.}
\label{tbl:baseline}
\vskip 0.15in
\begin{center}
\begin{small}
\begin{sc}
\begin{tabular}{lccc}
\toprule
Classifier & Data Set & Min. Weight Norm & Eq.~\ref{eq:ho-stability} \\
\midrule
Linear & MNIST & $56.8$ $(56.4, 57.2)$ & $64.0$ $(63.7, 64.4)$ \\
Linear & CIFAR-10 & $16.2$ $(15.9, 16.5)$ & $19.7$ $(19.6, 19.9)$ \\
ResNet-18 & MNIST & $63.3$ $(63.0, 63.7)$ & $71.9$ $(70.7, 73.2)$ \\
ResNet-18 & CIFAR-10 & $17.2$ $(16.8, 17.6)$ & $19.3$ $(18.8, 19.9)$         \\
ResNet-34 & MNIST & $52.5$ $(50.6, 54.1)$ & $58.6$ $(55.9, 60.8)$\\
ResNet-34 & CIFAR-10 & $14.4$ $(13.6, 15.2)$ & $16.1$ $(15.5, 16.7)$ \\
\bottomrule
\end{tabular}
\end{sc}
\end{small}
\end{center}
\vskip -0.1in
\end{table}

\begin{figure}[p]
     \centering
     \begin{subfigure}[t]{0.49\linewidth}
         \centering
         \includegraphics[width=\linewidth]{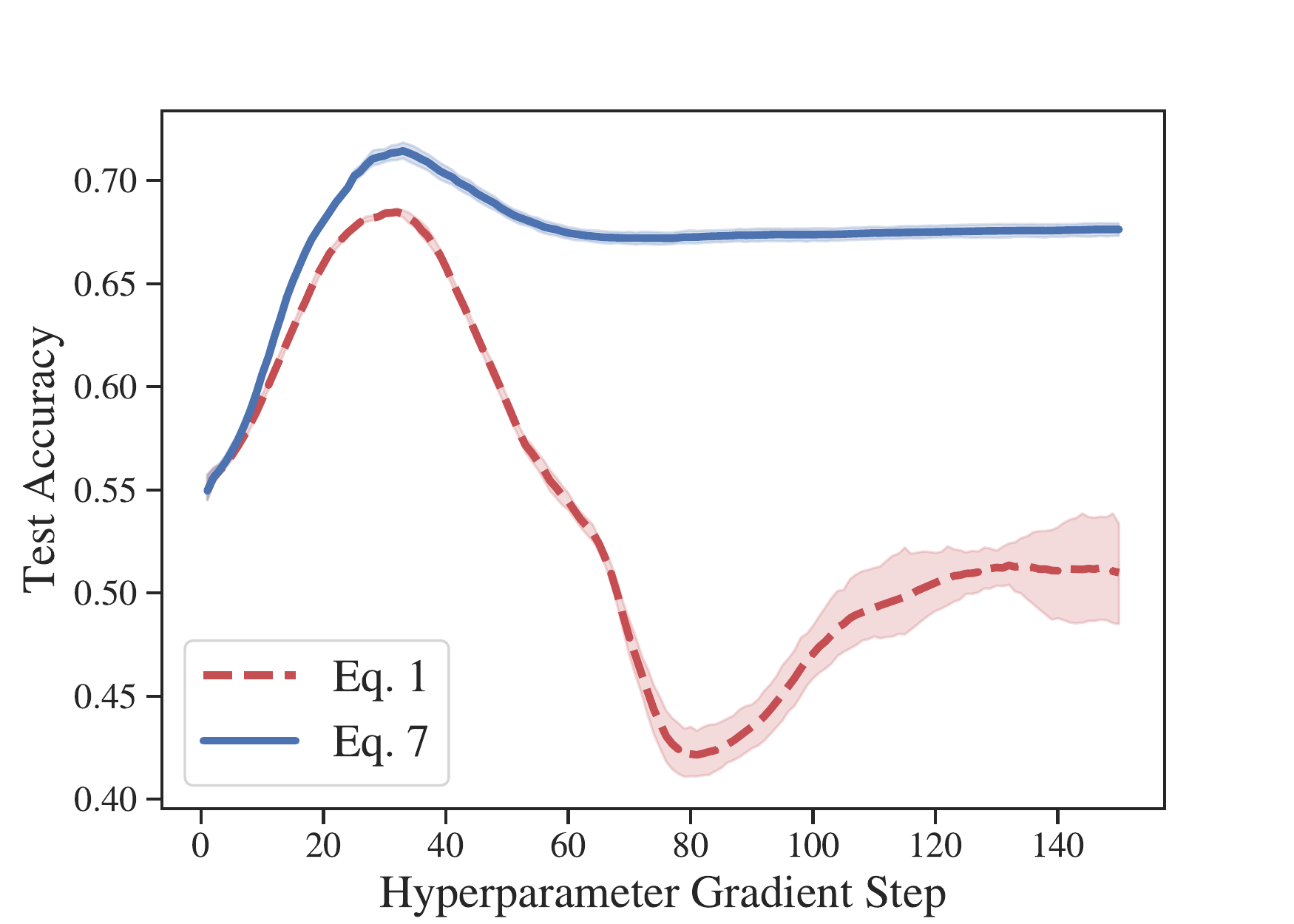}
         \caption{MNIST $\left (\zeta = \num{6.12e-5} \right)$}
         \label{fig:linear_mnist_test_online}
     \end{subfigure}
     \hfill
     \begin{subfigure}[t]{0.49\linewidth}
         \centering
         \includegraphics[width=\linewidth]{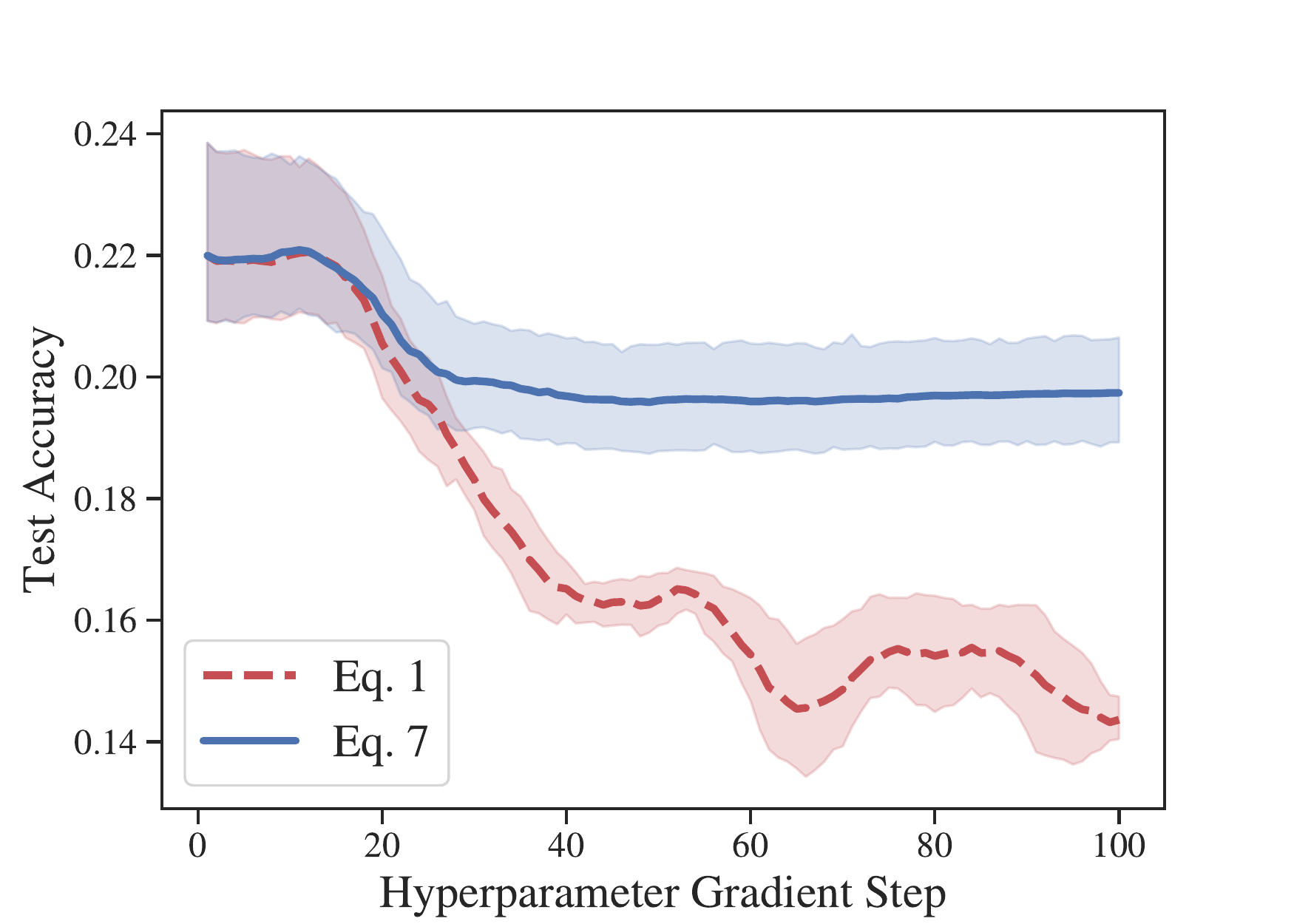}
         \caption{CIFAR-10 $\left (\zeta = \num{3.16e-4} \right)$}
         \label{fig:linear_cifar10_test_online}
     \end{subfigure}
     \caption{Overfitting a validation set with per-parameter weight decays on a linear classifier using Algorithm~\ref{alg:ho_stability_alg_online}.}
     \label{fig:linear_test_online}
\end{figure}

\begin{figure}[p]
     \centering
     \begin{subfigure}[t]{0.49\linewidth}
         \centering
         \includegraphics[width=\linewidth]{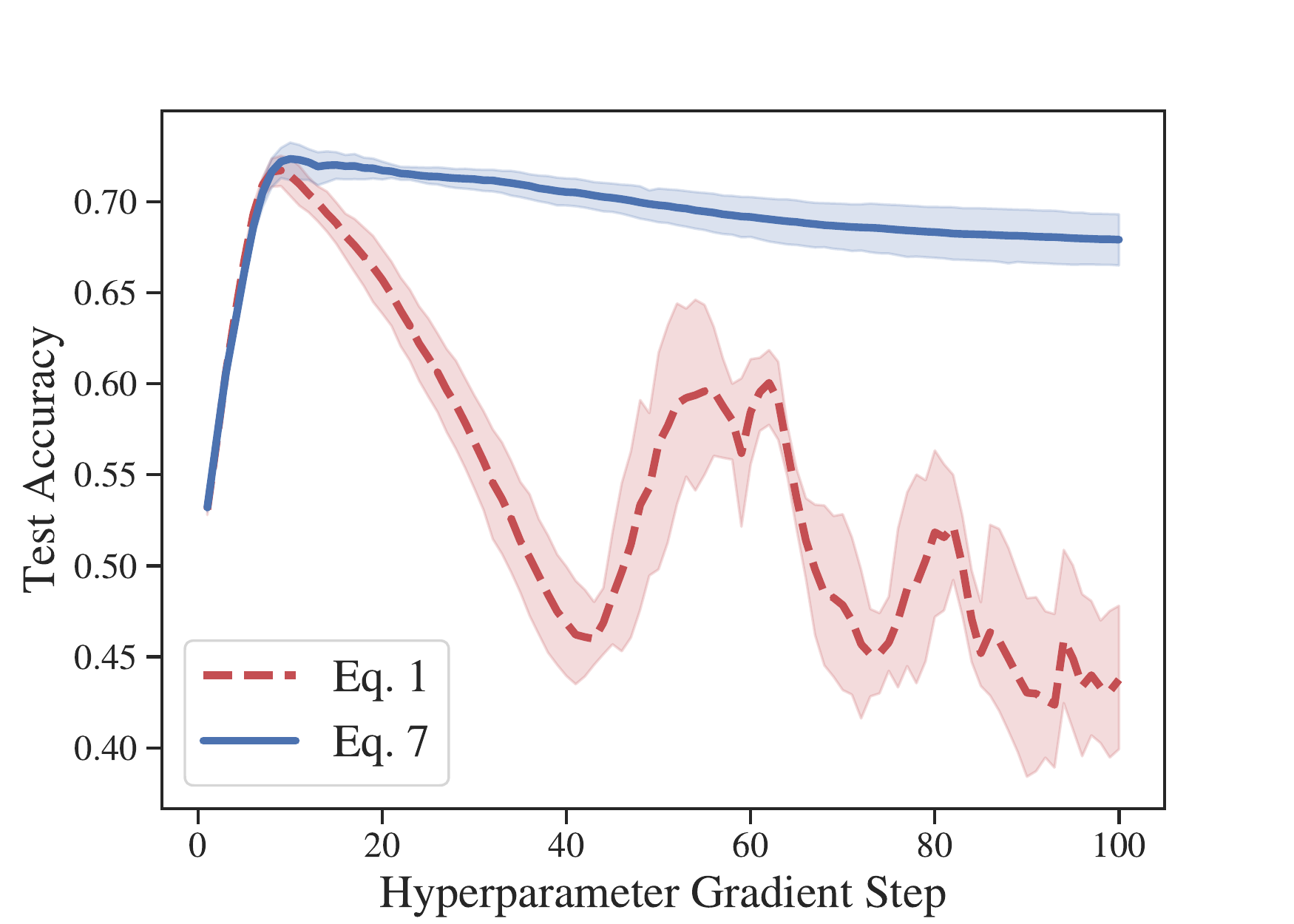}
         \caption{MNIST $\left (\zeta = \num{1.41e-4} \right)$}
         \label{fig:resnet18_mnist_test_online}
     \end{subfigure}
     \hfill
     \begin{subfigure}[t]{0.49\linewidth}
         \centering
         \includegraphics[width=\linewidth]{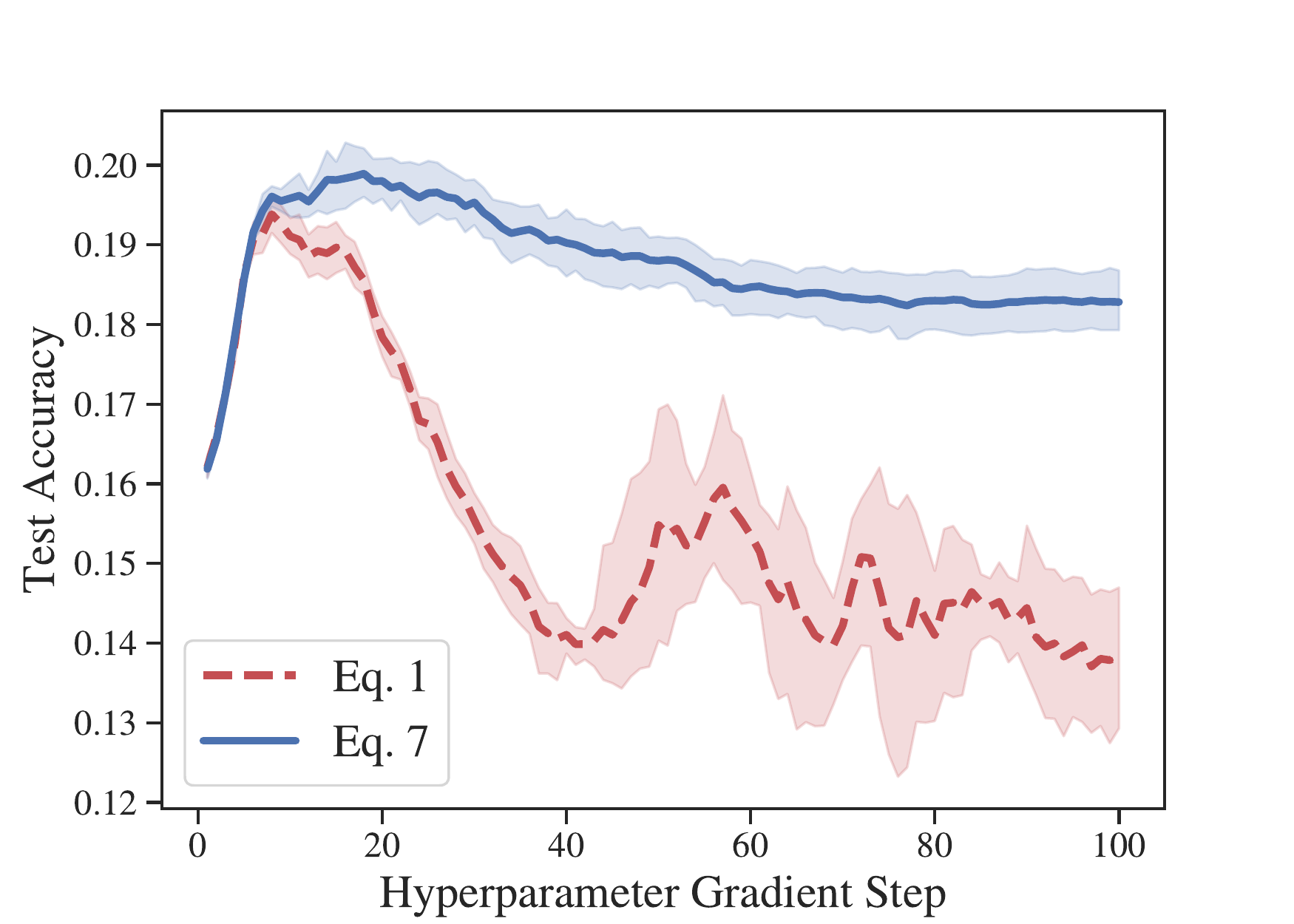}
         \caption{CIFAR-10 $\left (\zeta = \num{1.36e-4} \right)$}
         \label{fig:resnet18_cifar10_test_online}
     \end{subfigure}
     \caption{Overfitting a validation set with per-parameter weight decays on ResNet-18 using Algorithm~\ref{alg:ho_stability_alg_online}.}
     \label{fig:resnet18_test_online}
\end{figure}

\begin{figure}[p]
     \centering
     \begin{subfigure}[t]{0.49\linewidth}
         \centering
         \includegraphics[width=\linewidth]{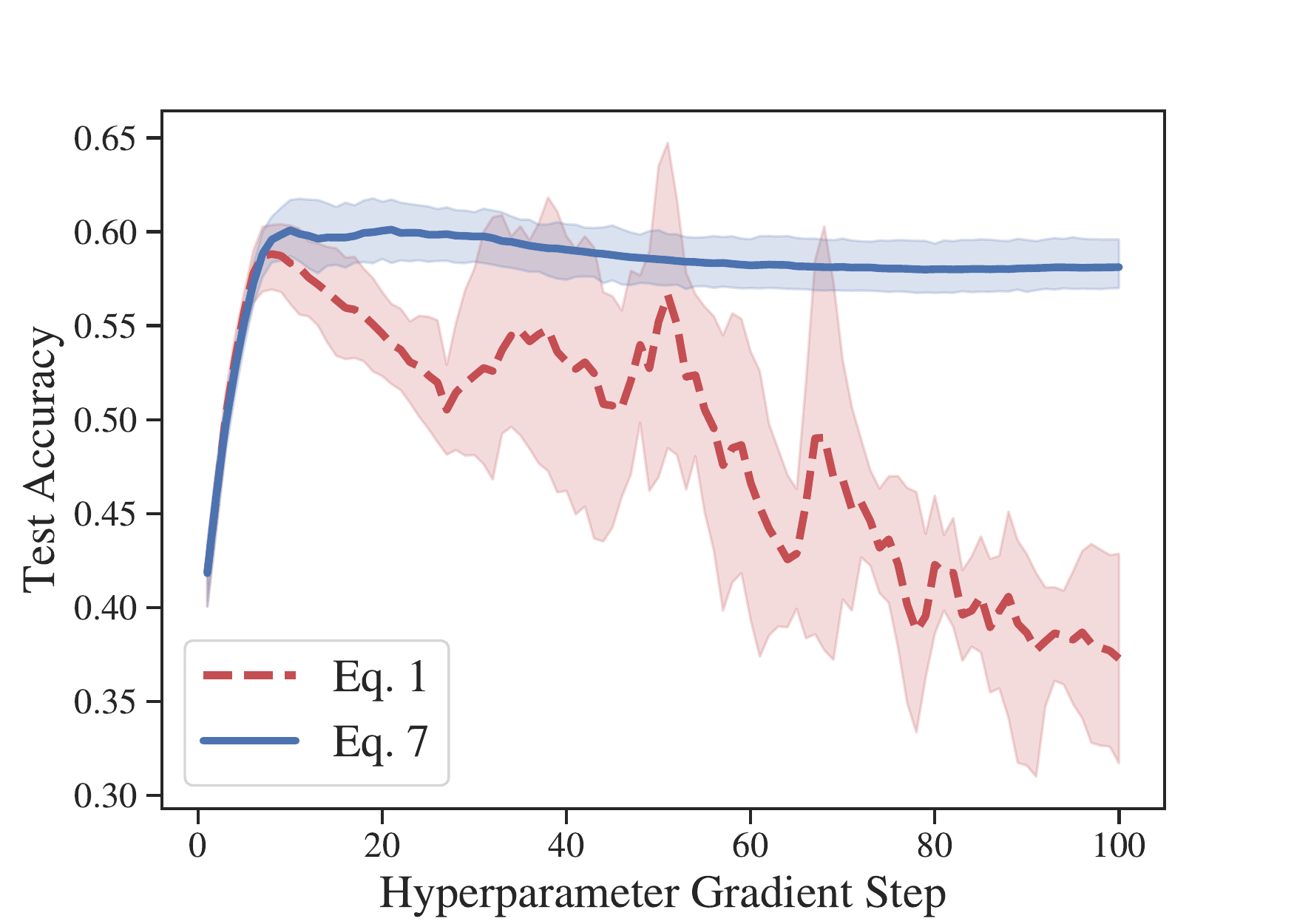}
         \caption{MNIST $\left (\zeta = \num{1.73e-4} \right)$}
         \label{fig:resnet34_mnist_test_online}
     \end{subfigure}
     \hfill
     \begin{subfigure}[t]{0.49\linewidth}
         \centering
         \includegraphics[width=\linewidth]{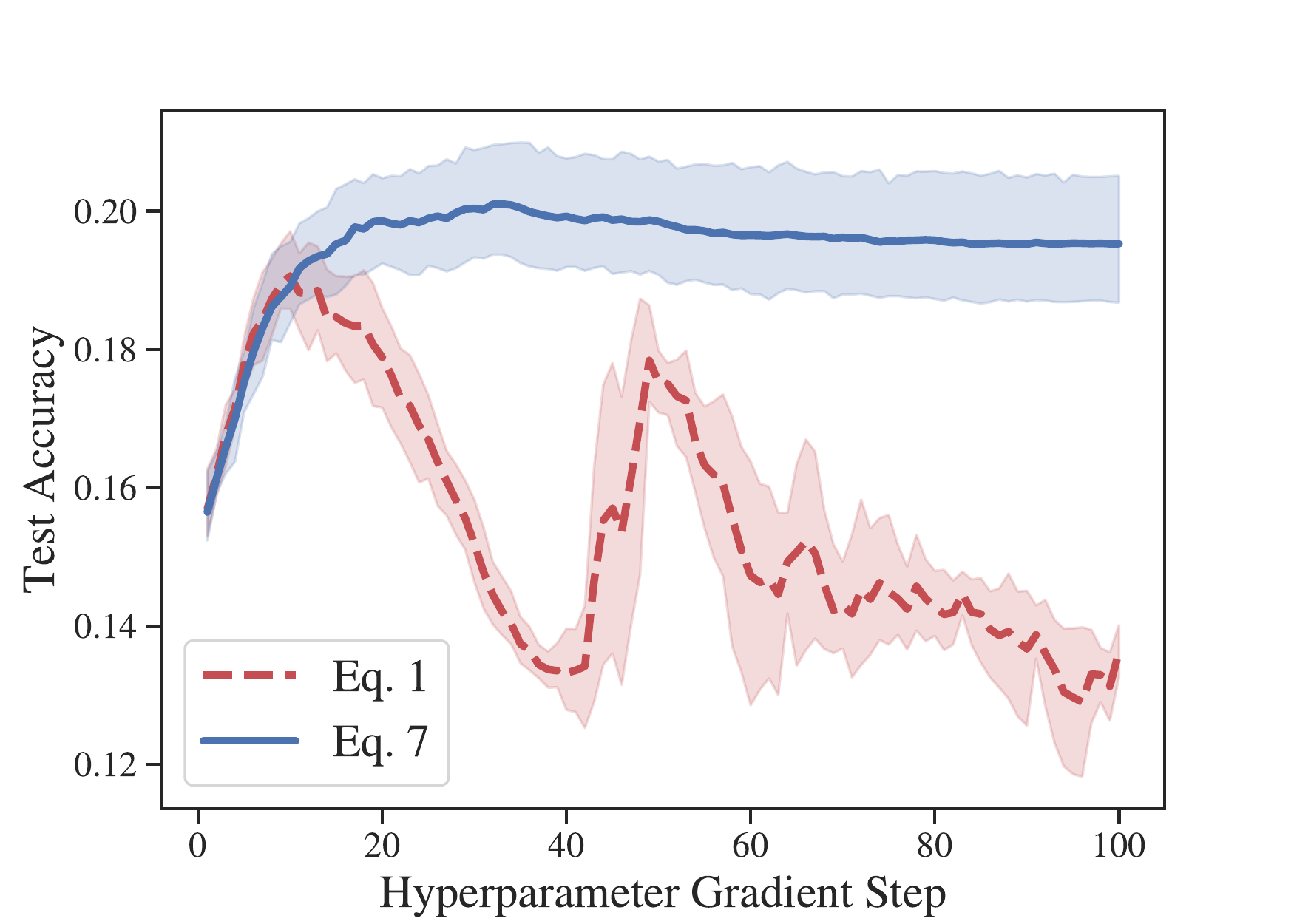}
         \caption{CIFAR-10 $\left (\zeta = \num{1.73e-4} \right)$}
         \label{fig:resnet34_cifar10_test_online}
     \end{subfigure}
     \caption{Overfitting a validation set with per-parameter weight decays on ResNet-34 using Algorithm~\ref{alg:ho_stability_alg_online}.}
     \label{fig:resnet34_test_online}
\end{figure}

The next set of experiments support our claim that minimizing Eq.~\ref{eq:ho-stability} is equivalent to minimizing a bound on the out-of-sample error. For each of the classifier-dataset pairs we consider, we minimize Eq.~\ref{eq:ho-stability} using Algorithm~\ref{alg:ho_stability_alg} over a wide range of $\zeta$. If our bounds were correct (and tight), the generalization error ought to monotonically decrease with larger choices of $\zeta$. In our experiments, we estimate the regularizer by taking the square root of the sum of squared Fisher distances over all parameter steps, and we estimate the generalization error by taking the difference between the average test loss and validation loss over the last five outer steps of hyperparameter optimization. Figures~\ref{fig:linear_reg_offline}--\ref{fig:resnet34_reg_offline} corroborate the theoretical connection we make between Eq.~\ref{eq:ho-stability} and generalization error bounds. These plots show a strong positive correlation between the value of the regularizer and the generalization error; the correlation remains positive even for the largest values of $\zeta$ tested. These empirical observations also support our use of the $K = 0$ approximation for the regularizer gradient. Even with this significant approximation, optimizing Eq.~\ref{eq:ho-stability} with larger values of $\zeta$ leads to smaller estimates of the regularizer after $100$ steps.

\begin{figure}[p]
     \centering
     \begin{subfigure}[t]{0.49\linewidth}
         \centering
         \includegraphics[width=\linewidth]{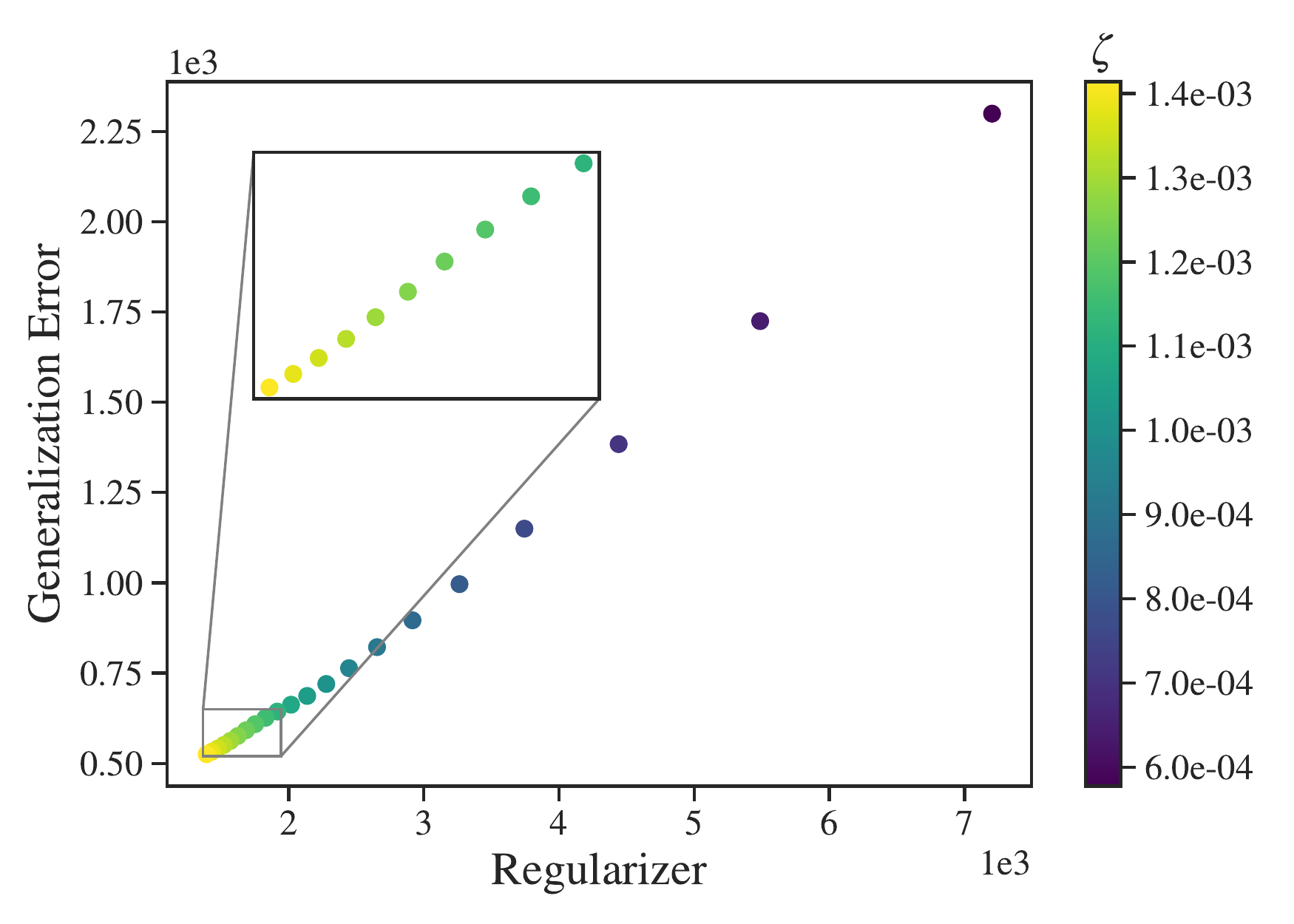}
         \caption{MNIST}
         \label{fig:linear_mnist_reg_offline}
     \end{subfigure}
     \hfill
     \begin{subfigure}[t]{0.49\linewidth}
         \centering
         \includegraphics[width=\linewidth]{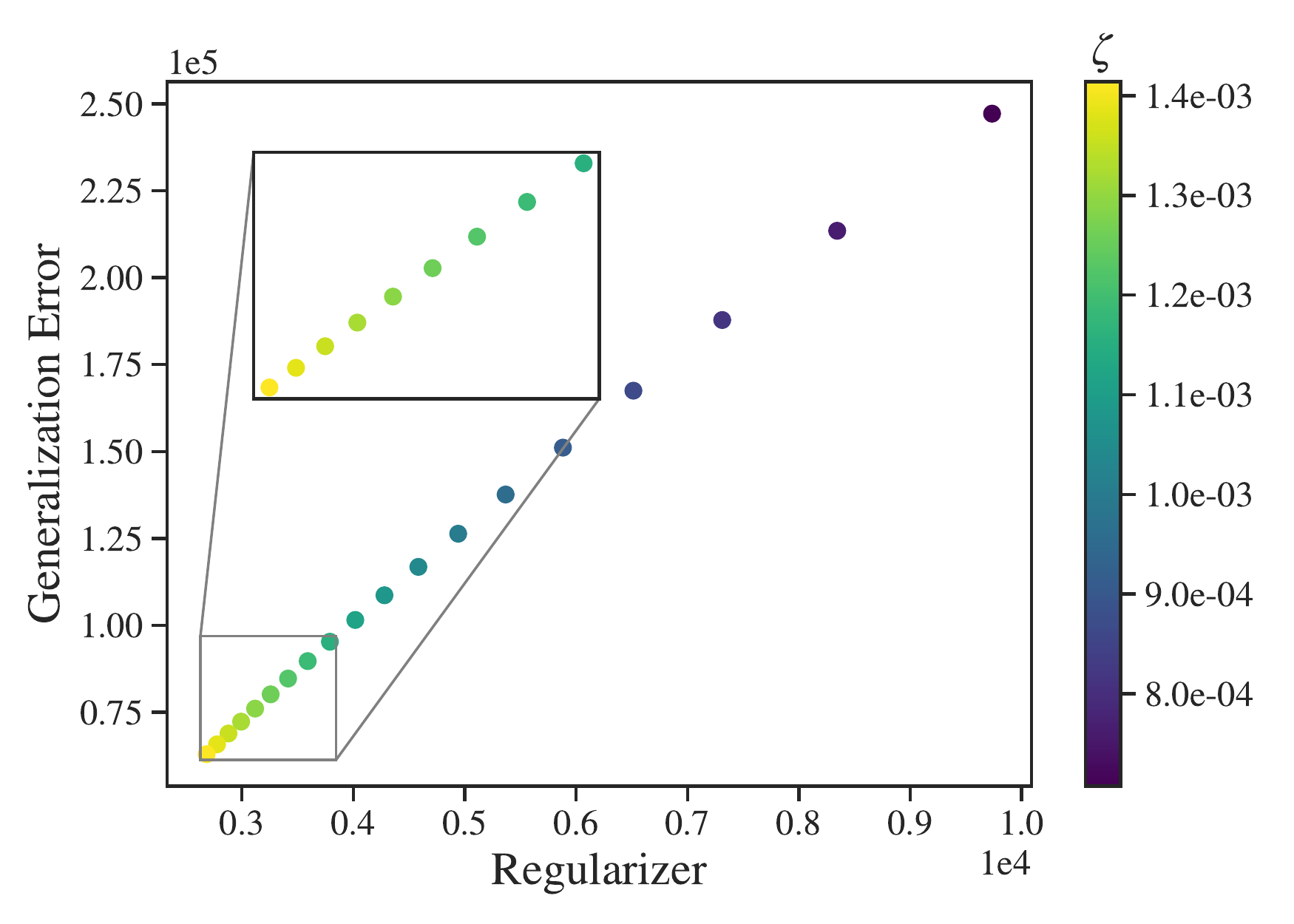}
         \caption{CIFAR-10}
         \label{fig:linear_cifar10_reg_offline}
     \end{subfigure}
     \caption{Comparing generalization error to the value of the regularizer in Eq.~\ref{eq:ho-stability} on a linear classifier using Algorithm~\ref{alg:ho_stability_alg}.}
     \label{fig:linear_reg_offline}
\end{figure}

\begin{figure}[p]
     \centering
     \begin{subfigure}[t]{0.49\linewidth}
         \centering
         \includegraphics[width=\linewidth]{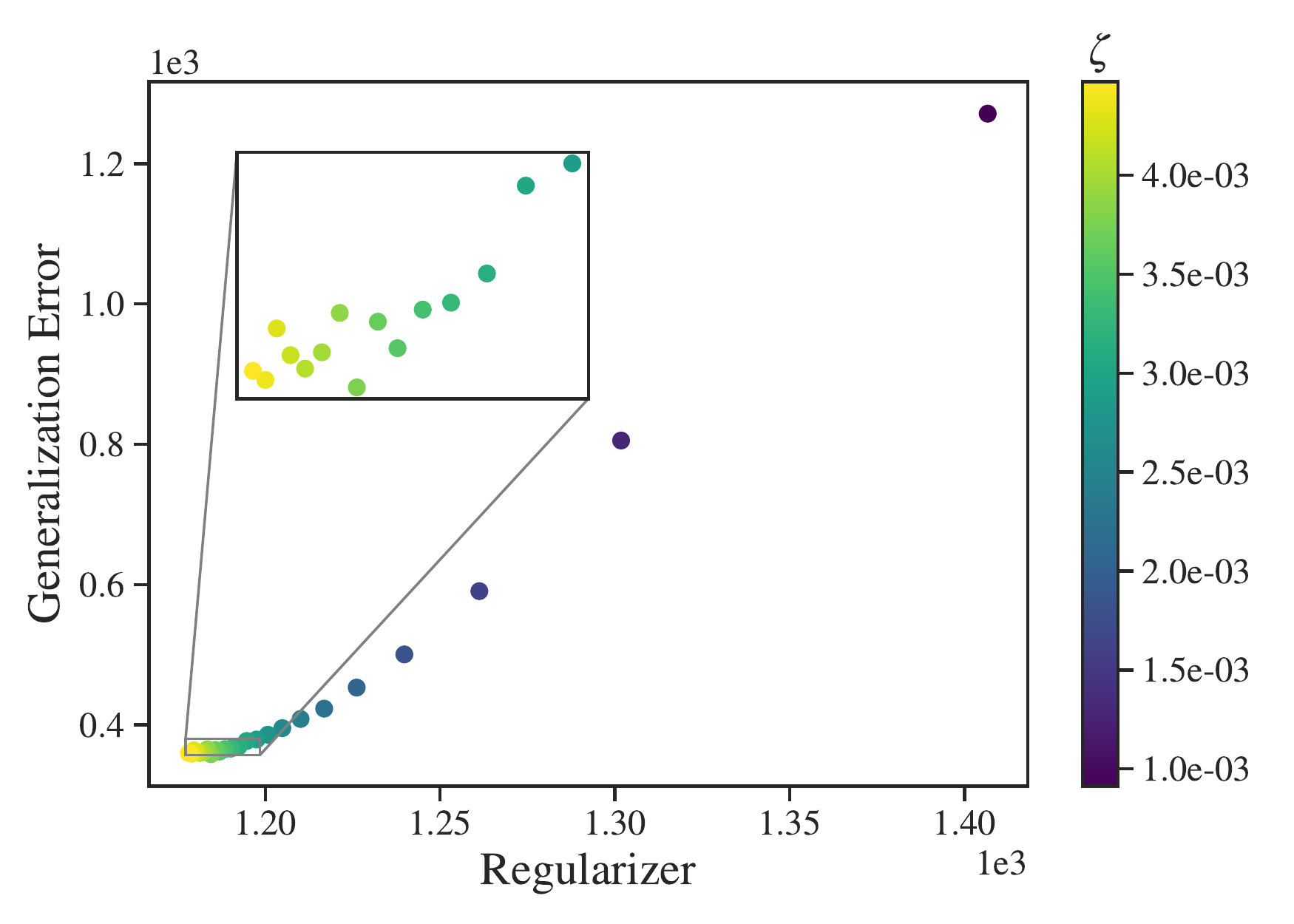}
         \caption{MNIST}
         \label{fig:resnet18_mnist_reg_offline}
     \end{subfigure}
     \hfill
     \begin{subfigure}[t]{0.49\linewidth}
         \centering
         \includegraphics[width=\linewidth]{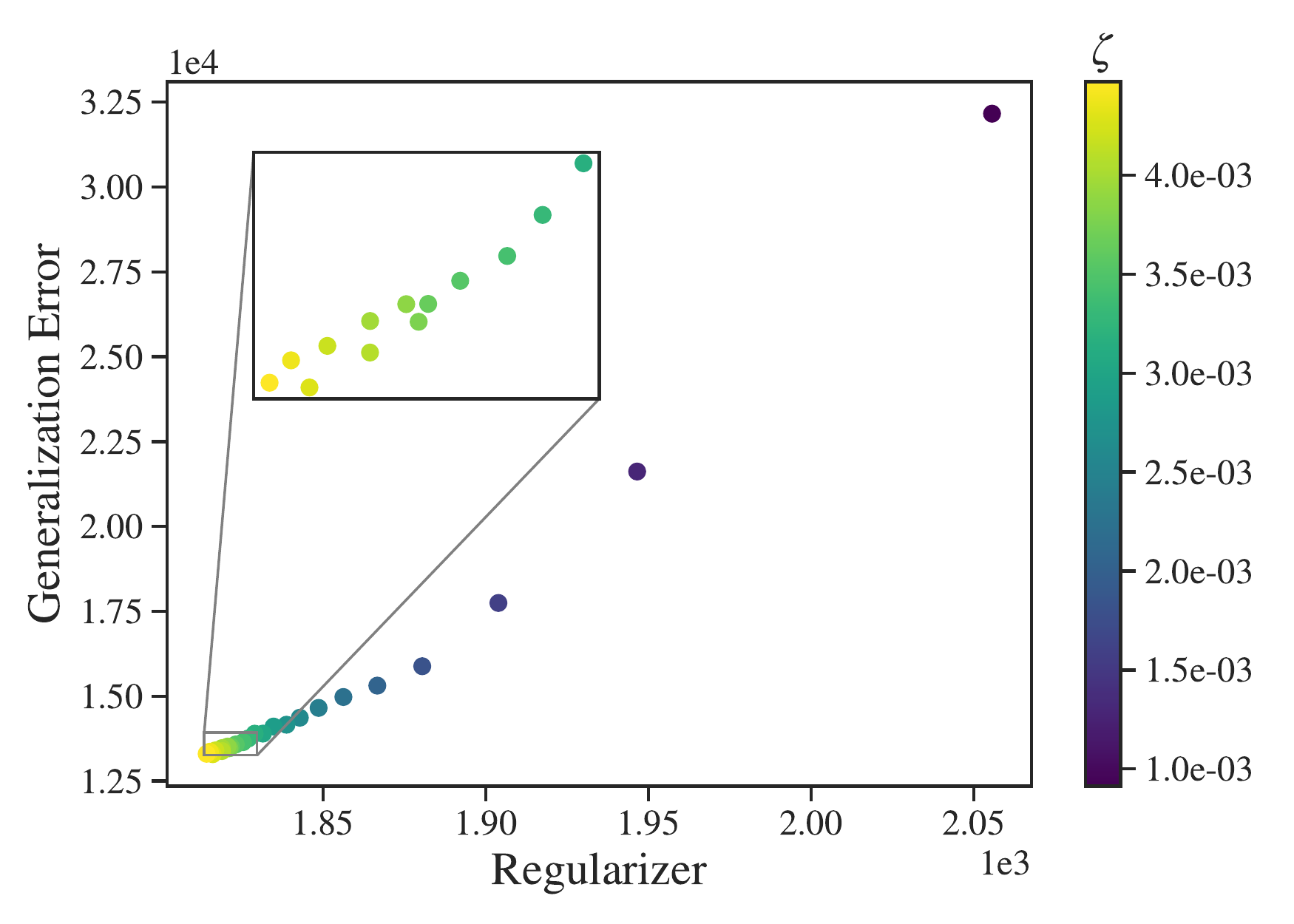}
         \caption{CIFAR-10}
         \label{fig:resnet18_cifar10_reg_offline}
     \end{subfigure}
     \caption{Comparing generalization error to the value of the regularizer in Eq.~\ref{eq:ho-stability} on ResNet-18 using Algorithm~\ref{alg:ho_stability_alg}.}
     \label{fig:resnet18_reg_offline}
\end{figure}

\begin{figure}[p]
     \centering
     \begin{subfigure}[t]{0.49\linewidth}
         \centering
         \includegraphics[width=\linewidth]{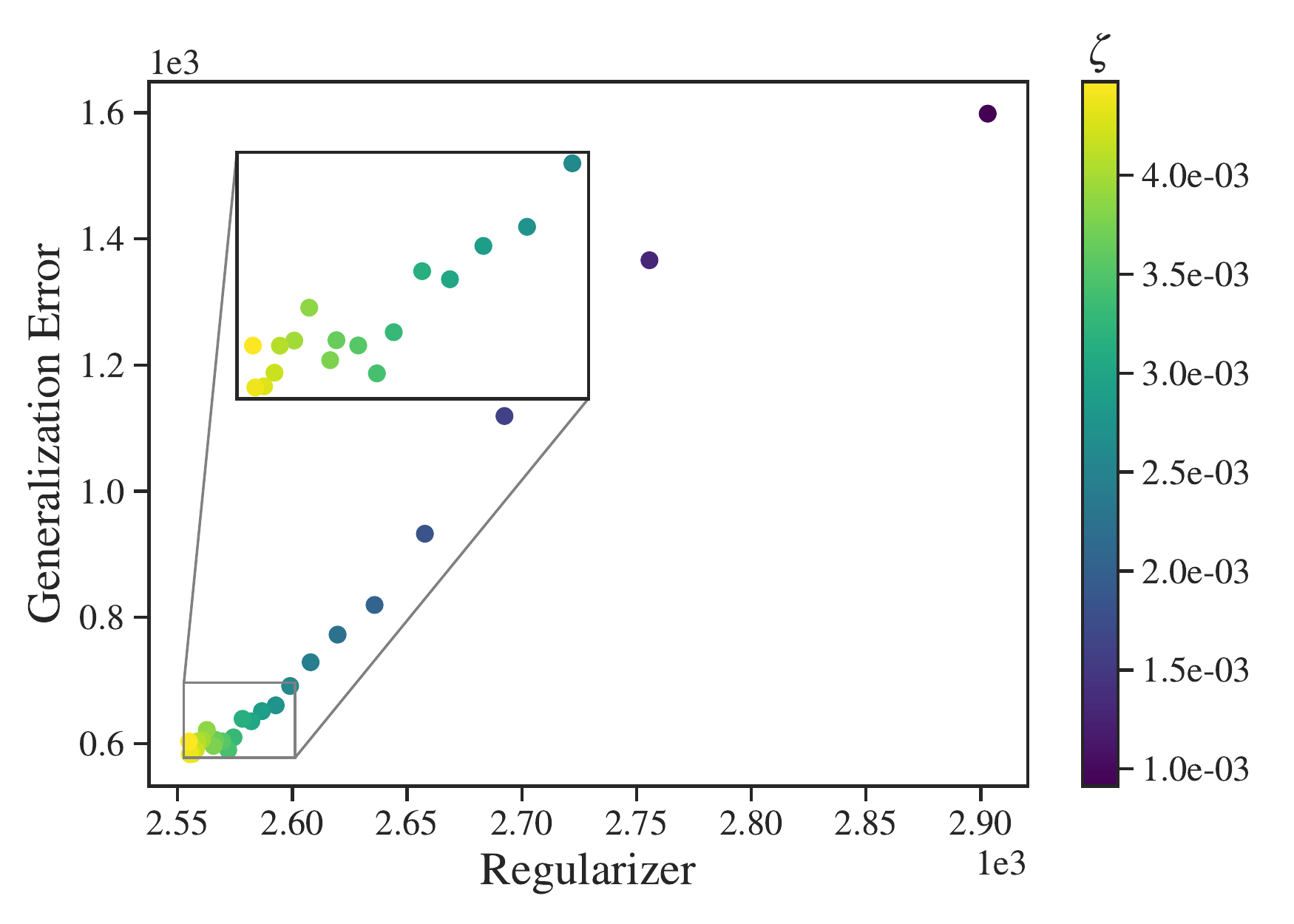}
         \caption{MNIST}
         \label{fig:resnet34_mnist_reg_offline}
     \end{subfigure}
     \hfill
     \begin{subfigure}[t]{0.49\linewidth}
         \centering
         \includegraphics[width=\linewidth]{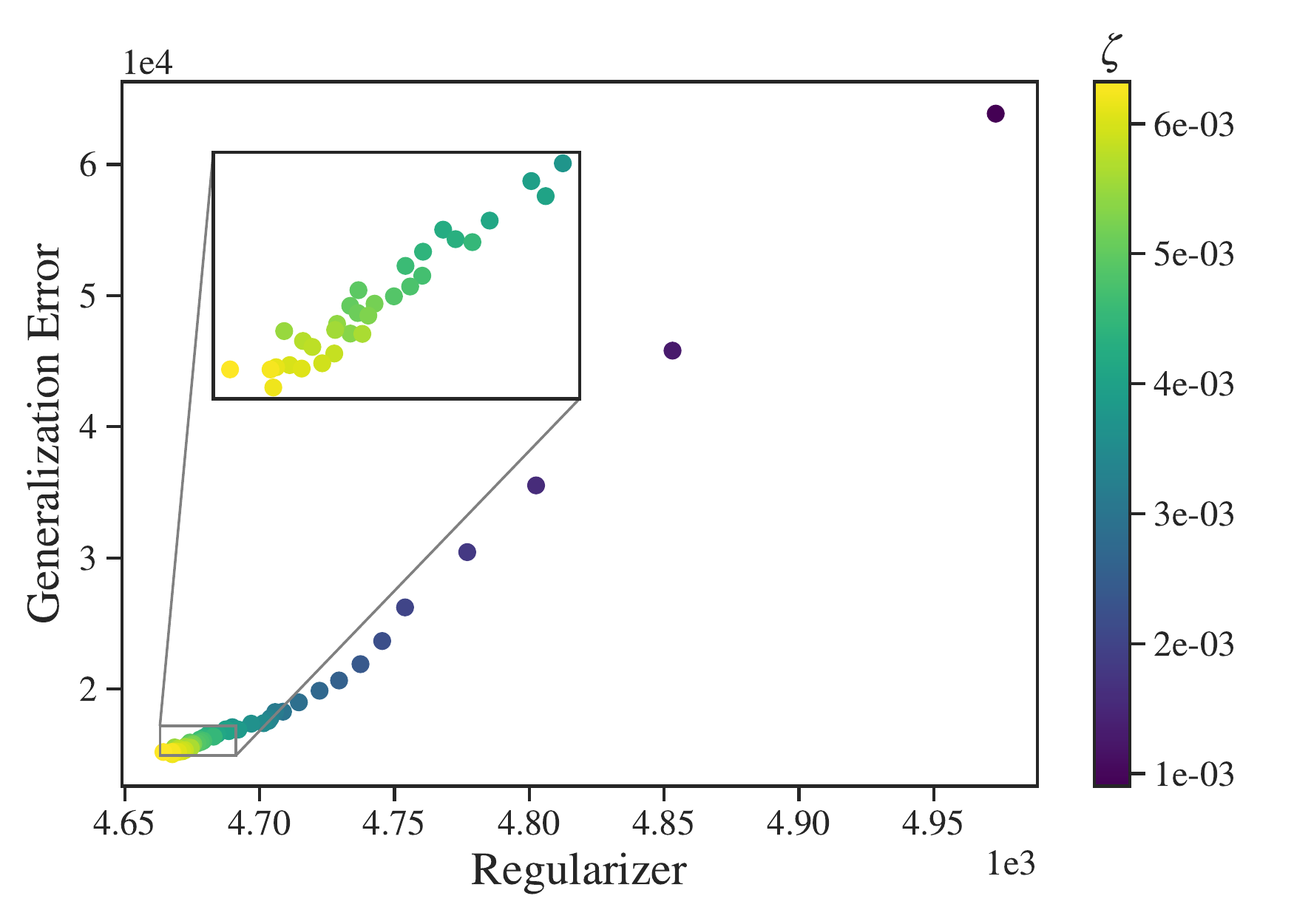}
         \caption{CIFAR-10}
         \label{fig:resnet34_cifar10_reg_offline}
     \end{subfigure}
     \caption{Comparing generalization error to the value of the regularizer in Eq.~\ref{eq:ho-stability} on ResNet-34 using Algorithm~\ref{alg:ho_stability_alg}.}
     \label{fig:resnet34_reg_offline}
\end{figure}

Last, we show that the test loss of the optimized model is reduced over a wide range of possible $\zeta$. Figures~\ref{fig:linear_scan}--\ref{fig:resnet34_scan} show that a wide range of $\zeta$ lead to improved test loss when compared to the unregularized objective. These plots support our argument from Section~\ref{sec:regpenalty} that careful tuning of $\zeta$ is not necessary to improve the optimized model's expected risk.

\begin{figure}[p]
     \centering
     \begin{subfigure}[t]{0.49\linewidth}
         \centering
         \includegraphics[width=\linewidth]{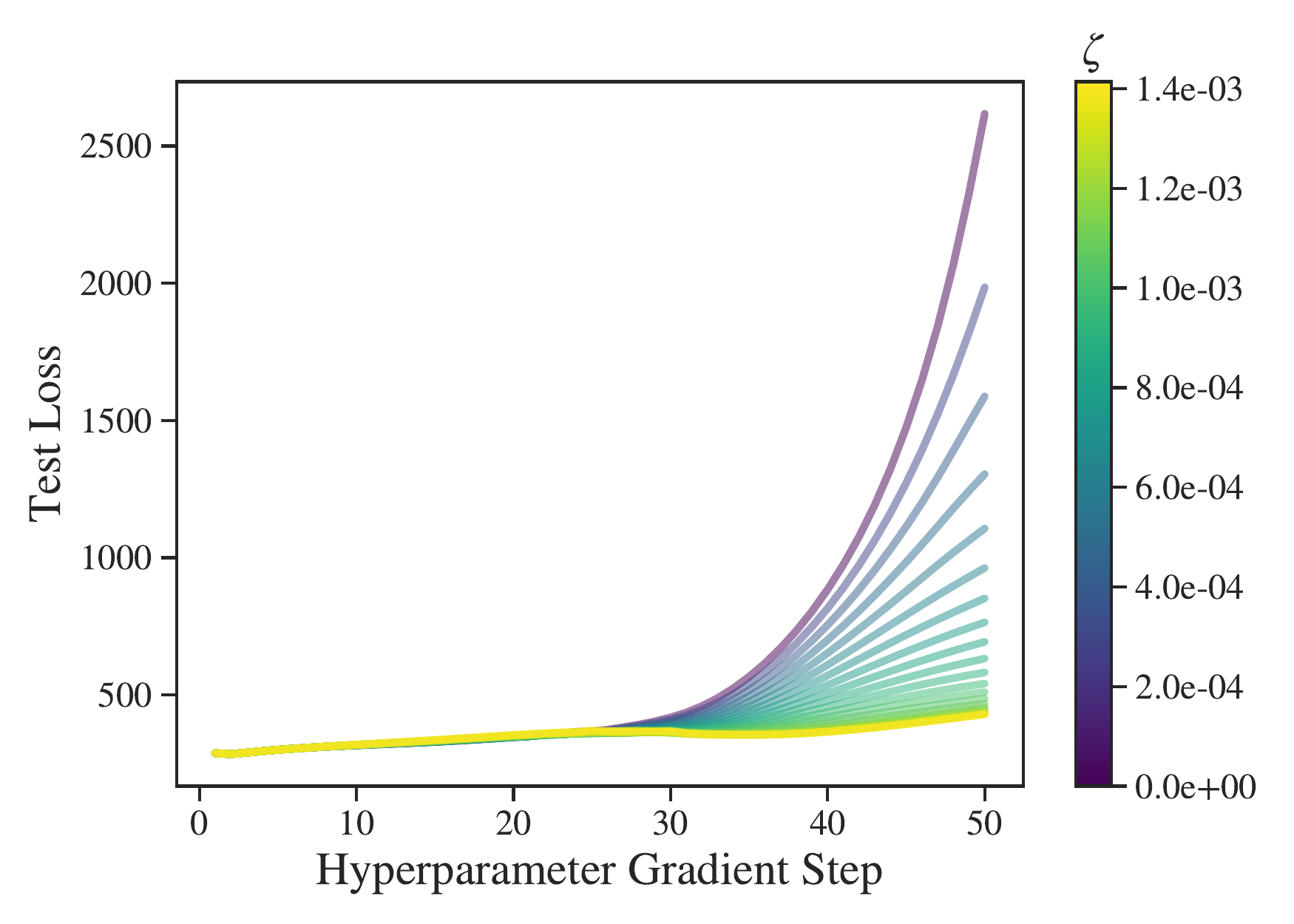}
         \caption{MNIST}
         \label{fig:linear_mnist_scan}
     \end{subfigure}
     \hfill
     \begin{subfigure}[t]{0.49\linewidth}
         \centering
         \includegraphics[width=\linewidth]{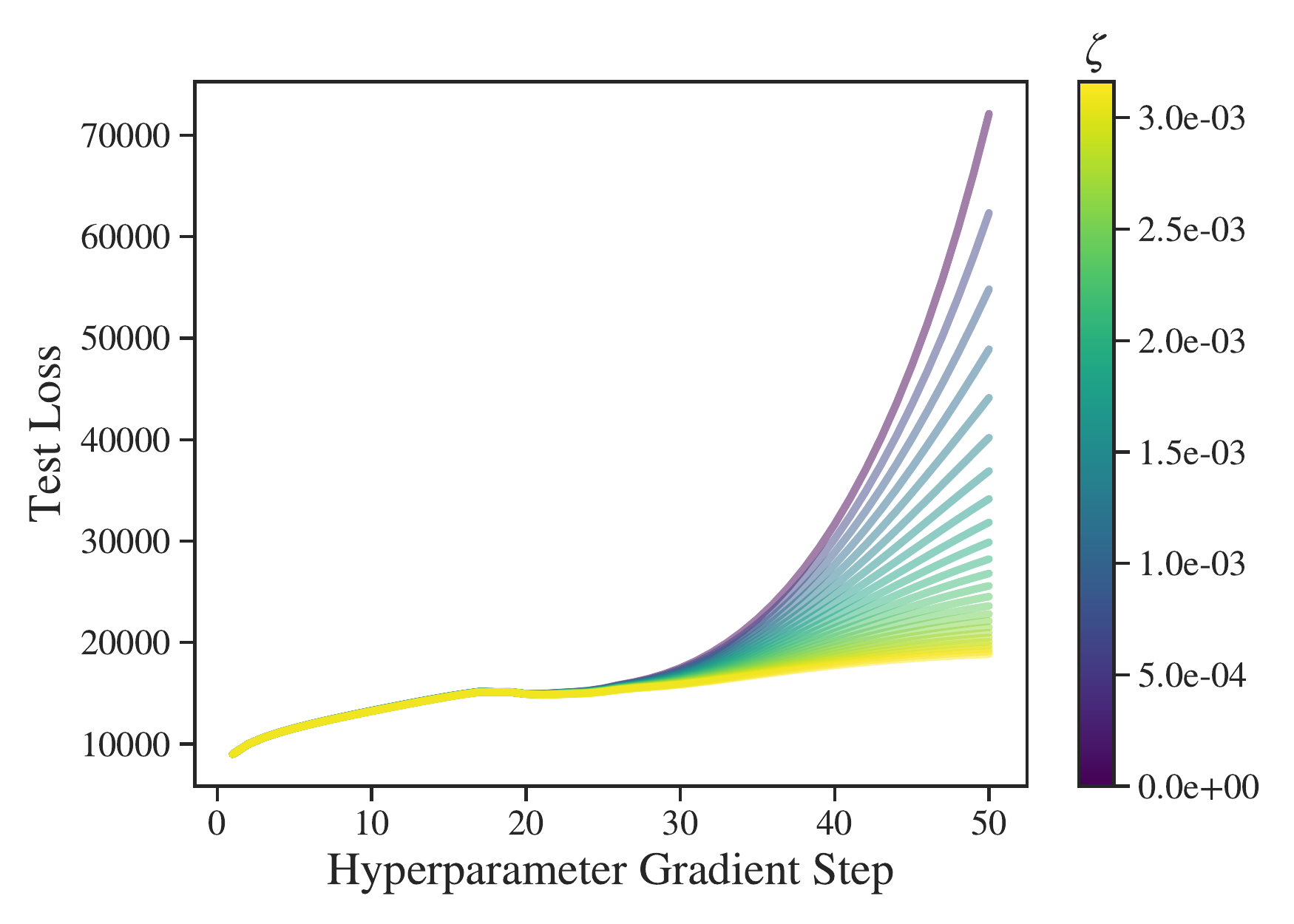}
         \caption{CIFAR-10}
         \label{fig:linear_cifar10_scan_}
     \end{subfigure}
     \caption{Plotting the effect of $\zeta$ on test loss for a linear classifier using Algorithm~\ref{alg:ho_stability_alg}.}
     \label{fig:linear_scan}
\end{figure}

\begin{figure}[p]
     \centering
     \begin{subfigure}[t]{0.49\linewidth}
         \centering
         \includegraphics[width=\linewidth]{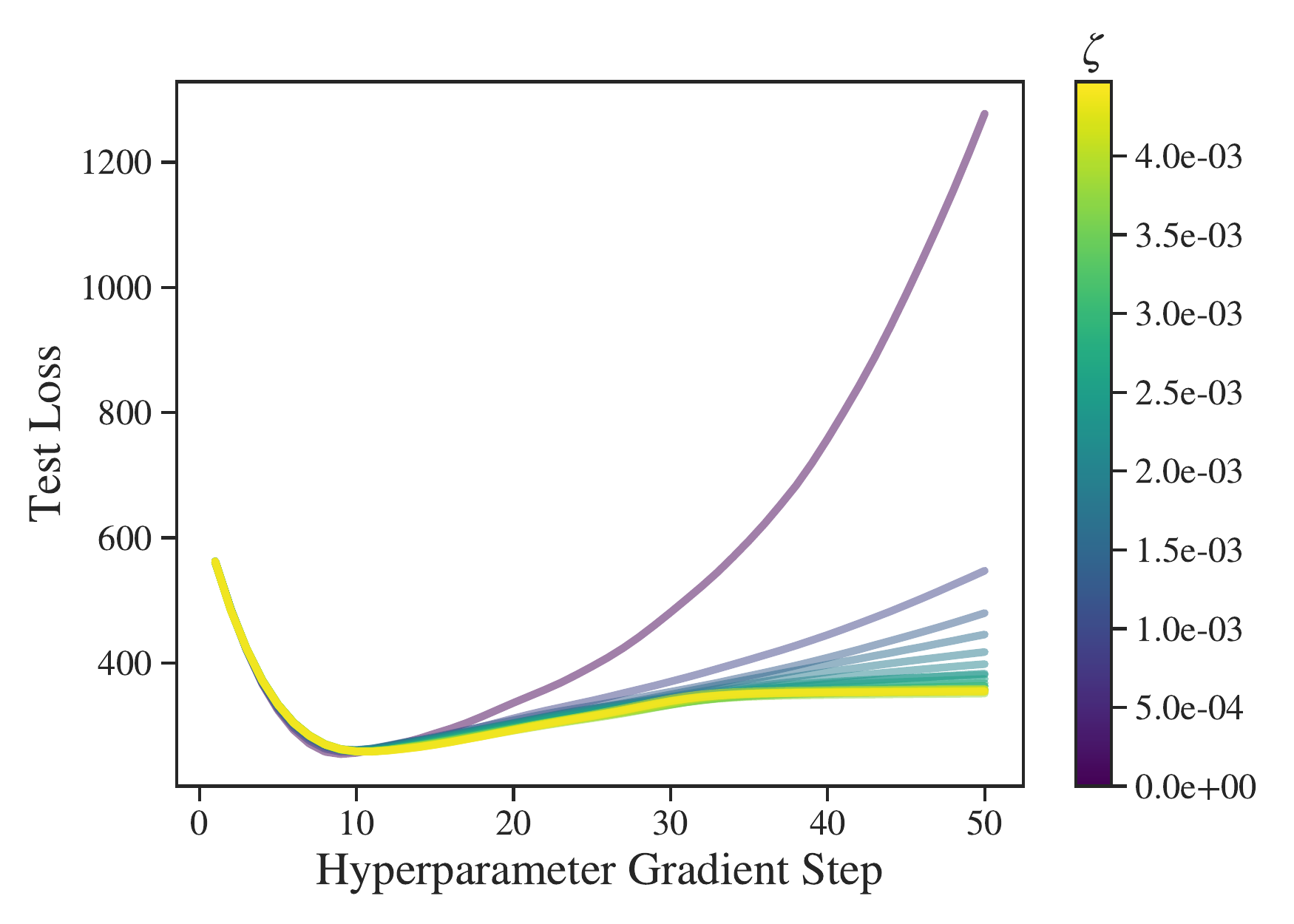}
         \caption{MNIST}
         \label{fig:resnet18_mnist_scan}
     \end{subfigure}
     \hfill
     \begin{subfigure}[t]{0.49\linewidth}
         \centering
         \includegraphics[width=\linewidth]{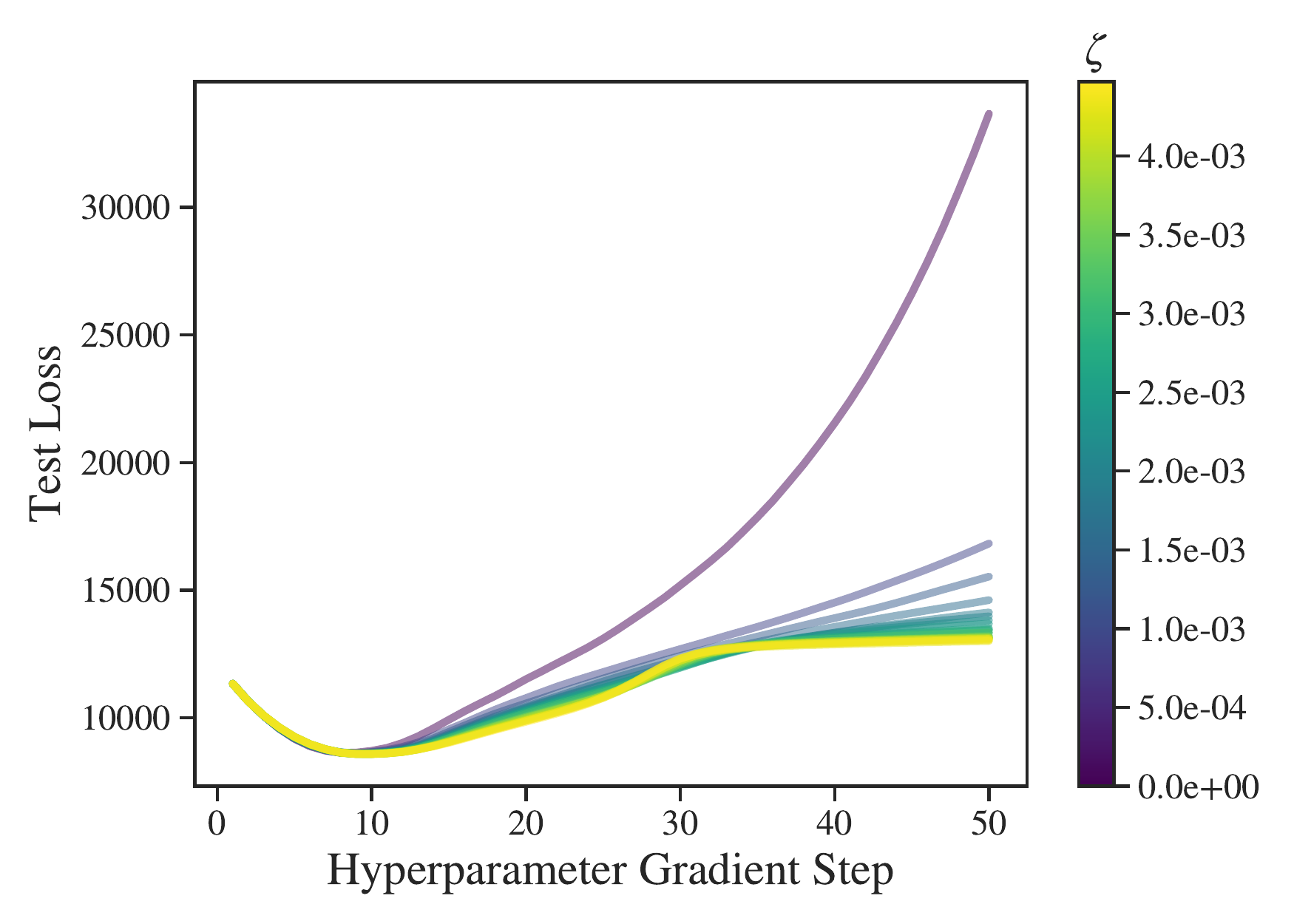}
         \caption{CIFAR-10}
         \label{fig:resnet18_cifar10_scan}
     \end{subfigure}
     \caption{Plotting the effect of $\zeta$ on test loss for ResNet-18 using Algorithm~\ref{alg:ho_stability_alg}.}
     \label{fig:resnet18_scan}
\end{figure}

\begin{figure}[p]
     \centering
     \begin{subfigure}[t]{0.49\linewidth}
         \centering
         \includegraphics[width=\linewidth]{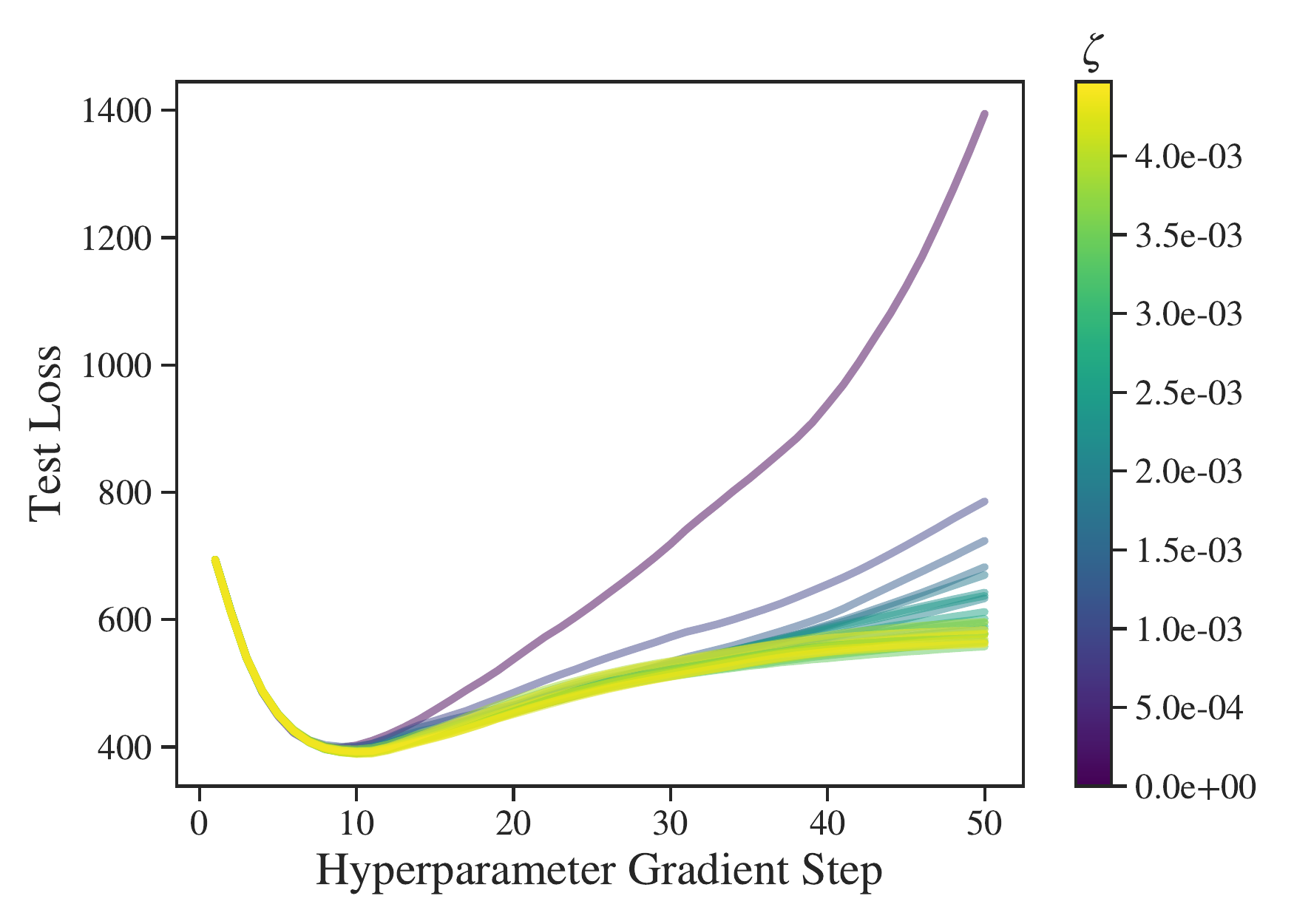}
         \caption{MNIST}
         \label{fig:resnet34_mnist_scan}
     \end{subfigure}
     \hfill
     \begin{subfigure}[t]{0.49\linewidth}
         \centering
         \includegraphics[width=\linewidth]{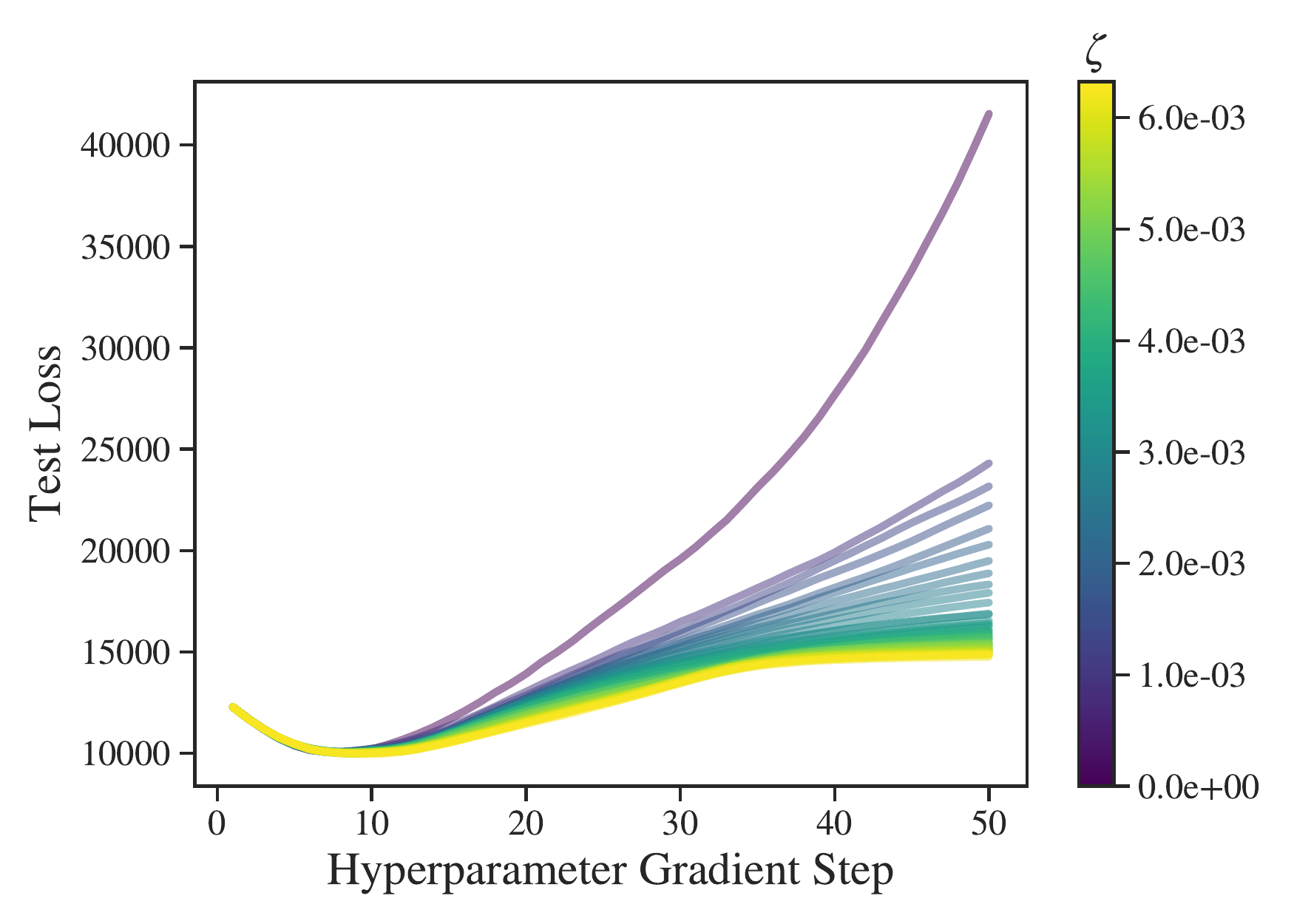}
         \caption{CIFAR-10}
         \label{fig:resnet34_cifar10_scan}
     \end{subfigure}
     \caption{Plotting the effect of $\zeta$ on test loss for ResNet-34 using Algorithm~\ref{alg:ho_stability_alg}.}
     \label{fig:resnet34_scan}
\end{figure}

Though we do not present the results of this experiment here, we also investigated the effect of retaining the parameter resampling step in Algorithm~\ref{alg:ho_stability_alg_mt}. If we instead define $P_0$ to be a delta function over the initial $\theta_0$ and resample $\theta_0$ at each outer step, optimizing Eq.~\ref{eq:bilevel-ho} no longer results in overfitting for certain classifier-dataset pairs and reduces overfitting for the other pairs. Optimizing  Eq.~\ref{eq:ho-stability} still improves test set accuracy when compared to optimizing Eq.~\ref{eq:bilevel-ho}; the improvement is, however, less substantial.

\subsection{Related Observations} \label{sec:relobsappendix}
In Section~\ref{sec:relatedobs}, we motivate recently developed heuristics for improving generalization in meta-learning using the ideas developed in this paper. Here we elaborate on the relationship between published observations and the bounds we prove.

\citet{guiroy2019towards} study meta-learning problems in which the hyperparameter learned is the parameter initialization for unseen tasks. Before we expound on the relationship between their observations and our work, we define some relevant meta-learning terminology. The procedure used for learning the parameter initialization is referred to as ``meta-training.'' A short sequence of gradient steps beginning from the initialization learned during meta-training is then termed ``task adaptation.'' \citet{guiroy2019towards} run task adaptation on a set of tasks held-out from meta-training, and measure the initialization's out-of-sample error by computing the average post-adaptation error for each of those tasks.

\citet{guiroy2019towards} evaluate several proxies for generalization in meta-learning to motivate their choice of regularizer. They first observe that the spectral norm of the Hessian after task adaptation is poorly correlated with post-adaptation error on unseen tasks; selecting a parameter initialization that tends to converge to flat optima does not appear to improve error on additional tasks. While optima flatness is sometimes well-correlated with generalization error \cite{zela2020understanding}, this finding is consistent with Observation~\ref{obs:app_trhessian}. Recall once more that minimizing optima flatness is equivalent to minimizing a data-independent PAC-Bayes bound. Depending on the problem, a data-independent PAC-Bayes prior may not result in a bound that is well-correlated with generalization error.

By contrast, they observe that two measures of gradient incoherence are well-correlated with error on unseen tasks. The first measure is the average cosine similarity between the direction vectors connecting the parameter initialization and  post-adaptation optimum for pairs of tasks held-out from meta-training. The second is the average inner product between the first gradient steps taken on a pair of held-out tasks.

Rigorously proving that these proxies for out-of-sample error are approximations to PAC-Bayes bounds would require extending our analysis to the meta-learning PAC-Bayes bound derived by \citet{AmitMeir2017}. In Observation~\ref{obs:normclip}, we showed how our bound can be modified to yield a cosine-similarity-based regularizer for hyperparameter optimization, albeit not for the problem setting \citet{guiroy2019towards} study. Though we do not prove this result here, we speculate that one could derive a similar regularizer for meta-learning by replacing the data-independent PAC-Bayes prior used in Eq.~4 of \citet{AmitMeir2017} with the prior we proposed in Observation~\ref{obs:normclip}. Repeating the derivation we lay out in Section~\ref{sec:gentheory} for this PAC-Bayes bound could thus lead to the discovery of improved regularization strategies for gradient-based meta-learning.
\end{document}